\documentclass[]{fairmeta}

\usepackage{amsmath,amsfonts,bm}

\def\eqref#1{equation~\ref{#1}}

\def\1{\bm{1}}

\DeclareMathAlphabet{\mathsfit}{\encodingdefault}{\sfdefault}{m}{sl}
\SetMathAlphabet{\mathsfit}{bold}{\encodingdefault}{\sfdefault}{bx}{n}

\usepackage{url}
\usepackage{algorithm}
\usepackage[noend]{algpseudocode}
\usepackage{listings}
\usepackage{tabularx,array}
\usepackage{enumitem}
\usepackage{wrapfig}
\usepackage{colortbl}  
\definecolor{findingaccent}{RGB}{46, 112, 108}   
\definecolor{findingbg}{RGB}{246, 250, 249}      

\newtcolorbox{keyfinding}[1]{
    enhanced,
    breakable,
    colback=findingbg,
    colframe=findingaccent,
    boxrule=0.8pt,
    arc=2mm,
    left=4mm,
    right=4mm,
    top=4mm,
    bottom=2.5mm,
    before skip=6pt,
    after skip=6pt,
    attach boxed title to top left={
        xshift=5mm,
        yshift=-2mm
    },
    boxed title style={
        colback=findingaccent,
        colframe=findingaccent,
        arc=1.5mm,
        boxrule=0pt,
        left=2.5mm,
        right=2.5mm,
        top=1mm,
        bottom=1mm,
    },
    coltitle=white,
    fonttitle=\bfseries,
    title={#1},
}

\lstdefinestyle{promptstyle}{
  basicstyle=\ttfamily\scriptsize,
  columns=fullflexible,
  keepspaces=true,
  breaklines=true,
  breakatwhitespace=true,
  showstringspaces=false,
  xleftmargin=6pt,
  xrightmargin=6pt,
  aboveskip=0pt,
  belowskip=0pt
}
\newtcblisting{promptbox}[1][]{
  listing only,
  breakable,
  colback=gray!5,
  colframe=gray!60,
  boxrule=0.5pt,
  arc=2pt,
  left=0pt,
  right=0pt,
  top=6pt,
  bottom=6pt,
  title={#1},
  listing options={style=promptstyle}
}

\newcommand{\basemethod}{\textsc{VHarness-BRSI}}
\newcommand{\method}{\textsc{VHarness-SRSI}}

\title{VHarness-SRSI: Structured Recursive Self-Improvement for Video Understanding Harness}

\author[1,5,*]{Zihao Lin}
\author[2]{Haoyu Dong}
\author[3]{Haobo Xu}
\author[4]{Chi-Chih Chang}
\author[3]{Xuying Ning}
\author[5]{Yinglong Xia}
\author[5]{Shuang Yang}
\author[5]{Zhuokai Zhao}
\author[5]{Xinying Song}
\author[6]{Dongfang Liu}
\author[1]{Lifu Huang}
\author[5]{Yang Song}
\author[5]{Qifan Wang}

\affiliation[1]{University of California, Davis}
\affiliation[2]{Duke University}
\affiliation[3]{University of Illinois at Urbana-Champaign}
\affiliation[4]{Cornell University}
\affiliation[5]{Meta MRS Research}
\affiliation[6]{Purdue University}

\contribution[*]{Work done at Meta}

\abstract{LLM system capabilities depend not only on model parameters but also on the executable \emph{harnesses} that organize information acquisition, tool use, reasoning, and memory. Recursive harness evolution can automate their design, but how it can reliably discover generalizable systems from limited, noisy, and repeatedly reused feedback remains unclear. We study this question in agentic video understanding, where rich multimodal evidence and coupled mechanisms make harness optimization challenging. As a testbed, we first build \basemethod{}, a baseline instantiation of the recursive self-improvement paradigm that evolves video harnesses around a frozen vision--language model. To diagnose it, we develop an analysis that decomposes each validation verdict into true gain, validation fit, champion inflation, and noise, and isolates each term with controlled experiments. Applied to over 4,000 evolved candidates, it shows that \basemethod{} is competitive with hand-engineered systems but reveals three interconnected failure modes: \emph{validation overfitting}, \emph{search-budget dilution}, and the \emph{early-experience trap}. All three share a common root: the same imperfect validation evidence drives both harness revisions and the proposer's search strategy. We therefore propose \method{}, an evidence-governed framework in which every search decision rests on evidence that is transferable, sufficient, and gathered before the search commits. It restricts the proposer to content-free feedback, informs exploration with instrumented warm-up experiments, and develops mechanisms in isolated probe--race branches before factorial composition and joint refinement. Across four video-understanding benchmarks, \method{} improves the equally weighted mean held-out score by $6.2\%$ relative to \basemethod{} and reduces the mean validation--test gap by $37.9\%$. Reliable harness evolution thus depends on how search is organized and informed, not only on how well it optimizes validation performance.
}

\date{\today}

\metadata[Project Page]{\url{https://zihaolinqz.github.io/VHarness-SRSI/}}

\begin{document}

\maketitle

\section{Introduction}
\label{sec:introduction}

An agentic system is shaped not only by its foundation model but by its \emph{harness}: the executable code around the frozen model that decides what the model observes, which tools run, and how intermediate outputs become an answer~\citep{ning2026codeagentharness}. A well-designed harness can substantially improve performance~\citep{lee2026meta}, but designing one is hard, since each mechanism admits many implementations whose effectiveness hinges on other components. Automated agent design and recursive harness evolution replace manual engineering with an LLM proposer that iteratively generates, evaluates, and refines harnesses from accumulated feedback~\citep{hu2025automated,shang2025agentsquare,lee2026meta}. Yet, as harnesses become more complex, a fundamental question arises: \emph{How can recursive evolution reliably discover generalizable designs when experimental evidence is limited, noisy, and repeatedly reused, and the search space contains many interacting mechanisms?}

We study this question through \textit{agentic video understanding}, a particularly demanding setting for harness optimization where a video harness must organize extensive, heterogeneous visual evidence via frame selection, retrieval, perception tools, multi-turn reasoning, reflection, memory, and multi-agent collaboration. These mechanisms are interdependent---retrieving more frames may recover missing evidence but also add distractors---so no mechanism can be judged in isolation. Moreover, evolution must infer useful changes from repeated evaluations on finite video--question pairs, whose rich content invites the proposer to specialize to seen examples. Video thus couples two challenges: learning reliable lessons from complex feedback, and efficiently exploring a combinatorial space of interacting mechanisms~\citep{wang2025videotreeadaptivetreebasedvideo, long2025vistatesttimeselfimprovingvideo, liang2025univauniversalvideoagent}.

To study these challenges, we first build \basemethod{} (Baseline Recursive Self-Improvement), which builds on Meta-Harness~\citep{lee2026meta} and adapts it to the video domain, and which was developed independently of and concurrently with VideoHarness-RSI~\citep{xu2026videoharnessrsi}. In \basemethod{}, a multimodal LLM coding proposer iteratively revises the executable harness around a frozen vision--language model, guided by accumulated multimodal evidence and execution feedback. 
Evolution discovers competitive harnesses: for example, in a Video-MME~\citep{fu2025videommefirstevercomprehensiveevaluation}
run evolving only retrieval, the evolved harness matches the held-out accuracy (68.4\%) of a manually designed Adaptive Keyframe Sampling (AKS) variant~\citep{tang2025adaptive} with $38.7\%$ fewer frames. 
Yet across 4,000+ evolved candidates, validation success does not reliably translate into generalization.
We trace this gap to a single source: the loop learns only from \emph{validation verdicts}---score changes on a finite, repeatedly reused validation set---which mix transferable improvements with instance-specific fit and noise, yet decide both which harness is kept and where the proposer searches next. This evidence misleads evolution in three ways. (1) \emph{Validation Overfitting} concerns what a verdict measures: instance-level multimodal feedback lets the proposer fit observed examples, so verdicts reward non-transferable gains. (2) \emph{Search-Budget Dilution} concerns how much evidence a verdict rests on: joint evolution spreads a limited budget over many mechanism combinations, and coupled edits change many answers at once, so each configuration is evaluated too rarely and noisily to attribute its gains. (3) \emph{Early-Experience Trap} concerns how verdicts are acted on: a single noisy early outcome persists as a judgment in the proposer's memory, so the proposer abandons useful mechanisms prematurely and over-explores early successes.
We identify these failure modes with a diagnostic analysis that decomposes each verdict into the true gain on unseen questions and three error terms---validation fit, champion inflation, and noise---and isolates each with controlled experiments that vary one factor of the loop at a time. Because \basemethod{} shares the propose-and-select logic of prior harness evolution, this analysis characterizes recursive harness evolution in general rather than our baseline alone.

Motivated by these findings, we further propose \method{} (Structured Recursive Self-Improvement), an evidence-governed framework built on one principle: every search decision should rest on evidence that reflects transferable quality, suffices for the decision, and is gathered before search commits. The proposer sees only aggregate scores and \emph{content-free execution statistics}, never the validation videos, questions, or answers, letting it diagnose harness behavior without tailoring revisions to seen data. Evolution then proceeds in stages: a \emph{warm-up} complements free proposals with \emph{controlled experiments} that ablate a mechanism or vary its intensity, measuring its contribution while keeping under-tested mechanisms open. This evidence seeds \emph{mechanism-specific branches}, each evolved in isolation from a shared starting harness under a \emph{probe--race} budget that tries every mechanism before focusing on the most promising. Finally, \emph{factorial composition} tests whether branch improvements transfer and how they interact, before \emph{joint refinement} co-optimizes them. Each stage thus supplies the evidence the next acts on.

We evaluate on four video question answering and temporal grounding benchmarks: Video-MME~\citep{fu2025videommefirstevercomprehensiveevaluation}, LongVideoBench~\citep{wu2024longvideobenchbenchmarklongcontextinterleaved}, OVO-Bench~\citep{niu2025ovo}, and MomentSeeker~\citep{yuan2026momentseeker}.
\method{} yields stronger, better-generalizing harnesses than \basemethod{}, raising the equally weighted mean held-out score by $6.2\%$ relative and cutting the mean validation--test gap by $37.9\%$; ablations validate its design.
These findings suggest that the organization and evidence quality of evolution matter beyond validation optimization.

Our main contributions are:
\begin{itemize}[noitemsep,topsep=0pt,leftmargin=*]
    \item \textbf{Diagnostic analysis.} A verdict-level analysis of recursive harness evolution that decomposes each validation verdict into true gain, validation fit, champion inflation, and noise, and isolates each term with controlled experiments. Applied to over 4,000 candidates of \basemethod{} in agentic video understanding, it shows that evolution discovers competitive harnesses and characterizes three failure modes rooted in shared imperfect validation evidence.

    \item \textbf{Framework.} We propose \method{}, an evidence-governed framework combining content-free proposer feedback, instrumented warm-up experiments, and isolated mechanism development before joint refinement.

    \item \textbf{Evaluation.} On four benchmarks, \method{} improves mean held-out score by $6.2\%$ relative and cuts the mean validation--test gap by $37.9\%$ over \basemethod{}.
\end{itemize}

\section{Related Work}
\label{sec:related_work}

\paragraph{Recursive self-improvement for harnesses.}
Automated optimization around frozen language models has progressed from prompts~\citep{fernando2023promptbreeder,yang2024large,pryzant2023automatic,yuksekgonul2024textgrad} to multi-stage programs~\citep{khattab2023dspy,opsahl2024optimizing}, reflective prompt evolution~\citep{agrawal2026gepa}, and executable workflows and agent architectures~\citep{hu2025automated,zhang2025aflow,zhang2025evoflow,novikov2025alphaevolve,zhang2026darwin}. Recent work evolves complete harnesses from past code, scores, and traces~\citep{lee2026meta}, searches modular design spaces~\citep{shang2025agentsquare,wu2026modularrsimodulargeneralizablerecursive}, regularizes edits against overfitting~\citep{xia2026rrsiregularizedrecursiveselfimprovement}, and studies harness updating~\citep{lin2026harness}, continual adaptation~\citep{liu2026adaptive}, and skill compilation~\citep{wei2026evo}. These methods refine how harnesses are proposed but take the validation signal that selects them as given. We instead make this signal the object of study: we show that it misleads evolution in three systematic ways and derive each component of \method{} from one of them. Rather than regularizing edits, we control what information reaches the proposer; rather than modularizing search for efficiency, we factorize it so that each mechanism accumulates sufficient evidence and factorial composition tests whether gains transfer; and we counter the path dependence of noisy early verdicts, which prior methods leave unaddressed, with instrumented warm-up experiments.

\paragraph{Recursive self-improvement for video harnesses.}
Concurrent work evolves video-agent pipelines toward a target video distribution~\citep{cui2026metavideoagent}, revises harnesses by re-examining training videos and selecting on accuracy and visual cost~\citep{luo2026videorsi}, and self-designs video tools through coupled solving and evolving loops~\citep{ye2026videoresearcher}. Most closely related, VideoHarness-RSI~\citep{xu2026videoharnessrsi} is among the first to study recursive harness self-improvement for long-video understanding and establishes a strong baseline for it. Its proposer edits the code that selects, retrieves, and packs visual evidence around a frozen VLM, keeping a candidate only if it improves development accuracy; the searched harnesses improve held-out accuracy and transfer to other long-video benchmarks. Our \basemethod{} is an independent implementation developed concurrently with VideoHarness-RSI, which studies a similar recursive self-improvement paradigm; we use it as a baseline and testbed rather than as a contribution. Where these works establish that self-improvement can strengthen video systems, we ask when the validation-driven verdicts of such search fail to generalize, diagnose three failure modes across over 4,000 candidates, and propose \method{} to organize search and the proposer's evidence so that discovered harnesses transfer.

\paragraph{Video understanding harnesses.}
Hand-designed video harnesses make long videos tractable through compression~\citep{shen2024longvu}, keyframe selection~\citep{tang2025adaptive}, hierarchical representation~\citep{wang2025videotreeadaptivetreebasedvideo}, and evidence retrieval~\citep{luo2025videoragvisuallyalignedretrievalaugmentedlong}, while agentic systems coordinate retrieval, tools, and memory~\citep{wang2024videoagent,zhang2026deep,yeo2026worldmm,long2025seeinglisteningrememberingreasoning}. Each commits to one design; we treat these mechanisms as a search space in which multimodal feedback invites specialization to validation instances and a mechanism's value hinges on its context.

\section{Investigation of Video Understanding Harness Evolution}
\label{sec:initial_investigation}

\begin{figure*}[t]
    \centering
    \includegraphics[width=0.99\textwidth]{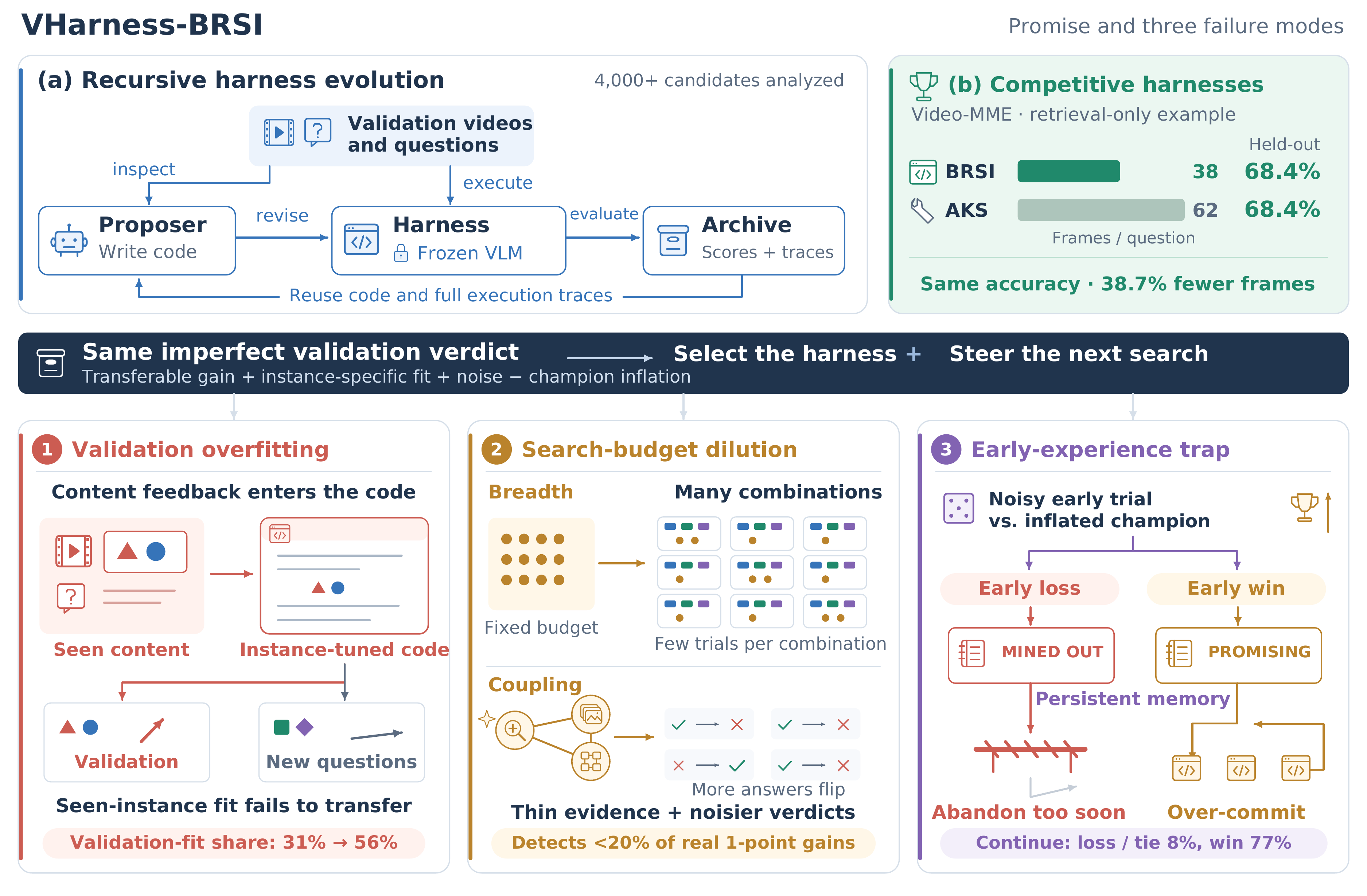}
    \caption{\textbf{Overview of \basemethod{}.} (a)~The evolution loop: a coding proposer inspects the validation videos and questions, revises the harness around a frozen VLM, and archives each candidate's validation scores and execution traces, which later iterations reuse. (b)~In a retrieval-only run on Video-MME, the evolved harness matches the held-out accuracy of the strongest hand-designed AKS variant with $38.7\%$ fewer frames (Section~\ref{sec:metaharness_effectiveness}). Bottom: the same validation verdict both selects the harness and steers the next search; it mixes the true gain with validation fit, noise, and champion inflation, which gives rise to the three failure modes of Section~\ref{sec:meta_harness_failure_modes}.}
    \label{fig:method_overview}
\end{figure*}

Building on prior work in recursive harness self-improvement and hand-designed video understanding harnesses, we first introduce \basemethod{}, a baseline recursive self-improvement framework for agentic video understanding, to understand how harness evolution behaves in this more complex setting. \basemethod{} builds on Meta-Harness~\citep{lee2026meta} and adapts it to the video domain. It is an independent implementation, developed concurrently with VideoHarness-RSI~\citep{xu2026videoharnessrsi}, which studies a similar recursive self-improvement paradigm. We use \basemethod{} as a baseline and testbed rather than claim it as a contribution. As shown in Figure~\ref{fig:method_overview}(a), an agentic coding proposer $P$ evolves executable harnesses around a frozen VLM $M$ and directly accesses the multimodal evidence behind each outcome: the validation videos and questions $\mathcal{I}_{\mathrm{val}}$, prior implementations, and execution traces. This full access follows the design of Meta-Harness, extended to video. The loop is a simple propose--evaluate--accumulate cycle: starting from an archive $\mathcal{D}_0$ of seed harnesses, iteration $t$ proposes candidates $H_{t,j}\sim P(\cdot\mid\mathcal{D}_{t-1},\mathcal{I}_{\mathrm{val}})$, evaluates each on the validation set $\mathcal{S}_{\mathrm{val}}$ to obtain a record $\mathcal{E}_{t,j}$ of scores, frame counts, and traces, and adds $(H_{t,j},\mathcal{E}_{t,j})$ to the archive $\mathcal{D}_t$. 
Unless otherwise specified, $M$ is Qwen3-VL-8B-Instruct~\citep{bai2025qwen3vltechnicalreport} and $P$ is Claude Opus~5~\citep{anthropic2026claudeopus5}. \basemethod{} is specified in Appendix~\ref{app:basemethod}.

\subsection{Can \basemethod{} Evolve Effective Video Harnesses?}
\label{sec:metaharness_effectiveness}
We first examine whether recursive evolution can discover harness designs competitive with expert-engineered alternatives. A fair test needs a strong hand-designed reference for the same mechanism, so we evolve only the \emph{retrieval strategy} and compare with Adaptive Keyframe Sampling (AKS)~\citep{tang2025adaptive}, a state-of-the-art retrieval method; evolved and manual harnesses then differ only in how that mechanism is designed. We run on Video-MME~\citep{fu2025videommefirstevercomprehensiveevaluation}, a widely used benchmark on which AKS was evaluated. We use 450 questions for validation and the remaining 2{,}250 for held-out evaluation.

\begin{figure*}[t]
    \centering
\includegraphics[width=0.99\textwidth]{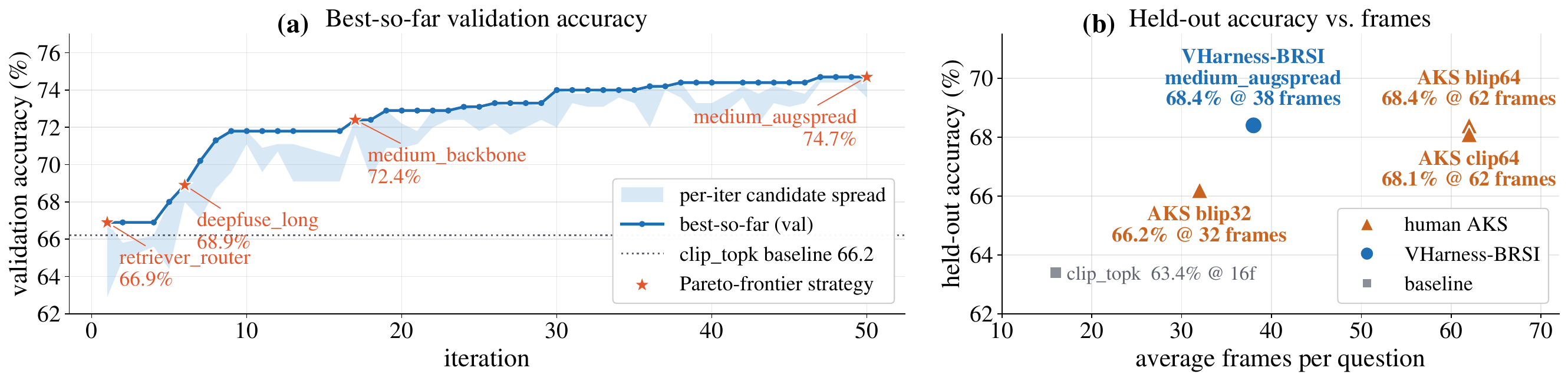}
    \caption{\textbf{Evolving the retrieval strategy with \basemethod{} on Video-MME.} (a)~Best validation accuracy found so far over 50 iterations (line) and the range of the candidates proposed in each iteration (shaded), starting from the \texttt{clip\_topk} seed (dotted). Stars mark the strategies on the Pareto frontier; the red labels are their names, chosen autonomously by \basemethod{}. (b)~Held-out accuracy on 2{,}250 questions never seen during evolution against the mean number of frames per question, for the deployed evolved harness, three variants of the hand-designed AKS method, and the seed.}
    \label{fig:initial_metaharness}
\end{figure*}

\paragraph{\textbf{Finding 1: Recursive evolution discovers harnesses competitive with manual designs.}}
\emph{Evolution steadily improves the harness.} Figure~\ref{fig:initial_metaharness}(a) tracks the best validation accuracy found so far. Starting from the \texttt{clip\_topk} seed at $66.2\%$, \basemethod{} improves it in several distinct steps rather than a single lucky jump: to $66.9\%$ in the first iteration, $68.9\%$ around iteration~6, $72.4\%$ around iteration~17, and finally $74.7\%$ at iteration~50 with \texttt{medium\_augspread}, $8.5$ points above the seed. Validation accuracy thus keeps rising throughout the 50 iterations.
\emph{The evolved harness rivals manual design on held-out data.} Figure~\ref{fig:initial_metaharness}(b) compares the deployed harness with the seed and three AKS variants on the 2{,}250 held-out questions, which evolution never sees. \texttt{medium\_augspread} reaches $68.4\%$ with 38 frames per question: $5.0$ points above its seed, $2.2$ points above AKS-BLIP with 32 frames, and on par with the two strongest AKS variants ($68.4\%$ and $68.1\%$), which need 62 frames. It thus matches the best hand-designed method with $38.7\%$ fewer frames. Evolution also improves other mechanism families and works on other benchmarks. The setting, the seed harnesses (Table~\ref{tab:seed_harnesses}), and these additional experiments are detailed in Appendix~\ref{app:evolving-capability}.
\begin{keyfinding}{Takeaway}
\textbf{Recursive evolution discovers video harnesses that rival expert designs.} Starting from a simple retrieval seed, \basemethod{} raises validation accuracy by $8.5$ points and deploys a harness that matches the strongest AKS variant on held-out data ($68.4\%$) with $38.7\%$ fewer frames.
\end{keyfinding}

\subsection{Diagnosing \basemethod{} in Video Understanding}
\label{sec:meta_harness_failure_modes}

Despite these gains, \basemethod{} shows three systematic failure modes, as shown in Figure~\ref{fig:diagnosis}. Since our baseline shares the core logic of prior harness evolution---propose new candidates from past ones and their scores, then keep those that win on validation~\citep{hu2025automated,zhang2025aflow,lee2026meta}---these failures expose limitations of this general strategy rather than of our baseline. We explain all three through the one signal the evolution loop learns from, the \emph{validation verdict} $\hat{\delta}=V(H)-V(H^{\mathrm{inc}})$: how much a candidate $H$'s validation accuracy exceeds that of the current champion $H^{\mathrm{inc}}$. This verdict decides both what is kept and where the proposer searches next. What we actually want to improve, however, is accuracy on unseen questions, so the quantity that matters is the candidate's \emph{true gain} $\Delta$: how much more accurate it is than the champion on new questions. Evolution works only as far as the verdict $\hat{\delta}$ reflects the true gain $\Delta$. To make the relation between the two explicit, we write each validation score as the sum of three parts, $V(H)=\mu(H)+S(H)+e(H)$: the accuracy $\mu(H)$ the harness would reach on new questions, a fit $S(H)$ that helps only on the validation questions (zero for harnesses written without seeing them), and a zero-mean noise $e(H)$, the chance deviation of its accuracy on this particular set of validation questions from $\mu(H)$. The true gain is then $\Delta=\mu(H)-\mu(H^{\mathrm{inc}})$. Subtracting the two scores would recover $\Delta$ up to the fit and noise terms, if the champion's noise were zero on average; but it is not, because the champion is not a random harness: it is the candidate with the highest validation score so far. Among many candidates of similar quality, the highest score tends to belong to one whose noise happened to be positive, so the champion's noise is positive on average: $\mathbb{E}[e(H^{\mathrm{inc}})]=\omega>0$. We call $\omega$ the \emph{champion inflation}; it is the same as the selection bias known as the optimizer's curse~\citep{smith2006optimizer}, here arising because the champion is selected by its validation score. Every new candidate is compared against this inflated score, so its verdict is shifted down by $\omega$. Separating this bias from the noise difference, i.e., writing $e(H)-e(H^{\mathrm{inc}})=-\omega+\varepsilon$ with the zero-mean remainder $\varepsilon$, splits the verdict into four terms:
\begin{equation}
    \hat{\delta}
    =\underbrace{\mu(H)-\mu(H^{\mathrm{inc}})}_{\Delta:\ \text{true gain}}
    +\underbrace{S(H)-S(H^{\mathrm{inc}})}_{S:\ \text{validation fit}}
    -\underbrace{\omega}_{\text{champion inflation}}
    +\underbrace{\varepsilon}_{\text{noise}} .
    \label{eq:verdict}
\end{equation}
Only the true gain $\Delta$ carries over to held-out data; the other three terms are errors that make the verdict misstate it. The zero-mean noise $\varepsilon$ has a standard deviation $\sigma$ that grows with the number of answers the candidate changes. Evolution therefore goes wrong whenever one of the three error terms outweighs $\Delta$, and each failure mode makes a different one do so: validation overfitting inflates $S$ (Section~\ref{sec:overfitting}), search-budget dilution inflates $\varepsilon$ (Section~\ref{sec:search_space_efficiency}), and the early-experience trap lets $\varepsilon$ and $\omega$ decide irreversible choices (Section~\ref{sec:early_experience_trap}). Appendix~\ref{app:verdict_model} derives this model.

\subsubsection{Validation Overfitting}
\label{sec:overfitting}
We start with the validation fit $S$ in Eq.~(\ref{eq:verdict}). A harness is written by the proposer, so it can fit the validation questions only through what the proposer learns about them. \basemethod{}, following Meta-Harness, lets the proposer inspect the validation videos, questions, and execution traces directly. If this access drives the fit, the gap between validation and held-out accuracy should grow with how much the proposer observes. We therefore keep the evolution loop fixed and vary only what the proposer sees, in three nested arms from least to most: \emph{Blindfolded} sees only outcomes, i.e., whether each question was answered correctly, plus behavioral counters such as frames and tool calls; \emph{Trace} additionally sees question text, answer options, and interaction traces; and \emph{Multimodal} additionally sees video frames, which recovers the full access of \basemethod{}. We run this study on LongVideoBench~\citep{wu2024longvideobenchbenchmarklongcontextinterleaved}, whose questions refer to specific moments in long videos, so seeing a question reveals much about what it needs. Each arm runs three seeds at each of three validation fractions ($33\%$, $50\%$, $75\%$), so the effect does not hinge on one split size.

\begin{figure}[t]
    \centering
    \setlength{\tabcolsep}{0pt}
    \begin{tabular}{@{}ccc@{}}
        {\small\textbf{Validation overfitting}} & {\small\textbf{Search-budget dilution}} & {\small\textbf{Early-experience trap}} \\[1pt]
        \includegraphics[width=0.33\textwidth]{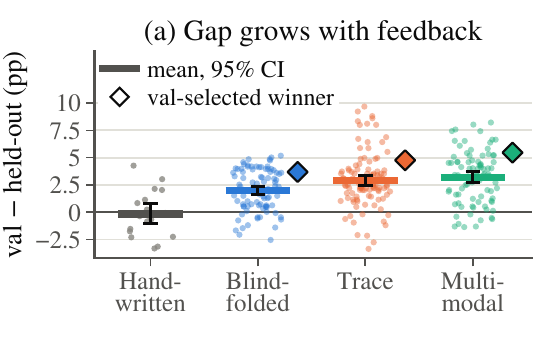} &
        \includegraphics[width=0.33\textwidth]{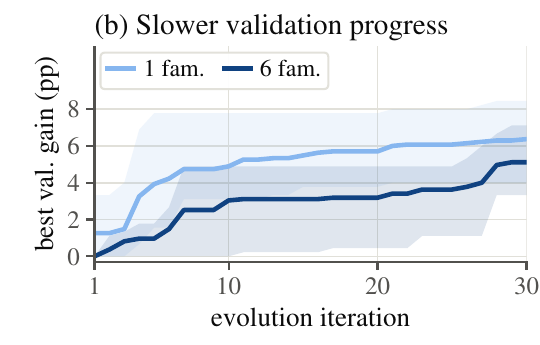} &
        \includegraphics[width=0.33\textwidth]{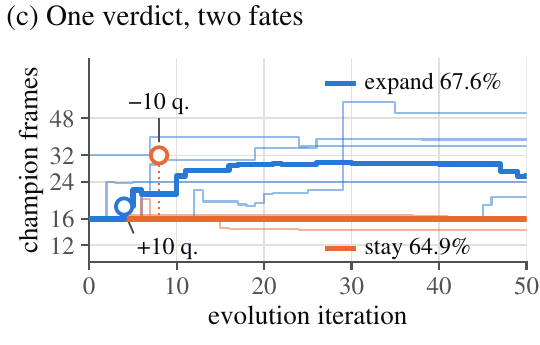} \\
        \includegraphics[width=0.33\textwidth]{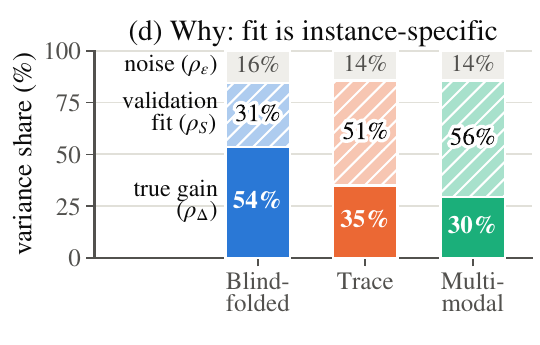} &
        \includegraphics[width=0.33\textwidth]{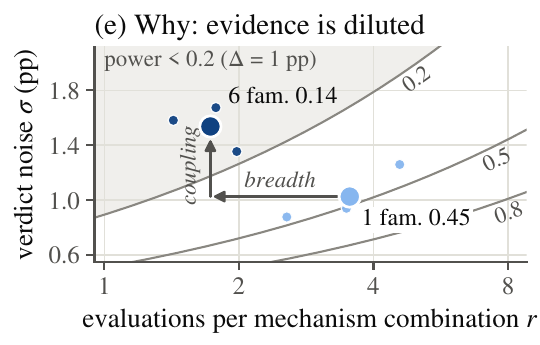} &
        \includegraphics[width=0.33\textwidth]{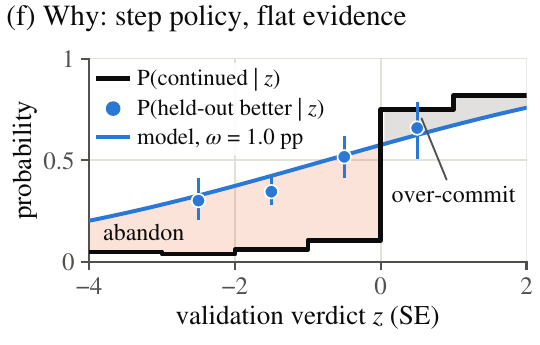} \\
    \end{tabular}
    \caption{\textbf{Three failure modes of \basemethod{}: what happens (top) and why (bottom).} (a)~Validation minus held-out accuracy of hand-written baselines and of harnesses evolved under the three feedback arms on LongVideoBench; each dot is a harness, bars show means with 95\% intervals, and diamonds mark each arm's validation-selected winner. (b)~Best validation gain over the seeds on Video-MME when evolving retrieval only (one family) or the full harness (six families); lines are means and bands span the three runs. (c)~Champion frame count of retrieval runs on Video-MME. Thick lines are the two runs discussed in the text, with their frame-raising probes circled ($+10$ and $-10$ validation questions); the labels give the mean held-out accuracy of runs that expand the frame budget and of runs that stay near 16 frames. (d)~Shares of the spread in validation scores due to true gain ($\rho_\Delta$), validation fit ($\rho_S$), and noise ($\rho_\varepsilon$) in each feedback arm. (e)~Each run placed by its evaluations per mechanism combination $r$ and verdict noise $\sigma$; small and large dots are runs and search-space means, and contours mark equal power to detect a 1-point gain, shaded below $0.2$. (f)~Against the verdict $z$ (in standard errors), the probability that later proposals build on a candidate (black) and that the candidate beats the champion on held-out data (blue dots, with the model curve); shading marks abandoned candidates that were likely better (red) and continued ones that were likely worse (gray).}
    \label{fig:diagnosis}
\end{figure}

\paragraph{\textbf{Finding 2: Richer feedback widens the validation--test gap by letting harnesses fit the validation questions.}}
We call a harness's validation accuracy minus its held-out accuracy its \emph{gap}. As shown in Figure~\ref{fig:diagnosis}(a), the more the proposer sees, the larger the gap: it grows from $2.0$ points (Blindfolded) to $3.2$ points (Multimodal) on average, whereas hand-written baselines, which are written without seeing validation data, show none.
The gap alone, however, does not tell how much of the validation \emph{progress} is real. We therefore look at all evolved harnesses of a run together. By Eq.~(\ref{eq:verdict}), their validation scores differ for three reasons: true gain ($\Delta$), validation fit ($S$), and noise ($\varepsilon$); $\omega$ plays no role here, since all of them face the same champion. Only the true gain also shows up on held-out data, so we can attribute the spread of validation scores to the three sources. The \emph{true-gain share} $\rho_\Delta$ is how much of a validation difference reappears on held-out data: a harness one point higher on validation is, on average, $\rho_\Delta$ points higher on held-out data. The \emph{noise share} $\rho_\varepsilon$ is measured by scoring each harness on two random halves of the validation set, which differ only by noise. The rest is the \emph{validation-fit share} $\rho_S$. As shown in Figure~\ref{fig:diagnosis}(d), $\rho_\Delta$ falls from $54\%$ (Blindfolded) to $35\%$ (Trace) and $30\%$ (Multimodal), while $\rho_\varepsilon$ stays at $14$--$16\%$, so the loss goes entirely to $\rho_S$, which rises from $31\%$ to $56\%$. Richer feedback thus does not make the validation score noisier; it makes it reward fitting the validation questions rather than real improvement. The fit is also visible in the code itself: proposers that see question text quote validation questions in their harnesses more often than chance, mostly in comments that record the cases motivating a design, whereas Blindfolded proposers do not. The protocol, the estimators of the three shares, the supporting theory, and robustness checks are in Appendix~\ref{app:overfitting_analysis}.
\begin{keyfinding}{Takeaway}
\textbf{Hide the validation instances from the proposer.} The more the proposer sees of the validation questions, the more of its progress is fit to those questions rather than real improvement ($56\%$ versus $31\%$ of the validation spread). It should see \emph{how} a harness behaves, not \emph{what} it is evaluated on.
\end{keyfinding}

\subsubsection{Search-Budget Dilution}
\label{sec:search_space_efficiency}
We next turn to the noise $\varepsilon$. Video harnesses combine mechanisms with distinct roles: frame selection decides which moments the model sees~\citep{tang2025adaptive}, retrieval gathers question-relevant evidence~\citep{luo2025videoragvisuallyalignedretrievalaugmentedlong}, tools extend perception~\citep{wang2024videoagent,zhang2026deep}, and memory carries information across long videos~\citep{yeo2026worldmm,long2025seeinglisteningrememberingreasoning}. Letting evolution edit more of these mechanisms enlarges the search space and can, in principle, reach higher accuracy. But a larger space also spreads a fixed evaluation budget over more \emph{mechanism combinations}, i.e., over more distinct choices of which mechanisms a harness uses and how, and a real gain can be told apart from noise only after enough verdicts on it accumulate to average the noise out. Whether the added expressiveness or the thinner evidence dominates is thus an empirical question. Keeping everything else fixed, we compare on Video-MME the smallest space, retrieval only (one mechanism family, as in Section~\ref{sec:metaharness_effectiveness}), with the full harness (six families), at a matched budget of 87 candidates and three seeds each.

\paragraph{\textbf{Finding 3: Larger search spaces leave too little evidence to tell real gains from noise, slowing evolution without improving held-out accuracy.}}
As shown in Figure~\ref{fig:diagnosis}(b), the larger space is not better: after 87 candidates, retrieval alone gains $6.4$ validation points, whereas the full harness gains only $5.1$ and lags for most of the run. Its deployed champions are also no better on held-out data, so the added expressiveness goes unused. To see why, recall that a single verdict is noisy: the loop can trust that a mechanism combination truly helps only after several candidates using it have been evaluated. We therefore measure each run's statistical \emph{power}, the chance that its verdicts detect a real 1-point gain. This is a lenient test, since 1 point is already larger than most improvements the loop accepts. Power depends on two quantities: how many evaluations $r$ each mechanism combination receives, since more evaluations average the noise out, and how noisy each verdict is ($\sigma$), since more noise hides the gain. A larger space can lower power through two channels. \emph{Breadth} spreads the same budget over more mechanism combinations, which lowers $r$. \emph{Coupling} lets mechanisms feed into one another: a change to retrieval, for example, alters what every downstream stage receives, so even a local edit flips the correctness of more answers, and the more answers an edit flips, the noisier its verdict, which raises $\sigma$. As shown in Figure~\ref{fig:diagnosis}(e), where each run is placed by its $r$ and $\sigma$ and the curves mark equal power, both channels are at work: from one family to six, $r$ falls from $3.5$ to $1.7$, moving the runs left, and $\sigma$ rises from $1.0$ to $1.5$ points, moving them up. Together they push the full harness into the shaded region of power below $0.2$: its search misses more than four in five real 1-point improvements. Coupling is also visible in the edits themselves: even edits that touch a single mechanism family flip $47.6$ validation answers on average in the full harness, against $19.8$ with retrieval alone, and edits that flip more answers are more often harmful. Intermediate spaces with two and three families, the power analysis, and edit footprints are in Appendix~\ref{app:search_space_analysis}.
\begin{keyfinding}{Takeaway}
\textbf{Develop mechanisms separately before combining them.} Searching all mechanisms jointly leaves each verdict too noisy to detect most real gains (power below $0.2$). Concentrating evaluations on one mechanism at a time, with local edits, yields verdicts that can be trusted before interactions are optimized.
\end{keyfinding}

\subsubsection{The Early-Experience Trap}
\label{sec:early_experience_trap}
The first two failure modes concern whether a verdict is right; this one concerns what the evolution loop does with it. In \basemethod{}, a verdict decides not only which candidate is kept but also which \emph{direction} the proposer pursues next, where a direction is a kind of change, such as raising the frame budget or adding re-ranking. Early verdicts therefore shape the rest of a run. For each verdict, we ask two questions. How did the loop react: do later proposals build on the candidate (\emph{continuation})? And was the verdict right: does the candidate beat the champion on held-out data, which the proposer never sees? We analyze 14 Video-MME runs; 12 of them evolve retrieval and form the frame-budget contrast below.

\paragraph{\textbf{Finding 4: A single noisy early verdict can fix a run's course, as weak verdicts drive irreversible decisions.}}
Figure~\ref{fig:diagnosis}(c) traces each run's champion frame count: the same kind of probe can meet opposite verdicts, and the first verdict decides the rest of the run. In two runs of the same setting (thick lines), a probe raising the frame budget (circled) wins by ten validation questions in one run and loses by ten in the other. The first run then keeps raising the frame budget; the second writes into its memory that the frame count is ``MINED OUT'', i.e., not worth exploring further, and its champion never exceeds 16 frames. This single verdict matters: across 12 runs, champions that expand reach $67.6\%$ held-out accuracy, versus $64.9\%$ for those staying near 16 frames. Eq.~(\ref{eq:verdict}) explains why such a verdict misleads: a direction's first verdict is $\Delta-\omega+\varepsilon$, and both error terms are large. The noise $\varepsilon$ dominates early verdicts: in the first ten iterations, only $8\%$ of their variation reflects true differences, and the champion inflation starts every new direction about $\omega\approx1$ point behind. A loop that weighed such weak evidence properly would treat verdicts as graded: the better a verdict, the more likely the direction is continued, without a sharp jump at zero. Figure~\ref{fig:diagnosis}(f) shows that \basemethod{} does not. Against the verdict $z$, measured in standard errors, it plots how often later proposals build on a candidate (black) and how often the candidate truly beats the champion on held-out data (blue). The loop acts as a step function, continuing about $77\%$ of winning candidates but only $8\%$ of the rest, whereas the held-out win rate rises gradually and already reaches one half for slightly negative verdicts. In fact, among dropped candidates with held-out evaluations, $37\%$ beat the champion on held-out data; the shaded gap marks these abandoned candidates. Because a rejected direction is rarely retried, an unlucky first verdict is never corrected. A simple model of this policy predicts that the frame-budget direction is lost in $54\%$ of runs, consistent with the 6 of 12 runs in Figure~\ref{fig:diagnosis}(c) that never expand. The verdict model, the persistence analysis, a replication on LongVideoBench, and case studies are in Appendix~\ref{app:early_experience_trap}.
\begin{keyfinding}{Takeaway}
\textbf{Do not let one noisy verdict close a direction.} Early verdicts are mostly noise, yet the loop treats them as final: one unlucky outcome can rule out a useful mechanism for the rest of the run ($64.9\%$ versus $67.6\%$ held-out for the frame budget). Early evidence should come from controlled comparisons collected regardless of earlier outcomes.
\end{keyfinding}

\subsection{Implications for Harness Evolution}
\label{sec:implications}

\paragraph{One root, three symptoms.}
The three failure modes share one root: the loop learns from validation verdicts, and the same imperfect verdicts steer both which harness is kept and where the proposer searches next. Each failure mode lets a different term of Eq.~(\ref{eq:verdict}) overwhelm the true gain $\Delta$: the validation fit $S$ under content-level feedback, the noise $\varepsilon$ in a large and coupled search space, and the champion inflation $\omega$ together with $\varepsilon$ in early, irreversible decisions. The three also reinforce one another: a misleading early win narrows the search, a broad space leaves little budget to revisit it, and each further verdict rests on too few evaluations to correct the course.

\paragraph{Simple remedies are insufficient.}
Standard evolutionary levers act on where the next verdict is taken, not on what it measures. Changing how parents are selected reshapes the search as intended: on Video-MME, the share of candidates bred from the champion falls from $47\%$ under greedy selection to $30\%$ under clade-UCB and $8\%$ under random parents. Yet the deployed harnesses of the three rules differ by at most $0.5$ held-out points, and no statistic of the search shape predicts held-out gain across 18 runs. Starting from a stronger harness, AKS, speeds up early validation progress, but the final harness ($68.2\%$ held-out) lies within the range of those evolved from simple seeds. In both cases, every verdict still compares a single candidate with a selected champion on the same validation questions, so $S$, $\omega$, and $\varepsilon$ are unchanged (Appendix~\ref{sec:simple_remedies}).

\paragraph{Design principles.}
The three takeaways instead act on the evidence itself, and together they form the design principles of \method{} (Section~\ref{sec:method}): every search decision should rest on evidence that is \emph{transferable}, \emph{sufficient}, and \emph{gathered before search commits}. Against validation fit, the proposer sees only content-free feedback; against dilution, mechanisms are developed in isolated branches that concentrate evaluations and keep edits local, and only then composed; and against the early-experience trap, an instrumented warm-up measures every mechanism with controlled comparisons before any direction is abandoned.

\section{\method{}}
\label{sec:method}
Section~\ref{sec:meta_harness_failure_modes} traced all three failure modes to the validation verdict that \basemethod{} learns from: rich feedback lets it reward fitting the validation questions ($S$), a large search space leaves it too noisy to separate real gains ($\varepsilon$), and noisy early verdicts close directions for good ($\varepsilon$ and $\omega$). \method{} therefore keeps the propose--evaluate--accumulate loop of \basemethod{} and changes only the evidence behind its decisions (Figure~\ref{fig:srsi_overview}). Each component implements one takeaway: \emph{information isolation} limits what the proposer can fit (Section~\ref{sec:method_setup}), and \emph{three-stage evolution} concentrates evidence on one mechanism at a time (Section~\ref{sec:three_stage_evolution}), beginning with an \emph{instrumented warm-up} that keeps early directions open (Section~\ref{sec:method_warm}). The full algorithm (Algorithm~\ref{alg:method}) and further details are in Appendix~\ref{app:method_details}.

\begin{figure*}[t]
    \centering
    \includegraphics[width=0.99\textwidth]{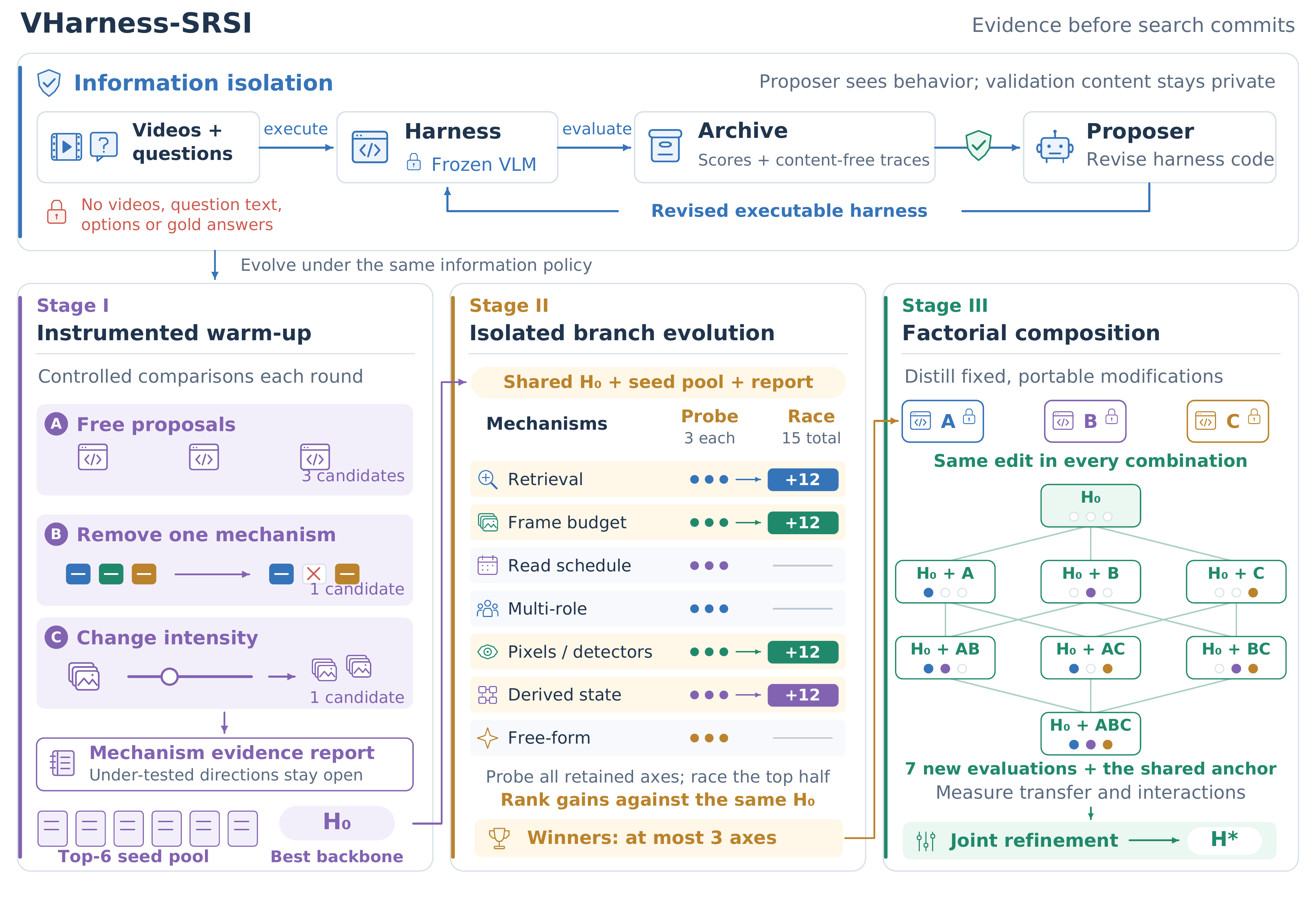}
    \caption{\textbf{Overview of \method{}.} Top: under information isolation, the harness executes on the validation videos and questions, but the proposer receives only scores and content-free traces, never the videos, question text, options, or answers. Bottom: evolution proceeds in three stages under this policy. Stage~I (instrumented warm-up) adds a removal and an intensity change to the free proposals of each round, summarizes the evidence in a per-mechanism report, and selects a seed pool and a backbone $H_0$. Stage~II (isolated branch evolution) evolves each mechanism axis in its own branch from $H_0$; every axis is probed with three candidates, the top half race on to fifteen, and at most three axes are kept. Stage~III (factorial composition) distills each kept axis into a fixed modification, evaluates all combinations against $H_0$ to measure transfer and interactions, and then jointly refines them into the final harness $H^\star$.}
    \label{fig:srsi_overview}
\end{figure*}

\subsection{Information Isolation}
\label{sec:method_setup}

Finding~2 demonstrates that the validation fit grows with what the proposer observes about the validation instances, and that the Blindfolded arm kept the largest transferable share of its progress ($54\%$) while deploying harnesses at least as accurate on held-out data as the richer arms. Since a harness can fit the validation questions only through what the proposer learns about them, we close this channel at its source instead of penalizing edits afterwards, when the information has already been used. In \method{}, the proposer never accesses the validation videos, questions, options, or answers; only the harness it writes does, at execution time. The proposer receives each candidate's validation accuracy and a \emph{content-free trace} $\mathcal{T}(H)$: per question, its correctness under an opaque identifier plus behavioral counters such as model and tool calls, frames, retrieval, and memory operations. This is exactly the Blindfolded feedback, so isolation lowers the fit without starving the search. A deny list on all proposer tools enforces it; details are in Appendix~\ref{app:method_isolation}.

\subsection{Three-Stage Evolution}
\label{sec:three_stage_evolution}

Finding~3 shows that searching the full harness jointly loses power in two ways: \emph{breadth} spreads the budget over many mechanism combinations, and \emph{coupling} lets each edit change many answers. Restricting the search to a single mechanism would avoid both, but forfeits the expressiveness of the full space. We instead \emph{order} the search: first collect controlled evidence about every mechanism (Stage~I), then improve each mechanism separately, where evaluations concentrate and edits stay local (Stage~II), and only then optimize their interactions among a few mechanisms whose effects are already measured (Stage~III). The search thus keeps the full space while each decision gets the evidence of a small one.

\subsubsection{Stage I: Instrumented Warm-up}
\label{sec:method_warm_stage}
\label{sec:method_warm}

Finding~4 reveals that \basemethod{} gathers evidence only where earlier verdicts were favorable, so one unlucky verdict can close a beneficial direction for the rest of a run. The warm-up runs the ordinary loop but reserves part of each round for \emph{controlled interventions} that are collected regardless of earlier verdicts. Besides three free proposals, each round includes one \emph{removal}, which deletes one of the seven mechanism families into which \method{} groups the harness primitives (a finer grouping than the families chartered for \basemethod{}; Appendix~\ref{app:method_warmup}), and one \emph{intensity} change, which roughly doubles or halves how much one family is used, such as the retrieval budget, the number of frames a detector inspects, or the number of reflection passes. For each warm-up candidate $H$ with parent $p(H)$, we record its verdict and its change in mechanism usage,
\begin{equation}
    \mathbf{w}(H)=\bigl(\hat{\delta}(H,p(H)),\,
    \mathbf{u}(H)-\mathbf{u}(p(H))\bigr),
    \label{eq:method_warm_evidence}
\end{equation}
where $\mathbf{u}(H)$ summarizes runtime mechanism usage from $\mathcal{T}(H)$; slot requests are audited from this usage change rather than trusted. Evidence on each mechanism thus accumulates over many comparisons instead of resting on its first verdict, and because all removals are judged the same way, the champion's inflation shifts them alike and does not distort how mechanisms compare. The warm-up summarizes this evidence in a per-mechanism \emph{warm-up report}; mechanisms without enough evidence are flagged as under-explored rather than discarded. When accuracy plateaus, the six best warm-up harnesses form a seed pool $\mathcal{B}$, and the best of them becomes the \emph{backbone} $H_0$ shared by all later stages, so later stages draw on several strong lineages rather than a single, possibly lucky, champion (ablation in Section~\ref{sec:exp:warmablation}). The mechanism families, slot audit, stopping rule, and warm-up report are detailed in Appendix~\ref{app:method_warmup}.

\subsubsection{Stage II: Isolated Branch Evolution}
\label{sec:method_branches}

To remove coupling, Stage~II splits the harness into six mechanism axes---\emph{retrieval}, \emph{frame budget}, \emph{read schedule}, \emph{multi-role}, \emph{pixels and detectors}, and \emph{derived state}---plus a free-form axis that guards against a misspecified split, giving the axis set $\mathcal{A}=\{\alpha_1,\ldots,\alpha_6,\alpha_{\mathrm{free}}\}$. Each axis $\alpha\in\mathcal{A}$ evolves from $H_0$ in its own isolated branch, and a fixed template states which parts of the harness the branch may edit and a test that every edit must pass (e.g., only frame-budget edits may change the number of frames). Edits therefore stay local, and a branch's gains are attributable to its mechanism. Templates are fixed in advance, so the choice of axes cannot itself fit the validation set; axes defined after the warm-up generalize worse (Section~\ref{sec:exp:axesablation}). From the warm-up report, the proposer writes each axis a \emph{charter} $c_{\alpha}$: the mechanism to try, a first edit, and a criterion that would falsify it.

To counter breadth, a \emph{probe--race} schedule concentrates evaluations without leaving any axis untried: every axis is probed with three candidates, the top half by best probe verdict against $H_0$ race on to fifteen, and at most three raced axes, those with a clearly positive mean verdict first, proceed to composition. Since all axes share the backbone $H_0$, the champion's inflation shifts them alike and leaves their ranking intact. More details are given in Appendix~\ref{app:method_branch}.

\subsubsection{Stage III: Factorial Composition and Joint Refinement}
\label{sec:method_composition}

Gains found in separate branches need not add up. Greedy merging cannot tell whether they do. The proposer therefore distills each selected axis $\alpha\in\mathcal{A}_{\mathrm{comp}}$ into a fixed, portable modification $d_{\alpha}$, the smallest edit of $H_0$ carrying the mechanism. We evaluate every combination as a factorial design:
\begin{equation}
    H_{\mathcal{U}}=\operatorname{Apply}\bigl(H_0,\{d_{\alpha}:\alpha\in\mathcal{U}\}\bigr),
    \quad \mathcal{U}\subseteq\mathcal{A}_{\mathrm{comp}},
    \quad H_{\emptyset}=H_0.
    \label{eq:method_composition}
\end{equation}
With three axes this costs seven evaluations, yet each modification's effect is averaged over four comparisons rather than resting on one, and the same comparisons measure interactions. Joint refinement then starts from all combinations and may modify the selected mechanisms together: coupling returns, but only among at most three mechanisms whose effects are already measured. The final harness $H^{\star}$ is the validation-best one after refinement; this last selection on validation reintroduces some fit and champion inflation, a source of the residual gaps in Section~\ref{sec:main_results}. The composition lattice, interaction estimates, and refinement are detailed in Appendix~\ref{app:method_compose}.

\section{Experiments}
\label{sec:experiments}

We evaluate \method{} against fixed hand-written harnesses and \basemethod{} on four video-understanding benchmarks (Section~\ref{sec:main_results}), and ask whether its gains come from a larger inference budget (Section~\ref{sec:exp:budget}) and which components drive them (Section~\ref{sec:exp:ablations}). Appendix~\ref{app:unseen_videos} further verifies that the gains hold on videos never seen during evolution, and Appendix~\ref{app:method_diagnostics} re-measures the three failure modes of Section~\ref{sec:meta_harness_failure_modes} on the runs of \method{}.

\subsection{Experimental Setup}
\label{sec:exp:setup}

\paragraph{Benchmarks and splits.}
We use four benchmarks covering general, long-form, and streaming video question answering and temporal grounding: Video-MME~\citep{fu2025videommefirstevercomprehensiveevaluation}, LongVideoBench~\citep{wu2024longvideobenchbenchmarklongcontextinterleaved}, and OVO-Bench~\citep{niu2025ovo}, scored by accuracy, and MomentSeeker~\citep{yuan2026momentseeker}, scored by mIoU. Each benchmark is randomly split once into validation and held-out test sets shared by all methods and seeds (Table~\ref{tab:dataset_splits} in Appendix~\ref{sec:experimental_datasets}). Evolution and selection see only the validation set, and held-out results never reach the proposer. The split unit is the question, except on MomentSeeker, where whole source videos are assigned to one side; Appendix~\ref{app:unseen_videos} checks that our conclusions do not depend on videos shared across the split. All benchmarks are used video-only: subtitles and metadata are withheld from harness code.

\paragraph{Models.}
The executor is a frozen Qwen3-VL-8B-Instruct~\citep{bai2025qwen3vltechnicalreport} with greedy decoding; it is the only model that sees video or answers questions. The proposer is Claude Opus~5~\citep{anthropic2026claudeopus5}, which writes harness code but never answers questions or runs the benchmark. Harnesses may call CLIP retrieval~\citep{radford2021learning} and Grounding DINO detection~\citep{liu2024grounding} under fixed per-question limits on frames, executor calls, and tool calls (Table~\ref{tab:inference_budgets}).

\paragraph{Baselines.}
The six hand-written harnesses of Table~\ref{tab:seed_harnesses} (Appendix~\ref{sec:experimental_baselines}) serve as fixed baselines and seed every evolution run. They cover the main mechanism families of video harnesses: uniform sampling, query-conditioned retrieval, self-reflection, tool-gated selection, multi-role debate, and a working memory. The evolved baseline is \basemethod{} (Section~\ref{sec:initial_investigation}), run for 50 iterations of three candidates from the same seeds, with the same executor, proposer, and interface as \method{}.

\paragraph{Protocol.}
Each framework is run with three seeds per benchmark. Each run deploys the harness with the highest validation score, and we report its held-out score as the mean and sample standard deviation over seeds. A harness's cost is its mean number of frames shown to the executor per question, summed over calls. Further details on models, budgets, and statistical tests are in Appendix~\ref{sec:experimental_details}.

\subsection{Main Results}
\label{sec:main_results}
\begin{table*}[t]
    \centering
    \caption{
    \textbf{Main results.} Evolved harnesses report mean $\pm$ sample standard deviation over three seeds; fixed baselines are deterministic.
    }
    \label{tab:main_results}
    \begingroup
    \small
    \setlength{\tabcolsep}{3pt}
    \renewcommand{\arraystretch}{1.0}
    \begin{tabular*}{\textwidth}{@{\extracolsep{\fill}}l*{8}{c}@{}}
        \toprule
        \multirow{2}{*}{\textbf{Method}}
        & \multicolumn{2}{c}{\textbf{Video-MME}}
        & \multicolumn{2}{c}{\textbf{LVB}}
        & \multicolumn{2}{c}{\textbf{OVO-Bench}}
        & \multicolumn{2}{c}{\textbf{MomentSeeker}} \\
        \cmidrule(lr){2-3}
        \cmidrule(lr){4-5}
        \cmidrule(lr){6-7}
        \cmidrule(lr){8-9}
        & Val & Test & Val & Test & Val & Test & Val & Test \\
        \midrule

        \rowcolor{gray!8}
        \multicolumn{9}{l}{\textit{\textbf{Fixed Baselines}}} \\
        uniform
        & $60.44$ & $61.64$
        & $58.97$ & $59.15$
        & $52.76$ & $52.20$
        & $11.94$ & $11.43$ \\

        clip\_topk
        & $66.21$ & $63.40$
        & $62.33$ & $62.07$
        & $51.94$ & $53.32$
        & $16.58$ & $17.98$ \\

        tool\_survey
        & $65.78$ & $62.27$
        & $59.42$ & $60.49$
        & $50.51$ & $51.28$
        & $16.72$ & $17.66$ \\

        reflect\_loop
        & $60.44$ & $60.04$
        & $62.11$ & $60.61$
        & $51.33$ & $55.36$
        & $16.23$ & $17.66$ \\

        debate\_pair
        & $63.78$ & $61.42$
        & $57.40$ & $55.56$
        & $44.58$ & $45.86$
        & $16.42$ & $17.92$ \\

        memo\_multiturn
        & $63.11$ & $58.98$
        & $58.52$ & $59.93$
        & $55.42$ & $54.65$
        & $15.91$ & $16.00$ \\

        \midrule
        \rowcolor{gray!8}
        \multicolumn{9}{l}{\textit{\textbf{Evolved Harnesses}}} \\

        \basemethod~
        & $72.2_{\pm 1.8}$ & $66.5_{\pm 1.2}$
        & $70.2_{\pm 0.3}$ & $66.6_{\pm 1.0}$
        & $72.9_{\pm 2.1}$ & $67.2_{\pm 1.7}$
        & $23.8_{\pm 0.7}$ & $19.0_{\pm 1.8}$ \\

        \textbf{\method}~
        & $75.0_{\pm 1.9}$ & $70.1_{\pm 1.0}$
        & $71.1_{\pm 2.0}$ & $67.3_{\pm 1.0}$
        & $73.8_{\pm 2.4}$ & $72.5_{\pm 3.1}$
        & $25.3_{\pm 0.3}$ & $23.0_{\pm 0.7}$ \\

        \bottomrule
    \end{tabular*}
    \endgroup
\end{table*}

\paragraph{Evolution outperforms hand-written harnesses.}
Both evolved frameworks beat every fixed baseline on held-out data on all four benchmarks (Table~\ref{tab:main_results}). Averaged over benchmarks, they exceed the strongest baseline on each benchmark by $5.12$ (\basemethod{}) and $8.52$ points (\method{}). No single fixed design is best everywhere: \texttt{clip\_topk} leads on Video-MME, LongVideoBench, and MomentSeeker, and \texttt{reflect\_loop} on OVO-Bench. Nor does adding mechanisms by hand reliably help: \texttt{debate\_pair} spends five times the frames of \texttt{clip\_topk} on Video-MME (Table~\ref{tab:seed_harnesses}) but is less accurate on held-out data on every benchmark. Which mechanisms help, and how to combine them, is benchmark-dependent, which is exactly what evolution searches for.

\paragraph{\method{} finds stronger and better-generalizing harnesses.}
\method{} improves on \basemethod{} on all four benchmarks, raising the mean held-out score from $54.83$ to $58.23$, a $6.2\%$ relative gain. The gains are largest on OVO-Bench ($+5.3$), MomentSeeker ($+4.0$), and Video-MME ($+3.6$), and smallest on LongVideoBench ($+0.7$). At the same time, \method{} lowers the mean validation--test gap from $4.95$ to $3.08$ points ($-37.9\%$), with the largest reduction on OVO-Bench ($5.7\to1.3$) and essentially no change on LongVideoBench ($3.6\to3.8$).

\paragraph{The gains come mostly from better transfer.}
On every benchmark except LongVideoBench, \method{} gains more on held-out data than on validation. The clearest case is OVO-Bench, where held-out accuracy rises from $67.2$ to $72.5$ while validation accuracy rises only from $72.9$ to $73.8$; on MomentSeeker, the corresponding gains are $+4.0$ and $+1.5$ points. \method{} thus selects harnesses that transfer better rather than fitting validation more tightly.

\paragraph{Overfitting is reduced, not eliminated.}
Both evolved frameworks still score higher on validation than on held-out data on every benchmark. Information isolation and structured search mitigate overfitting (Section~\ref{sec:method}), but any procedure that repeatedly selects on the same validation set can still favor harnesses specialized to it.

\subsection{Inference Cost and Budget-Matched Comparison}
\label{sec:exp:budget}

Evolution may raise accuracy simply by showing the executor more frames. We therefore record three per-question costs of each deployed harness on its held-out traces: frames passed to the executor, executor (VLM) calls, and tool calls, averaged over the three seeds.

\begin{table}[t]
\centering
\small
\caption{\textbf{Cost of the deployed harnesses.} Held-out score and per-question cost, averaged over three seeds.}
\label{tab:cost_deployed}
\begin{tabular}{@{}llcccc@{}}
\toprule
Benchmark & Method & Test & Frames/q & VLM calls/q & Tool calls/q \\
\midrule
\multirow{2}{*}{Video-MME} & \method{} & 70.1 & 106.1 & 2.44 & 19.3 \\
 & \basemethod{} & 66.5 & 55.1 & 2.43 & 1.6 \\
\midrule
\multirow{2}{*}{LVB} & \method{} & 67.3 & 64.8 & 2.85 & 17.5 \\
 & \basemethod{} & 66.6 & 69.2 & 2.77 & 5.0 \\
\midrule
\multirow{2}{*}{OVO-Bench} & \method{} & 72.5 & 47.8 & 4.91 & 2.6 \\
 & \basemethod{} & 67.2 & 27.3 & 3.57 & 1.5 \\
\midrule
\multirow{2}{*}{MomentSeeker} & \method{} & 23.0 & 69.8 & 4.38 & 11.1 \\
 & \basemethod{} & 19.0 & 60.3 & 2.52 & 14.7 \\
\bottomrule
\end{tabular}
\end{table}

\paragraph{Cost of the deployed harnesses.}
Table~\ref{tab:cost_deployed} shows that the deployed \method{} harnesses spend more frames than \basemethod{} on Video-MME, OVO-Bench, and MomentSeeker, e.g., $106.1$ vs.\ $55.1$ on Video-MME, and use many more tool calls on Video-MME and LongVideoBench, where they retrieve and re-inspect evidence before answering. Part of the gain in Table~\ref{tab:main_results} therefore coincides with a larger budget. The relation is not uniform, however: on LongVideoBench, \method{} is more accurate with fewer frames ($64.8$ vs.\ $69.2$), and on MomentSeeker with fewer tool calls.

\paragraph{Budget-matched selection.}
To separate a better harness from a larger budget, we cap each \method{} run at the budget of \basemethod{}, defined per benchmark as the mean frames per question of its three deployed harnesses. Within each \method{} run, we keep the candidates whose validation traces stay within this budget and deploy the one with the highest validation accuracy, breaking ties by fewer frames. Since selection uses no held-out data, \method{} could have deployed this harness under the budget.

\begin{table}[t]
\centering
\small
\setlength{\tabcolsep}{3.5pt}
\caption{\textbf{Budget-matched comparison.} \method{} restricted to the frame budget of the deployed \basemethod{} harness and selected by validation accuracy, averaged over three seeds. $\Delta$: held-out gain of \method{}.}
\label{tab:cost_matched}
\begin{tabular}{@{}llccccccc@{}}
\toprule
& & \multicolumn{3}{c}{\method{} (budget-matched)} & \multicolumn{3}{c}{\basemethod{} (deployed)} & \\
\cmidrule(lr){3-5}\cmidrule(lr){6-8}
Benchmark & Budget & Test & Frames/q & Calls/q & Test & Frames/q & Calls/q & $\Delta$ \\
\midrule
Video-MME & $\le$55.1 & 68.93 & 52.5 & 2.04 & 66.50 & 55.1 & 2.43 & $+$2.43 \\
LVB & $\le$69.2 & 67.19 & 57.8 & 2.86 & 66.60 & 69.2 & 2.77 & $+$0.59 \\
OVO-Bench & $\le$27.3 & 68.42 & 25.8 & 3.64 & 67.20 & 27.3 & 3.57 & $+$1.22 \\
MomentSeeker & $\le$60.3 & 21.41 & 56.0 & 2.48 & 19.00 & 60.3 & 2.52 & $+$2.41 \\
\bottomrule
\end{tabular}
\end{table}

\paragraph{Results.}
As shown in Table~\ref{tab:cost_matched}, the budget-matched \method{} harnesses remain ahead of \basemethod{} on every benchmark while using fewer frames and at most $0.09$ more executor calls per question. Averaged over benchmarks, the matched gain ($1.66$ points) is about half of the deployed one ($3.40$ points). On the cost--accuracy plane (Figure~\ref{fig:cost_accuracy}), the budget-matched \method{} harness also lies above every fixed baseline. Thus \method{} spends a larger budget when it pays off on validation, which explains part of its deployed gain; the remainder is not a budget effect, since at the same frame budget the evidence of \method{} still selects more accurate harnesses.

\begin{figure}[t]
\centering
\includegraphics[width=\linewidth]{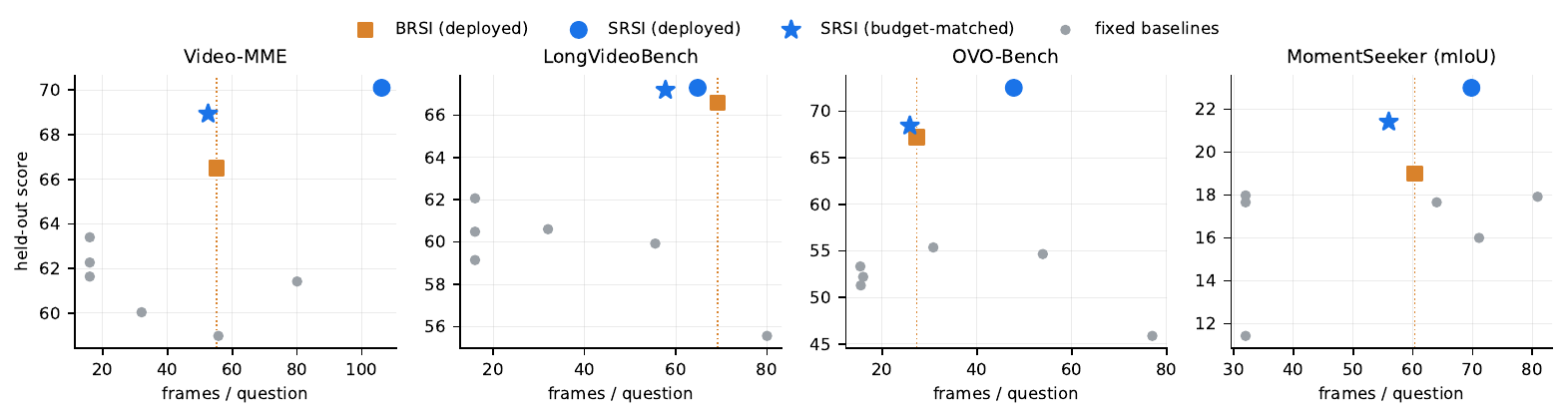}
\caption{\textbf{Held-out score versus frames per question on each benchmark.} Deployed harnesses of \basemethod{} (squares) and \method{} (circles), budget-matched \method{} harnesses (stars), and fixed baselines (gray dots), each averaged over three seeds. The dotted line marks the budget, i.e., the frames per question of the deployed \basemethod{} harness. On every benchmark, the budget-matched \method{} harness lies above \basemethod{} without using more frames.}
\label{fig:cost_accuracy}
\end{figure}

\subsection{Ablations}
\label{sec:exp:ablations}

We ablate the three components of \method{} (instrumented warm-up, branch-axis choice, information isolation) and vary the proposer model. Unless stated otherwise, ablations run on Video-MME with three seeds and share every other setting with the main results.

\begin{figure}[t]
    \centering
    \begin{minipage}[t]{0.458\textwidth}
        \centering
        \includegraphics[width=\linewidth]{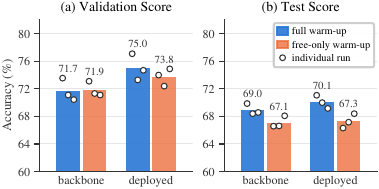}
        \caption{\textbf{Warm-up ablation on Video-MME.} Validation (a) and held-out (b) accuracy of the warm-up backbone and the deployed system, for the full instrumented warm-up and a free-only warm-up with the same budget; bars are means and dots are individual runs.}
        \label{fig:warmablation}
    \end{minipage}\hfill
    \begin{minipage}[t]{0.512\textwidth}
        \centering
        \includegraphics[width=\linewidth]{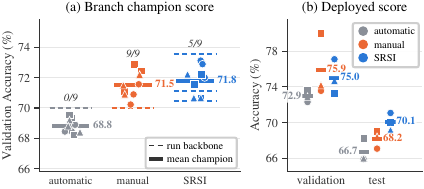}
        \caption{\textbf{Axis-selection ablation on Video-MME.} (a)~Validation accuracy of branch champions (markers) for automatic, manual, and \method{} axes, with each run's backbone (dashed) and the mean champion (solid); fractions count branches that improve on their backbone. (b)~Validation and held-out accuracy of the deployed systems.}
        \label{fig:axesablation}
    \end{minipage}
\end{figure}

\paragraph{Instrumented warm-up.}
\label{sec:exp:warmablation}
We compare the full warm-up of \method{} with a \emph{free-only} variant that spends the same five-candidate budget per round on free proposals, with no removal or intensity slots. As shown in Figure~\ref{fig:warmablation}, the two variants reach nearly the same warm-up validation accuracy ($71.70$ vs.\ $71.85$), yet the full warm-up transfers better to the held-out split ($69.0$ vs.\ $67.1$), and the advantage persists to the deployed system ($70.1$ vs.\ $67.3$). Every full-warm-up run beats every free-only run on held-out data, at both stages. Controlling for frames per question still leaves a held-out advantage of about $2$ points. The difference lies in the evidence: the full warm-up yields $11.7$ rather than $4.3$ removal contrasts per run, and leaves $5$ rather than $11$ of the $21$ per-family, per-run coverage increments undefined. A within-run replay attributes only $28\%$ of the backbone gap to selecting structured candidates; the rest is consistent with slot results informing later free proposals. The removal and intensity slots thus supply more informative evidence for later search rather than a stronger warm-up champion. Per-run results and the attribution analysis are in Appendix~\ref{app:warmablation}.

\paragraph{Who should choose the axes?}
\label{sec:exp:axesablation}
We compare three ways to define three branch axes: \emph{automatic}, where one proposer call defines axes and charters from the warm-up report; \emph{manual}, where a researcher defines them from the same report; and the fixed templates of \method{} with probe--race allocation (Section~\ref{sec:method_branches}). All variants share warm-up outputs, backbones, and downstream stages. As shown in Figure~\ref{fig:axesablation}, the fixed templates of \method{} generalize best. None of the nine automatic branches improves on its starting backbone, whereas all nine manual ones do. \method{} reaches slightly higher mean branch accuracy, with five of nine branches improving on their backbones. The manual variant scores highest on validation ($75.9$ vs.\ $75.0$) but lower on held-out data ($68.2$ vs.\ $70.1$). Automatic decomposition thus misses productive directions, while researcher-written charters, informed by the same validation evidence, overfit to it. Fixing the axes in advance avoids both failures.

\paragraph{Information isolation.}
\label{sec:exp:isolation}
To isolate content-free feedback from the rest of \method{}, we reuse the LongVideoBench runs of the feedback study (Section~\ref{sec:meta_harness_failure_modes}) on its 33\% validation ladder (Appendix~\ref{app:proposer_information}), whose 446 validation and 891 held-out questions match the split of Table~\ref{tab:main_results}. In the \emph{Multimodal} arm, the proposer has the full access of \basemethod{}; the \emph{Blindfolded} arm keeps the same loop but shows only per-question correctness and behavioral counters, the content-free feedback of \method{} (Appendix~\ref{app:method_isolation}). The Blindfolded arm is thus \basemethod{} with information isolation and nothing else. Each arm runs three seeds of 30 iterations with three candidates, with perception tools disabled, and deploys each run's validation argmax. The \method{} row is the LongVideoBench entry of Table~\ref{tab:main_results}.

\begin{table}[t]
\centering
\caption{\textbf{Information-isolation ablation on LongVideoBench.} Accuracy (\%) of each run's validation argmax, mean $\pm$ std over three seeds.}
\label{tab:isolation_ablation}
\small
\setlength{\tabcolsep}{5pt}
\begin{tabular}{@{}lcc@{}}
\toprule
Setting & Val & Held-out \\
\midrule
\basemethod{} feedback (Multimodal) & $71.00_{\pm 3.21}$ & $65.43_{\pm 1.18}$ \\
\basemethod{} + isolation (Blindfolded) & $71.22_{\pm 0.34}$ & $66.93_{\pm 0.81}$ \\
\method{} & $71.1_{\pm 2.0}$ & $\mathbf{67.3}_{\pm 1.0}$ \\
\bottomrule
\end{tabular}
\end{table}

As Table~\ref{tab:isolation_ablation} shows, the three settings reach nearly the same validation accuracy ($71.0$--$71.2\%$), so the differences lie entirely in transfer. Withholding question content raises held-out accuracy by $1.50$ points over \basemethod{} feedback, about $80\%$ of the way to \method{}, and shrinks the validation--test gap from $5.57$ to $4.29$ points. This is consistent with Finding~2: the proposer can no longer tailor harnesses to the validation questions, so less of its validation progress is instance-specific fit. The remaining stages of \method{} add another $0.37$ held-out points and shrink the gap to $3.8$ points. This second step comes with three caveats: \method{} evaluates more candidates per seed (at most 156, Appendix~\ref{app:method_config}, versus 90), the \basemethod{} row here is the 30-iteration Multimodal arm rather than the 50-iteration run of Table~\ref{tab:main_results}, and the two feedback arms run with perception tools disabled, whereas \method{} may use them. With three seeds per setting, we read this second step as directional, whereas the first changes only what the proposer sees.

\begin{wraptable}{r}{0.36\textwidth}
    \vspace{-12pt}
    \centering
    \caption{\textbf{Proposer ablation on Video-MME.} Mean $\pm$ std over three seeds.}
    \label{tab:proposer_ablation}
    \small
    \setlength{\tabcolsep}{5pt}
    \begin{tabular}{@{}lcc@{}}
        \toprule
        \textbf{Proposer} & Val & Test \\
        \midrule
        Opus~5   & $75.0_{\pm 1.9}$ & $70.1_{\pm 1.0}$ \\
        Sonnet~5 & $73.3_{\pm 0.6}$ & $67.7_{\pm 0.4}$ \\
        Opus~4.8 & $71.4_{\pm 0.8}$ & $65.8_{\pm 0.1}$ \\
        \bottomrule
    \end{tabular}
    \vspace{-10pt}
\end{wraptable}
\paragraph{Proposer model.}
\label{sec:exp:proposerablation}
We run \method{} with different proposer models, keeping the executor and every other setting fixed. As shown in Table~\ref{tab:proposer_ablation}, \method{} remains effective with weaker proposers: all three exceed every fixed baseline on held-out data (best: $63.4$), and with Claude Sonnet~5 it still surpasses \basemethod{} driven by Claude Opus~5 ($67.7$ vs.\ $66.5$). Held-out accuracy nonetheless tracks proposer capability, from $65.8$ with Claude Opus~4.8 to $70.1$ with Claude Opus~5, suggesting that \method{} may benefit further from stronger proposers.

\section{Conclusion}
\label{sec:conclusion}

We studied how recursive harness evolution can reliably discover generalizable designs, using agentic video understanding as a test bed. Analyzing over 4,000 harnesses evolved by \basemethod{}, we found that evolution discovers competitive designs, but its validation gains do not reliably generalize: the same imperfect validation evidence adapts both the harness and the proposer's search strategy, which leads to validation overfitting, search-budget dilution, and the early-experience trap. \method{} governs this evidence through content-free proposer feedback, an instrumented warm-up, and isolated mechanism branches followed by factorial composition. Across four benchmarks, it raises the mean held-out score by $6.2\%$ and cuts the mean validation--test gap by $37.9\%$ relative to \basemethod{}, and about half of the gain remains at a matched frame budget. Reliable harness evolution thus depends not only on how well search optimizes validation performance, but on how it is organized and informed.

\clearpage
\newpage
\bibliographystyle{assets/plainnat}
\bibliography{paper}

\begin{thebibliography}{50}
\providecommand{\natexlab}[1]{#1}
\providecommand{\url}[1]{\texttt{#1}}
\expandafter\ifx\csname urlstyle\endcsname\relax
  \providecommand{\doi}[1]{doi: #1}\else
  \providecommand{\doi}{doi: \begingroup \urlstyle{rm}\Url}\fi

\bibitem[Agrawal et~al.(2026)Agrawal, Tan, Soylu, Ziems, Khare, Opsahl-Ong,
  Singhvi, Shandilya, Ryan, Jiang, et~al.]{agrawal2026gepa}
Lakshya~A Agrawal, Shangyin Tan, Dilara Soylu, Noah Ziems, Rishi Khare, Krista
  Opsahl-Ong, Arnav Singhvi, Herumb Shandilya, Michael~J Ryan, Meng Jiang,
  et~al.
\newblock Gepa: Reflective prompt evolution can outperform reinforcement
  learning.
\newblock In \emph{International Conference on Learning Representations},
  volume 2026, pages 8479--8565, 2026.

\bibitem[{Anthropic}(2026)]{anthropic2026claudeopus5}
{Anthropic}.
\newblock Introducing claude opus 5.
\newblock \url{https://www.anthropic.com/news/claude-opus-5}, July 2026.
\newblock Accessed: 2026-09-12.

\bibitem[Arthur(1989)]{arthur1989competing}
W.~Brian Arthur.
\newblock Competing technologies, increasing returns, and lock-in by historical
  events.
\newblock \emph{The Economic Journal}, 99\penalty0 (394):\penalty0 116--131,
  1989.

\bibitem[Audibert et~al.(2010)Audibert, Bubeck, and Munos]{audibert2010best}
Jean-Yves Audibert, S{\'e}bastien Bubeck, and R{\'e}mi Munos.
\newblock Best arm identification in multi-armed bandits.
\newblock In \emph{Proceedings of the 23rd Conference on Learning Theory
  (COLT)}, pages 41--53, 2010.

\bibitem[Bai et~al.(2025)Bai, Cai, Chen, Chen, Chen, Cheng, Deng, Ding, Gao,
  Ge, Ge, Guo, Huang, Huang, Huang, Hui, Jiang, Li, Li, Li, Li, Lin, Lin, Liu,
  Liu, Liu, Liu, Liu, Liu, Lu, Luo, Lv, Men, Meng, Ren, Ren, Song, Sun, Tang,
  Tu, Wan, Wang, Wang, Wang, Wang, Xie, Xu, Xu, Xu, Yang, Yang, Yang, Yang, Yu,
  Zhang, Zhang, Zhang, Zheng, Zhong, Zhou, Zhou, Zhou, Zhu, and
  Zhu]{bai2025qwen3vltechnicalreport}
Shuai Bai, Yuxuan Cai, Ruizhe Chen, Keqin Chen, Xionghui Chen, Zesen Cheng,
  Lianghao Deng, Wei Ding, Chang Gao, Chunjiang Ge, Wenbin Ge, Zhifang Guo,
  Qidong Huang, Jie Huang, Fei Huang, Binyuan Hui, Shutong Jiang, Zhaohai Li,
  Mingsheng Li, Mei Li, Kaixin Li, Zicheng Lin, Junyang Lin, Xuejing Liu,
  Jiawei Liu, Chenglong Liu, Yang Liu, Dayiheng Liu, Shixuan Liu, Dunjie Lu,
  Ruilin Luo, Chenxu Lv, Rui Men, Lingchen Meng, Xuancheng Ren, Xingzhang Ren,
  Sibo Song, Yuchong Sun, Jun Tang, Jianhong Tu, Jianqiang Wan, Peng Wang,
  Pengfei Wang, Qiuyue Wang, Yuxuan Wang, Tianbao Xie, Yiheng Xu, Haiyang Xu,
  Jin Xu, Zhibo Yang, Mingkun Yang, Jianxin Yang, An~Yang, Bowen Yu, Fei Zhang,
  Hang Zhang, Xi~Zhang, Bo~Zheng, Humen Zhong, Jingren Zhou, Fan Zhou, Jing
  Zhou, Yuanzhi Zhu, and Ke~Zhu.
\newblock Qwen3-vl technical report, 2025.
\newblock \url{https://arxiv.org/abs/2511.21631}.

\bibitem[Brezzi and Lai(2000)]{brezzi2000incomplete}
Monica Brezzi and Tze~Leung Lai.
\newblock Incomplete learning from endogenous data in dynamic allocation.
\newblock \emph{Econometrica}, 68\penalty0 (6):\penalty0 1511--1516, 2000.

\bibitem[Cui et~al.(2026)Cui, Wang, Li, Chen, Huang, Chen, Zhai, Tang, Jia,
  Wu, Sun, and Hong]{cui2026metavideoagent}
Benlei Cui, Ruize Wang, Junjie Li, Jinhao Chen, Longtao Huang, Yinghao Chen,
  Yuwen Zhai, Jingqun Tang, Ruijian Jia, Weiwei Wu, Pengfei Sun, and Haiwen
  Hong.
\newblock Metavideoagent: Automated video-agent evolution for long-form video
  understanding.
\newblock \emph{arXiv preprint arXiv:2608.04587}, 2026.

\bibitem[Denrell and March(2001)]{denrell2001adaptation}
Jerker Denrell and James~G. March.
\newblock Adaptation as information restriction: The hot stove effect.
\newblock \emph{Organization Science}, 12\penalty0 (5):\penalty0 523--538,
  2001.

\bibitem[Fernando et~al.(2023)Fernando, Banarse, Michalewski, Osindero, and
  Rockt{\"a}schel]{fernando2023promptbreeder}
Chrisantha Fernando, Dylan Banarse, Henryk Michalewski, Simon Osindero, and Tim
  Rockt{\"a}schel.
\newblock Promptbreeder: Self-referential self-improvement via prompt
  evolution.
\newblock \emph{arXiv preprint arXiv:2309.16797}, 2023.

\bibitem[Fisher(1930)]{fisher1930genetical}
Ronald~A. Fisher.
\newblock \emph{The Genetical Theory of Natural Selection}.
\newblock Clarendon Press, Oxford, 1930.

\bibitem[Fu et~al.(2025)Fu, Dai, Luo, Li, Ren, Zhang, Wang, Zhou, Shen, Zhang,
  Chen, Li, Lin, Zhao, Li, Xu, Zheng, Chen, Shan, He, and
  Sun]{fu2025videommefirstevercomprehensiveevaluation}
Chaoyou Fu, Yuhan Dai, Yongdong Luo, Lei Li, Shuhuai Ren, Renrui Zhang, Zihan
  Wang, Chenyu Zhou, Yunhang Shen, Mengdan Zhang, Peixian Chen, Yanwei Li,
  Shaohui Lin, Sirui Zhao, Ke~Li, Tong Xu, Xiawu Zheng, Enhong Chen, Caifeng
  Shan, Ran He, and Xing Sun.
\newblock Video-mme: The first-ever comprehensive evaluation benchmark of
  multi-modal llms in video analysis, 2025.
\newblock \url{https://arxiv.org/abs/2405.21075}.

\bibitem[Hu et~al.(2025)Hu, Lu, and Clune]{hu2025automated}
Shengran Hu, Cong Lu, and Jeff Clune.
\newblock Automated design of agentic systems.
\newblock In \emph{International Conference on Learning Representations},
  volume 2025, pages 21344--21377, 2025.

\bibitem[Khattab et~al.(2023)Khattab, Singhvi, Maheshwari, Zhang, Santhanam,
  Vardhamanan, Haq, Sharma, Joshi, Moazam, et~al.]{khattab2023dspy}
Omar Khattab, Arnav Singhvi, Paridhi Maheshwari, Zhiyuan Zhang, Keshav
  Santhanam, Sri Vardhamanan, Saiful Haq, Ashutosh Sharma, Thomas~T Joshi,
  Hanna Moazam, et~al.
\newblock Dspy: Compiling declarative language model calls into self-improving
  pipelines.
\newblock \emph{arXiv preprint arXiv:2310.03714}, 2023.

\bibitem[Lee et~al.(2026)Lee, Nair, Zhang, Lee, Khattab, and Finn]{lee2026meta}
Yoonho Lee, Roshen Nair, Qizheng Zhang, Kangwook Lee, Omar Khattab, and Chelsea
  Finn.
\newblock Meta-harness: End-to-end optimization of model harnesses.
\newblock \emph{arXiv preprint arXiv:2603.28052}, 2026.

\bibitem[Liang et~al.(2025)Liang, Zhang, Zhou, Huang, Li, Zhang, Wu, Wang, Luo,
  Liao, and Fei]{liang2025univauniversalvideoagent}
Zhengyang Liang, Daoan Zhang, Huichi Zhou, Rui Huang, Bobo Li, Yuechen Zhang,
  Shengqiong Wu, Xiaohan Wang, Jiebo Luo, Lizi Liao, and Hao Fei.
\newblock Univa: Universal video agent towards open-source next-generation
  video generalist, 2025.
\newblock \url{https://arxiv.org/abs/2511.08521}.

\bibitem[Lin et~al.(2026)Lin, Wu, Wang, Shi, Sang, He, Liu, Wei, Wu, Zhang,
  et~al.]{lin2026harness}
Minhua Lin, Juncheng Wu, Zijun Wang, Zhan Shi, Yisi Sang, Bing He, Zewen Liu,
  Tianxin Wei, Zongyu Wu, Zhiwei Zhang, et~al.
\newblock Harness updating is not harness benefit: Disentangling evolution
  capabilities in self-evolving llm agents.
\newblock \emph{arXiv preprint arXiv:2605.30621}, 2026.

\bibitem[Liu et~al.(2024)Liu, Zeng, Ren, Li, Zhang, Yang, Jiang, Li, Yang, Su,
  Zhu, and Zhang]{liu2024grounding}
Shilong Liu, Zhaoyang Zeng, Tianhe Ren, Feng Li, Hao Zhang, Jie Yang, Qing
  Jiang, Chunyuan Li, Jianwei Yang, Hang Su, Jun Zhu, and Lei Zhang.
\newblock Grounding {DINO}: Marrying {DINO} with grounded pre-training for
  open-set object detection.
\newblock In \emph{European Conference on Computer Vision (ECCV)}, 2024.

\bibitem[Liu et~al.(2026)Liu, Shi, Sang, He, Lin, Wei, Wang, Dumoulin, Jin, and
  Lu]{liu2026adaptive}
Zewen Liu, Zhan Shi, Yisi Sang, Bing He, Minhua Lin, Tianxin Wei, Dakuo Wang,
  Benoit Dumoulin, Wei Jin, and Hanqing Lu.
\newblock Adaptive auto-harness: Sustained self-improvement for agentic system
  deployment on open-ended task streams.
\newblock \emph{arXiv preprint arXiv:2606.01770}, 2026.

\bibitem[Long et~al.(2025{\natexlab{a}})Long, Wan, Nakhost, Lee, Pfister, and
  Arık]{long2025vistatesttimeselfimprovingvideo}
Do~Xuan Long, Xingchen Wan, Hootan Nakhost, Chen-Yu Lee, Tomas Pfister, and
  Sercan~Ö. Arık.
\newblock Vista: A test-time self-improving video generation agent,
  2025{\natexlab{a}}.
\newblock \url{https://arxiv.org/abs/2510.15831}.

\bibitem[Long et~al.(2025{\natexlab{b}})Long, He, Ye, Pan, Lin, Li, Zhao, and
  Li]{long2025seeinglisteningrememberingreasoning}
Lin Long, Yichen He, Wentao Ye, Yiyuan Pan, Yuan Lin, Hang Li, Junbo Zhao, and
  Wei Li.
\newblock Seeing, listening, remembering, and reasoning: A multimodal agent
  with long-term memory, 2025{\natexlab{b}}.
\newblock \url{https://arxiv.org/abs/2508.09736}.

\bibitem[Lord and Novick(1968)]{lord1968statistical}
Frederic~M. Lord and Melvin~R. Novick.
\newblock \emph{Statistical Theories of Mental Test Scores}.
\newblock Addison-Wesley, Reading, MA, 1968.

\bibitem[Luo et~al.(2025)Luo, Zheng, Li, Yin, Lin, Fu, Huang, Ji, Chao, Luo,
  and Ji]{luo2025videoragvisuallyalignedretrievalaugmentedlong}
Yongdong Luo, Xiawu Zheng, Guilin Li, Shukang Yin, Haojia Lin, Chaoyou Fu,
  Jinfa Huang, Jiayi Ji, Fei Chao, Jiebo Luo, and Rongrong Ji.
\newblock Video-rag: Visually-aligned retrieval-augmented long video
  comprehension, 2025.
\newblock \url{https://arxiv.org/abs/2411.13093}.

\bibitem[Luo et~al.(2026)Luo, Guo, and Li]{luo2026videorsi}
Bingjun Luo, Jialin Guo, and Siqi Li.
\newblock Video-rsi: Recursive self-improvement of video understanding agents
  via harness evolution.
\newblock \emph{arXiv preprint arXiv:2609.37950}, 2026.

\bibitem[McNemar(1947)]{mcnemar1947note}
Quinn McNemar.
\newblock Note on the sampling error of the difference between correlated
  proportions or percentages.
\newblock \emph{Psychometrika}, 12\penalty0 (2):\penalty0 153--157, 1947.

\bibitem[Ning et~al.(2026)Ning, Tieu, Fu, Wei, Li, Bei, Zou, Ai, Liu, Li, Chen,
  Zhao, Yang, Li, Qian, Li, Lin, Zeng, Qiu, Chen, Sun, Yang, Wang, Pan, Yang,
  Zhang, Fang, Cui, Cao, Chen, Sun, Chen, Srinivasan, Mathur, Xia, Li, Yan, Lu,
  Zhang, Zhang, Tong, and He]{ning2026codeagentharness}
Xuying Ning, Katherine Tieu, Dongqi Fu, Tianxin Wei, Zihao Li, Yuanchen Bei,
  Jiaru Zou, Mengting Ai, Zhining Liu, Ting-Wei Li, Lingjie Chen, Yanjun Zhao,
  Ke~Yang, Bingxuan Li, Cheng Qian, Gaotang Li, Xiao Lin, Zhichen Zeng,
  Ruizhong Qiu, Sirui Chen, Yifan Sun, Xiyuan Yang, Ruida Wang, Rui Pan,
  Chenyuan Yang, Dylan Zhang, Liri Fang, Zikun Cui, Yang Cao, Pan Chen, Dorothy
  Sun, Ren Chen, Mahesh Srinivasan, Nipun Mathur, Yinglong Xia, Hong Li, Hong
  Yan, Pan Lu, Lingming Zhang, Tong Zhang, Hanghang Tong, and Jingrui He.
\newblock Code as agent harness, 2026.
\newblock \url{https://arxiv.org/abs/2605.18747}.

\bibitem[Niu et~al.(2025)Niu, Li, Miao, Ge, Zhou, He, Dong, Duan, Ding, Qian,
  et~al.]{niu2025ovo}
Junbo Niu, Yifei Li, Ziyang Miao, Chunjiang Ge, Yuanhang Zhou, Qihao He, Xiaoyi
  Dong, Haodong Duan, Shuangrui Ding, Rui Qian, et~al.
\newblock Ovo-bench: How far is your video-llms from real-world online video
  understanding?
\newblock In \emph{2025 IEEE/CVF Conference on Computer Vision and Pattern
  Recognition (CVPR)}, pages 18902--18913. IEEE, 2025.

\bibitem[Novikov et~al.(2025)Novikov, V{\~u}, Eisenberger, Dupont, Huang,
  Wagner, Shirobokov, Kozlovskii, Ruiz, Mehrabian,
  et~al.]{novikov2025alphaevolve}
Alexander Novikov, Ng{\^a}n V{\~u}, Marvin Eisenberger, Emilien Dupont, Po-Sen
  Huang, Adam~Zsolt Wagner, Sergey Shirobokov, Borislav Kozlovskii,
  Francisco~JR Ruiz, Abbas Mehrabian, et~al.
\newblock Alphaevolve: A coding agent for scientific and algorithmic discovery.
\newblock \emph{arXiv preprint arXiv:2506.13131}, 2025.

\bibitem[Opsahl-Ong et~al.(2024)Opsahl-Ong, Ryan, Purtell, Broman, Potts,
  Zaharia, and Khattab]{opsahl2024optimizing}
Krista Opsahl-Ong, Michael~J Ryan, Josh Purtell, David Broman, Christopher
  Potts, Matei Zaharia, and Omar Khattab.
\newblock Optimizing instructions and demonstrations for multi-stage language
  model programs.
\newblock In \emph{Proceedings of the 2024 Conference on Empirical Methods in
  Natural Language Processing}, pages 9340--9366, 2024.

\bibitem[Orr(2000)]{orr2000adaptation}
H.~Allen Orr.
\newblock Adaptation and the cost of complexity.
\newblock \emph{Evolution}, 54\penalty0 (1):\penalty0 13--20, 2000.

\bibitem[Pryzant et~al.(2023)Pryzant, Iter, Li, Lee, Zhu, and
  Zeng]{pryzant2023automatic}
Reid Pryzant, Dan Iter, Jerry Li, Yin Lee, Chenguang Zhu, and Michael Zeng.
\newblock Automatic prompt optimization with “gradient descent” and beam
  search.
\newblock In \emph{Proceedings of the 2023 conference on empirical methods in
  natural language processing}, pages 7957--7968, 2023.

\bibitem[Radford et~al.(2021)Radford, Kim, Hallacy, Ramesh, Goh, Agarwal,
  Sastry, Askell, Mishkin, Clark, Krueger, and Sutskever]{radford2021learning}
Alec Radford, Jong~Wook Kim, Chris Hallacy, Aditya Ramesh, Gabriel Goh,
  Sandhini Agarwal, Girish Sastry, Amanda Askell, Pamela Mishkin, Jack Clark,
  Gretchen Krueger, and Ilya Sutskever.
\newblock Learning transferable visual models from natural language
  supervision.
\newblock In \emph{Proceedings of the 38th International Conference on Machine
  Learning (ICML)}, pages 8748--8763, 2021.

\bibitem[Rothschild(1974)]{rothschild1974two}
Michael Rothschild.
\newblock A two-armed bandit theory of market pricing.
\newblock \emph{Journal of Economic Theory}, 9\penalty0 (2):\penalty0 185--202,
  1974.

\bibitem[Russo and Zou(2016)]{russo2016controlling}
Daniel Russo and James Zou.
\newblock Controlling bias in adaptive data analysis using information theory.
\newblock In \emph{Proceedings of the 19th International Conference on
  Artificial Intelligence and Statistics (AISTATS)}, pages 1232--1240, 2016.

\bibitem[Shang et~al.(2025)Shang, Li, Zhao, Ma, Liu, Xu, and
  Li]{shang2025agentsquare}
Yu~Shang, Yu~Li, Keyu Zhao, Likai Ma, Jiahe Liu, Fengli Xu, and Yong Li.
\newblock Agentsquare: Automatic llm agent search in modular design space.
\newblock In \emph{International Conference on Learning Representations},
  volume 2025, pages 3841--3865, 2025.

\bibitem[Shen et~al.(2024)Shen, Xiong, Zhao, Wu, Chen, Zhu, Liu, Xiao,
  Varadarajan, Bordes, et~al.]{shen2024longvu}
Xiaoqian Shen, Yunyang Xiong, Changsheng Zhao, Lemeng Wu, Jun Chen, Chenchen
  Zhu, Zechun Liu, Fanyi Xiao, Balakrishnan Varadarajan, Florian Bordes, et~al.
\newblock Longvu: Spatiotemporal adaptive compression for long video-language
  understanding.
\newblock \emph{arXiv preprint arXiv:2410.17434}, 2024.

\bibitem[Smith and Winkler(2006)]{smith2006optimizer}
James~E. Smith and Robert~L. Winkler.
\newblock The optimizer's curse: Skepticism and postdecision surprise in
  decision analysis.
\newblock \emph{Management Science}, 52\penalty0 (3):\penalty0 311--322, 2006.

\bibitem[Tang et~al.(2025)Tang, Qiu, Xie, Tian, Jiao, and Ye]{tang2025adaptive}
Xi~Tang, Jihao Qiu, Lingxi Xie, Yunjie Tian, Jianbin Jiao, and Qixiang Ye.
\newblock Adaptive keyframe sampling for long video understanding.
\newblock In \emph{2025 IEEE/CVF Conference on Computer Vision and Pattern
  Recognition (CVPR)}, pages 29118--29128. IEEE, 2025.

\bibitem[Wang et~al.(2024)Wang, Zhang, Zohar, and
  Yeung-Levy]{wang2024videoagent}
Xiaohan Wang, Yuhui Zhang, Orr Zohar, and Serena Yeung-Levy.
\newblock Videoagent: Long-form video understanding with large language model
  as agent.
\newblock In \emph{European Conference on Computer Vision}, pages 58--76.
  Springer, 2024.

\bibitem[Wang et~al.(2025)Wang, Yu, Stengel-Eskin, Yoon, Cheng, Bertasius, and
  Bansal]{wang2025videotreeadaptivetreebasedvideo}
Ziyang Wang, Shoubin Yu, Elias Stengel-Eskin, Jaehong Yoon, Feng Cheng, Gedas
  Bertasius, and Mohit Bansal.
\newblock Videotree: Adaptive tree-based video representation for llm reasoning
  on long videos, 2025.
\newblock \url{https://arxiv.org/abs/2405.19209}.

\bibitem[Wei et~al.(2026)Wei, Shi, Lin, He, Liu, Sang, Bei, Ning, Zou, Li,
  et~al.]{wei2026evo}
Tianxin Wei, Zhan Shi, Minhua Lin, Bing He, Zewen Liu, Yisi Sang, Yuanchen Bei,
  Xuying Ning, Jiaru Zou, Ting-Wei Li, et~al.
\newblock Evo-harness: Context-to-harness skill compilation for self-evolving
  agents.
\newblock \emph{arXiv preprint arXiv:2608.15071}, 2026.

\bibitem[Wu et~al.(2024)Wu, Li, Chen, and
  Li]{wu2024longvideobenchbenchmarklongcontextinterleaved}
Haoning Wu, Dongxu Li, Bei Chen, and Junnan Li.
\newblock Longvideobench: A benchmark for long-context interleaved
  video-language understanding, 2024.
\newblock \url{https://arxiv.org/abs/2407.15754}.

\bibitem[Wu et~al.(2026)Wu, Ren, Li, Li, Yang, Zhang, Gu, Yang,
  Batista-Navarro, Zhang, Liu, Zhou, Dai, and
  Lin]{wu2026modularrsimodulargeneralizablerecursive}
Siwei Wu, Jincheng Ren, Yizhi Li, Haau-Sing Li, Chengran Yang, Yuxuan Zhang,
  Weicheng Gu, Jian Yang, Riza Batista-Navarro, Chuanyi Zhang, Xianglong Liu,
  Ming Zhou, Bryan Dai, and Chenghua Lin.
\newblock Modularrsi: Modular and generalizable recursive harness
  self-improvement, 2026.
\newblock \url{https://arxiv.org/abs/2609.14857}.

\bibitem[Xia et~al.(2026)Xia, Han, Wang, Chen, Zhang, Lee, Huang, Yu, CuiZhu,
  Ming, Yao, Gokturk, Pfister, and
  Lee]{xia2026rrsiregularizedrecursiveselfimprovement}
Peng Xia, Rujun Han, Zifeng Wang, Yanfei Chen, Yufan Zhang, Yoonho Lee,
  Chengsong Huang, Han Yu, Zhongying CuiZhu, Yifei Ming, Huaxiu Yao, Burak
  Gokturk, Tomas Pfister, and Chen-Yu Lee.
\newblock Rrsi: Regularized recursive self-improvement of agent harnesses,
  2026.
\newblock \url{https://arxiv.org/abs/2609.24972}.

\bibitem[Xu and Chen(2026)]{xu2026videoharnessrsi}
Guoyang Xu and Hao Chen.
\newblock Videoharness-rsi: Recursive harness self-improvement for long-video
  understanding with frozen vision-language models.
\newblock \emph{arXiv preprint arXiv:2608.24302}, 2026.

\bibitem[Xu and Raginsky(2017)]{xu2017information}
Aolin Xu and Maxim Raginsky.
\newblock Information-theoretic analysis of generalization capability of
  learning algorithms.
\newblock In \emph{Advances in Neural Information Processing Systems},
  volume~30, 2017.

\bibitem[Yang et~al.(2024)Yang, Wang, Lu, Liu, Le, Zhou, and
  Chen]{yang2024large}
Chengrun Yang, Xuezhi Wang, Yifeng Lu, Hanxiao Liu, Quoc~V Le, Denny Zhou, and
  Xinyun Chen.
\newblock Large language models as optimizers.
\newblock In \emph{International Conference on Learning Representations},
  volume 2024, pages 12028--12068, 2024.

\bibitem[Ye et~al.(2026)Ye, Zhao, Wang, Hu, Sun, Liang, Jia, Huang, Tian,
  Zhang, Jin, Zhang, Patel, and Fu]{ye2026videoresearcher}
Dingqiang Ye, Dongdi Zhao, Kaishen Wang, Qingqiao Hu, Jingchen Sun, Yijun
  Liang, Yuqi Jia, Yiqiao Huang, Yunjie Tian, Jiaxing Zhang, Chuanyang Jin,
  Ke Zhang, Vishal~M. Patel, and Di~Fu.
\newblock Videoresearcher: Self-improving tool design for long-video
  understanding.
\newblock \emph{arXiv preprint arXiv:2609.19664}, 2026.

\bibitem[Yeo et~al.(2026)Yeo, Kim, Yoon, and Hwang]{yeo2026worldmm}
Woongyeong Yeo, Kangsan Kim, Jaehong Yoon, and Sung~Ju Hwang.
\newblock Worldmm: Dynamic multimodal memory agent for long video reasoning.
\newblock In \emph{Proceedings of the IEEE/CVF Conference on Computer Vision
  and Pattern Recognition}, pages 25599--25609, 2026.

\bibitem[Yuan et~al.(2026)Yuan, Ni, Liu, Wang, Zhou, Liang, Zhao, Cao, Wen, and
  Dou]{yuan2026momentseeker}
Huaying Yuan, Jian Ni, Zheng Liu, Yueze Wang, Junjie Zhou, Zhengyang Liang,
  Bo~Zhao, Zhao Cao, Ji-Rong Wen, and Zhicheng Dou.
\newblock Momentseeker: A task-oriented benchmark for long-video moment
  retrieval.
\newblock \emph{Advances in Neural Information Processing Systems}, 38, 2026.

\bibitem[Yuksekgonul et~al.(2024)Yuksekgonul, Bianchi, Boen, Liu, Huang,
  Guestrin, and Zou]{yuksekgonul2024textgrad}
Mert Yuksekgonul, Federico Bianchi, Joseph Boen, Sheng Liu, Zhi Huang, Carlos
  Guestrin, and James Zou.
\newblock Textgrad: Automatic" differentiation" via text.
\newblock \emph{arXiv preprint arXiv:2406.07496}, 2024.

\bibitem[Zhang et~al.(2025{\natexlab{a}})Zhang, Chen, Wan, Chang, Cheng, Wang,
  Hu, and Bai]{zhang2025evoflow}
Guibin Zhang, Kaijie Chen, Guancheng Wan, Heng Chang, Hong Cheng, Kun Wang,
  Shuyue Hu, and Lei Bai.
\newblock Evoflow: Evolving diverse agentic workflows on the fly.
\newblock \emph{arXiv preprint arXiv:2502.07373}, 2025{\natexlab{a}}.

\bibitem[Zhang et~al.(2026{\natexlab{a}})Zhang, Hu, Lu, Lange, and
  Clune]{zhang2026darwin}
Jenny Zhang, Shengran Hu, Cong Lu, Robert Lange, and Jeff Clune.
\newblock Darwin g{\"o}del machine: open-ended evolution of self-improving
  agents.
\newblock In \emph{International Conference on Learning Representations},
  volume 2026, pages 104223--104294, 2026{\natexlab{a}}.

\bibitem[Zhang et~al.(2025{\natexlab{b}})Zhang, Xiang, Yu, Teng, Chen, Chen,
  Zhuge, Cheng, Hong, Wang, et~al.]{zhang2025aflow}
Jiayi Zhang, Jinyu Xiang, Zhaoyang Yu, Fengwei Teng, Xionghui Chen, Jiaqi Chen,
  Mingchen Zhuge, Xin Cheng, Sirui Hong, Jinlin Wang, et~al.
\newblock Aflow: Automating agentic workflow generation.
\newblock In \emph{International Conference on Learning Representations},
  volume 2025, pages 34040--34077, 2025{\natexlab{b}}.

\bibitem[Zhang et~al.(2026{\natexlab{b}})Zhang, Jia, Guo, Li, Li, Li, and
  Lu]{zhang2026deep}
Xiaoyi Zhang, Zhaoyang Jia, Zongyu Guo, Jiahao Li, Bin Li, Houqiang Li, and Yan
  Lu.
\newblock Deep video discovery: Agentic search with tool use for long-form
  video understanding.
\newblock \emph{Advances in Neural Information Processing Systems},
  38:\penalty0 89863--89895, 2026{\natexlab{b}}.

\end{thebibliography}

\clearpage
\newpage
\beginappendix
\newtheorem{proposition}{Proposition}[section]

\section*{Overview of the Appendix}

The appendix follows the order of the paper, and Table~\ref{tab:study_index} indexes the experiments with their scale and the appendix that documents it.
\begin{itemize}[leftmargin=*,itemsep=1pt,topsep=2pt]
    \item Appendix~\ref{sec:experimental_details} describes the models, benchmarks, splits, inference budgets, seeded harnesses, and metrics.
    \item Appendix~\ref{app:basemethod} specifies \basemethod{}: its evolution loop and study configurations (Section~\ref{sec:initial_investigation}).
    \item Appendix~\ref{app:verdict_model} formalizes the verdict model of Eq.~(\ref{eq:verdict}), and Appendices~\ref{app:overfitting_analysis}--\ref{app:early_experience_trap} analyze the three failure modes of Section~\ref{sec:meta_harness_failure_modes}, each with its protocol, full results, propositions and proofs, and robustness checks.
    \item Appendix~\ref{sec:simple_remedies} reports the parent-selection and initialization studies behind the simple remedies discussed in Section~\ref{sec:meta_harness_failure_modes}.
    \item Appendix~\ref{app:method_details} gives the implementation, statistical procedures, and configuration of \method{} (Section~\ref{sec:method}).
    \item Appendix~\ref{app:warmablation} gives the protocol and additional analyses of the warm-up ablation (Section~\ref{sec:exp:warmablation}).
    \item Appendix~\ref{app:unseen_videos} re-scores held-out results on videos absent from validation, and Appendix~\ref{app:method_diagnostics} reports failure-mode diagnostics of \method{} (Section~\ref{sec:experiments}).
    \item Appendix~\ref{app:prompts} reproduces the main proposer prompts of \method{}.
\end{itemize}

\begin{table*}[t]
\centering
\small
\caption{\textbf{Index of experiments.} All experiments use the frozen Qwen3-VL-8B-Instruct executor and, except for the proposer ablation of Section~\ref{sec:exp:proposerablation}, Claude Opus~5 as the proposer. \emph{Budget}: iterations $\times$ candidates per iteration of one run. \emph{Candidates}: scheduled total over the study's runs; for the early-experience trap, the candidates actually proposed (the ``over 4{,}000'' of the abstract). Studies share runs, so the column does not sum. BRSI: \basemethod{}; SRSI: \method{}. VM: Video-MME; LVB: LongVideoBench; OVO: OVO-Bench; MS: MomentSeeker.}
\label{tab:study_index}
\setlength{\tabcolsep}{4pt}
\begin{tabularx}{\textwidth}{@{}l l >{\raggedright\arraybackslash}X r c r l@{}}
\toprule
Study & Benchmarks & Factor varied & Runs & Budget & Candidates & Where \\
\midrule
\multicolumn{7}{@{}l}{\textit{Diagnosis of \basemethod{} (Section~\ref{sec:initial_investigation})}} \\
\quad Effectiveness & VM & -- (retrieval only) & 1 & $50\times3$ & 150 & Fig.~\ref{fig:initial_metaharness} \\
\quad Validation overfitting & LVB & proposer feedback; validation fraction & 27 & $30\times3$ & 2{,}430 & App.~\ref{app:overfitting_analysis} \\
\quad Search-budget dilution & VM & number of mechanism families (1, 2, 3, 6) & 11 & varies$^{a}$ & 1{,}820 & App.~\ref{app:search_space_analysis} \\
\quad Early-experience trap & VM, LVB & -- (re-analysis of existing runs) & 44 & varies$^{b}$ & 4{,}806 & App.~\ref{app:early_experience_trap} \\
\midrule
\multicolumn{7}{@{}l}{\textit{Simple remedies (Section~\ref{sec:meta_harness_failure_modes})}} \\
\quad Parent selection & VM, LVB & parent-suggestion rule & 18 & $30\times3$ & 1{,}620 & App.~\ref{sec:search_algorithm} \\
\quad Initialization & VM & starting harness (simple seeds vs.\ AKS) & 7 & $50\times3$ & 1{,}050 & App.~\ref{sec:better_initialization} \\
\midrule
\multicolumn{7}{@{}l}{\textit{Main results and ablations (Section~\ref{sec:experiments})}} \\
\quad Main results: BRSI & all four & -- & 12 & $50\times3$ & 1{,}800 & Tab.~\ref{tab:main_results} \\
\quad Main results: SRSI & all four & -- & 12 & staged$^{c}$ & $\leq$1{,}872 & Tab.~\ref{tab:main_results} \\
\quad Warm-up ablation & VM & warm-up slot allocation & 6 & staged$^{c}$ & $\leq$936 & App.~\ref{app:warmablation} \\
\quad Axis selection & VM & who defines the branch axes & 9 & staged$^{c}$ & $\leq$1{,}404 & Fig.~\ref{fig:axesablation} \\
\bottomrule
\end{tabularx}
\\[2pt]
\raggedright\footnotesize
$^{a}$Per space in Table~\ref{tab:dilution_ladder}. \quad
$^{b}$Per setting in Table~\ref{tab:eet_settings}. \quad
$^{c}$At most 156 candidates per run (Eq.~(\ref{eq:method_budget})); totals are upper bounds.
\end{table*}

\section{Experimental Setup}
\label{sec:experimental_details}

This section describes the models, benchmarks, inference budgets, seeded harnesses, and metrics shared by all experiments; Table~\ref{tab:study_index} lists the scale of each experiment.

\subsection{Models}
\label{sec:experimental_models}

The \textbf{executor} is a frozen Qwen3-VL-8B-Instruct~\citep{bai2025qwen3vltechnicalreport} with greedy decoding. It is the only model that sees video or answers questions; ``roles'' and ``agents'' inside a harness are prompts to the same model. Executor and tool outputs are cached by content, so re-evaluating an unchanged harness reproduces its outputs (batching varies seeded-harness scores by at most $\pm0.5$ points). The \textbf{proposer} is Claude Opus~5~\citep{anthropic2026claudeopus5}, run as a Claude Code agent at maximum reasoning effort in a new session per iteration; it writes harness code but never answers questions or runs the benchmark. Query-conditioned retrieval ranks a per-video pool of frames (4\,fps, at most 4{,}096) with CLIP ViT-B/32 and ViT-L/14-336~\citep{radford2021learning}; open-vocabulary detection uses Grounding DINO~\citep{liu2024grounding}.

\subsection{Benchmarks and Data Splits}
\label{sec:experimental_datasets}

Each benchmark has one fixed validation/held-out partition (Table~\ref{tab:dataset_splits}), drawn once and reused by every method, seed, and study except the validation ladder of Appendix~\ref{app:overfitting_analysis}. Every partition in this paper, including that ladder, is a seeded random split of questions, except on MomentSeeker, where whole source videos are assigned to one side; on the other benchmarks, a held-out question may come from a video that also has validation questions (Appendix~\ref{app:unseen_videos}). Evolution and selection use only the validation split, and held-out results never reach the proposer. All benchmarks are used video-only: subtitles and metadata such as duration categories or task labels are withheld from harness code.
\begin{itemize}[leftmargin=*,itemsep=1pt,topsep=2pt]
    \item \textbf{Video-MME}: all 2{,}700 questions on 900 videos.
    \item \textbf{LongVideoBench}: the released validation set (the test set has no public answers), after removing items with inconsistent answer indices.
    \item \textbf{OVO-Bench}: the backward-tracing and real-time-perception multiple-choice modes, each evaluated on the clip truncated at query time so that no future frames are visible.
    \item \textbf{MomentSeeker}: text-query moment retrieval, where each source video lies entirely on one side of the split.
\end{itemize}

\begin{table}[t]
\centering
\caption{\textbf{Benchmark splits.} Seeded random splits; the split unit is the question, or the source video for MomentSeeker.}
\label{tab:dataset_splits}
\small
\begin{tabular}{@{}lccccl@{}}
\toprule
Benchmark & Pool & Validation & Held-out & Split unit & Metric \\
\midrule
Video-MME      & 2{,}700 & 450 & 2{,}250 & question & accuracy \\
LongVideoBench & 1{,}337 & 446 & 891  & question & accuracy \\
OVO-Bench      & 1{,}468 & 489 & 979  & question & accuracy \\
MomentSeeker   & 1{,}000 & 261 & 739  & video & mIoU \\
\bottomrule
\end{tabular}
\end{table}

\subsection{Inference Budgets}
\label{sec:experimental_budgets}

Table~\ref{tab:inference_budgets} lists the default settings and per-question limits. A harness may override the frame count, pixel budget, and output length; MomentSeeker defaults to more, lower-resolution frames, since temporal grounding depends mainly on coverage. A primitive that fails or exhausts its budget returns an empty result, so every harness must return a valid answer without any tool. Search spaces without perception tools, including the LongVideoBench runs of Appendix~\ref{app:overfitting_analysis}, remove them from the interface.

\begin{table}[t]
\centering
\caption{\textbf{Default inference settings and per-question limits.} The frame and pixel values are defaults that a harness may override per call; the limits are enforced by the interface.}
\label{tab:inference_budgets}
\small
\setlength{\tabcolsep}{4pt}
\begin{tabular}{@{}lcc@{}}
\toprule
& Multiple choice & Temporal grounding \\
\midrule
Default frames per pass & 16 & 32 \\
Default pixels per frame & 230{,}400 & 151{,}200 \\
Default output tokens & 16 & 64 \\
Maximum frames per pass & 768 & 256 \\
\midrule
Executor calls per question & \multicolumn{2}{c}{12} \\
Tool calls per question & \multicolumn{2}{c}{60} \\
Pixel re-rendering calls per question & \multicolumn{2}{c}{12} \\
Working memory (characters) & \multicolumn{2}{c}{900} \\
\bottomrule
\end{tabular}
\end{table}

\subsection{Seeded Harnesses}
\label{sec:experimental_baselines}

Six deliberately simple, hand-written harnesses (Table~\ref{tab:seed_harnesses}) seed every full-harness run and serve as fixed baselines. \textsc{Self-reflection} and \textsc{Role Debate} read the frames of \textsc{Retrieval}, and \textsc{Tool-gated Selection} reads as many, so their differences from \textsc{Retrieval} reflect the added mechanism rather than the frame budget. On MomentSeeker they predict a time span instead of an option letter. Where perception tools are disabled, \textsc{Tool-gated Selection} is replaced by \textsc{Dense Uniform}, one call on 64 uniform frames. Spaces with one or two mechanism families start instead from five simple seeds: uniform and block-center sampling, CLIP top-$k$ retrieval, and two space-specific retrieval or selection seeds.

\begin{table*}[t]
\centering
\small
\caption{\textbf{Seeded harnesses.} Frames: mean frames shown to the executor per Video-MME question, summed over calls.}
\label{tab:seed_harnesses}
\begin{tabularx}{\textwidth}{@{}>{\raggedright\arraybackslash}p{0.22\textwidth} >{\raggedright\arraybackslash}X c@{}}
\toprule
Harness (code name) & Procedure & Frames \\
\midrule
\textsc{Uniform} (\texttt{uniform}) & Samples frames uniformly over the video and answers in one grounded call. & 16 \\
\textsc{Retrieval} (\texttt{clip\_topk}) & Ranks pool frames by CLIP similarity to the question and options, backfills the selection with uniform frames to the budget, and answers in one call. & 16 \\
\textsc{Self-reflection} (\texttt{reflect\_loop}) & Answers on the retrieved frames, confronts the executor with its own answer on the same frames, and takes the revision (at most two rounds). & 32 \\
\textsc{Tool-gated Selection} (\texttt{tool\_survey}) & Asks the executor which entities would resolve the question, detects them in a retrieved candidate pool, and answers on the frames that contain them. & 16 \\
\textsc{Role Debate} (\texttt{debate\_pair}) & A proponent argues for the leading option, a skeptic for the strongest rival, and an arbiter decides after looking at the frames again with both arguments. & 80 \\
\textsc{Memory Survey} (\texttt{memo\_multiturn}) & Splits the video into shots, writes one observation per shot keyframe into working memory, answers with the memory digest, and re-retrieves once on a cue named by the executor if the answer remains uncertain. & 56 \\
\bottomrule
\end{tabularx}
\end{table*}

\subsection{Evaluation Metrics}
\label{sec:experimental_metrics}

Multiple-choice benchmarks report accuracy, parsing the first valid option letter from the harness output (none scores as incorrect). MomentSeeker reports mIoU, the mean temporal IoU of the predicted span, clamped to the video and maximized over annotated intervals (unparsable predictions score zero); it is also the evolution objective. A harness's cost is its mean number of frames shown to the executor per question, summed over calls. Evolved results average three seeds and report the standard deviation across seeds. Unless noted otherwise, paired comparisons of two harnesses use McNemar's test~\citep{mcnemar1947note}, contrasts between groups of runs use Welch's $t$-test or exact permutation tests, and bootstrap intervals resample runs rather than questions, since candidates of one run are dependent.

\section{\basemethod{} Details}
\label{app:basemethod}

This section specifies \basemethod{} as used in Section~\ref{sec:initial_investigation}: the evolution loop (Appendix~\ref{app:basemethod_loop}) and the configurations of the individual studies (Appendix~\ref{app:basemethod_variants}).

\subsection{Evolution Loop}
\label{app:basemethod_loop}

A harness is a Python class whose \texttt{answer} method receives a context object for one question---the video, the question, and its options---and returns an answer. Every access to frames, the executor, tools, or memory goes through the context, which enforces the per-question limits of Table~\ref{tab:inference_budgets}. Algorithm~\ref{alg:basemethod} summarizes the loop.

\begin{algorithm}[ht]
\caption{\basemethod{}: the propose--evaluate--accumulate loop.}
\label{alg:basemethod}
\small
\begin{algorithmic}[1]
\Require frozen VLM $M$; proposer $P$; seed harnesses $\mathcal{H}_0$; splits $\mathcal{S}_{\mathrm{val}}$, $\mathcal{S}_{\mathrm{test}}$; iterations $T$; candidates per iteration $k$
\State $\mathcal{D}_0\gets\{(H,\textsc{Eval}(H,\mathcal{S}_{\mathrm{val}})):H\in\mathcal{H}_0\}$ \Comment{seed archive}
\For{$t=1,\ldots,T$}
    \State start a new session of $P$ in a workspace exposing $\mathcal{D}_{t-1}$, the traces, and $\mathcal{I}_{\mathrm{val}}$
    \State $P$ writes $k$ harness files and a manifest: name, declared parent, hypothesis, mechanism tags
    \For{each candidate $H$ in the manifest}
        \If{$H$ fails to import or violates a leak rule} record the failure and skip $H$ \EndIf
        \State $\mathcal{E}\gets\textsc{Eval}(H,\mathcal{S}_{\mathrm{val}})$ \Comment{per-question answers, correctness, cost, traces}
        \State $\mathcal{D}_t\gets\mathcal{D}_{t-1}\cup\{(H,\mathcal{E})\}$; update the Pareto frontier over (accuracy, mean frames)
    \EndFor
\EndFor
\State evaluate the frontier, the validation-best harness, and $\mathcal{H}_0$ on $\mathcal{S}_{\mathrm{test}}$ \Comment{never shown to $P$}
\State \Return the validation-best harness \Comment{deployed}
\end{algorithmic}
\end{algorithm}

\paragraph{Proposer sessions.}
Each iteration starts a new session of the proposer (Appendix~\ref{sec:experimental_models}) in the run's workspace. A task description specifies the objective, the harness interface, the mechanism families that the search space charters, and the candidate rules: every iteration must add $k$ new candidates, each changing a mechanism rather than only constants. It prescribes a workflow: analyze the archive and the execution traces, prototype on a stubbed or real executor, implement $k$ new harnesses, check that each imports and passes the leak guard, and register them in a manifest. The manifest declares, for each candidate, the parent it modifies, a one-sentence hypothesis, and the mechanisms it uses. Unless stated otherwise, $k=3$.

\paragraph{Persistent memory.}
Besides the archive, the only state that survives an iteration is a persistent memory directory of the proposer's agent environment, which later sessions load automatically (up to 160\,KB per run in our logs). The proposer uses it to record its interpretation of verdicts, including lists of directions it considers exhausted; Appendix~\ref{app:eet_traces} shows how such entries propagate early verdicts.

\paragraph{Evaluation.}
Each candidate is evaluated once on the full validation split, and its score is never recomputed. Executor and tool outputs are cached by content, so calls that repeat an earlier input, within or across candidates, return the cached output. Before evaluation, a static leak guard rejects any candidate that accesses the file system, the network, or evaluation logs; references dataset-identity or metadata handles such as the video path, question and video IDs, or the task type; hashes text; or embeds lookup tables (collection literals with more than 12 strings or strings longer than 4{,}000 characters). Harness code sees only the question, the options, and the video's frame count, frame rate, and duration. A candidate that fails to import, violates a leak rule, or crashes during evaluation is recorded as failed and cannot become the champion.

\paragraph{Frontier and deployment.}
The loop maintains the Pareto frontier of validation performance against mean frames per question. After the last iteration, the frontier members, the validation-best harness, and the seeded harnesses are evaluated once on the held-out split. The deployed harness is the validation-best harness. Held-out results are stored outside the proposer's workspace.

\subsection{Study Configurations}
\label{app:basemethod_variants}

\paragraph{Budgets.}
Unless stated otherwise, retrieval-family runs evolve for 50 iterations and full-harness runs for 30 (50 for \basemethod{} in the main results of Table~\ref{tab:main_results}, which use Multimodal feedback and no parent suggestions on every benchmark), with three candidates per iteration; one perception-tools run uses 100 iterations of five candidates (Table~\ref{tab:dilution_ladder}).

\paragraph{Parent suggestions.}
The default loop leaves the choice of parent entirely to the proposer. The parent-selection study of Appendix~\ref{sec:search_algorithm} adds a list of suggested parents to the task description and varies only the rule that produces it. All other runs, including the full-harness runs of the search-space ladder and the main results, receive no parent suggestions; deployment is the validation-best harness in all cases.

\section{Evolving Capability of \basemethod{}}
\label{app:evolving-capability}

This section gives the experimental setting behind Section~\ref{sec:metaharness_effectiveness} and explains its three design choices: the evolved mechanism family, the validation/held-out split, and the iteration budget.

\subsection{Experimental Setting}
\label{app:evolving-capability-setting}

\paragraph{Evolving retrieval for a fair comparison.}
The question of Section~\ref{sec:metaharness_effectiveness} is whether recursive evolution can match expert manual design, so the comparison needs a strong hand-designed reference. We use Adaptive Keyframe Sampling (AKS)~\citep{tang2025adaptive}, which is a retrieval method: it scores frames by their relevance to the question and allocates the frame budget accordingly. We therefore let \basemethod{} evolve only the retrieval strategy in the main text. Evolved and hand-designed harnesses then differ in the same mechanism and run on the same frozen executor, so a gap between them reflects how the mechanism was designed rather than which mechanism was used. Evolving other or more mechanism families would give \basemethod{} components that AKS lacks, and a win could no longer be credited to evolution alone. Appendix~\ref{app:evolving-capability-families} reports evolution in other mechanism families and on other datasets.

\paragraph{A 1:5 validation/held-out split.}
We split Video-MME at random into 450 validation and 2{,}250 held-out questions, a ratio of 1:5 (Table~\ref{tab:dataset_splits}). This split keeps most of the benchmark away from evolution, for two reasons.
\begin{itemize}[leftmargin=*,itemsep=1pt,topsep=2pt]
    \item \textbf{Evaluation cost.} Every candidate is run in full on the validation split, which takes a frozen-VLM call on the sampled frames of every question. Evaluation is therefore the slow step of each iteration, and a small validation split shortens every one of the 150 evaluations in a run.
    \item \textbf{A stricter test of generalization.} A small validation split gives the proposer less of the benchmark to fit, while the large held-out split measures the deployed harness precisely: with 2{,}250 questions, the standard error of a held-out accuracy near $68\%$ is about one point, against about two points on the 450 validation questions. A larger validation share would let evolution fit a larger part of the benchmark itself and leave a smaller, noisier held-out split to detect it. Under the 1:5 split, a harness evolved on one sixth of the data must match the hand-designed methods on the remaining five sixths, which it never sees. That the evolved harness still matches the strongest of them on this split (Figure~\ref{fig:initial_metaharness}(b)) therefore strengthens, rather than weakens, the conclusion that \basemethod{} is competitive with human design.
\end{itemize}

\paragraph{Iteration budget.}
Each run evolves for 50 iterations of three candidates, 150 candidate evaluations in total.

\subsection{Evolution Across Mechanism Families and Datasets}
\label{app:evolving-capability-families}

The retrieval-only comparison isolates one mechanism for a fair test against AKS, but the evolving capability of \basemethod{} is not specific to retrieval or to Video-MME.

\paragraph{Other mechanism families.}
The search-space ladder of Appendix~\ref{app:search_space_analysis} evolves, on the same Video-MME split, retrieval together with frame selection, with perception tools, and the full harness of six families, which adds reflection, role-based collaboration, and working memory. In every space, evolution raises validation accuracy over the best seeded harness by $5.1$--$7.6$ points within 87 candidates, and the deployed champions reach $65.7$--$66.8\%$ held-out accuracy (Table~\ref{tab:dilution_summary}), above the best fixed baseline on Video-MME (\texttt{clip\_topk}, $63.4\%$; Table~\ref{tab:main_results}). Evolution thus improves harnesses whichever mechanism families it may edit; what the families change is how efficiently it does so (Section~\ref{sec:search_space_efficiency}).

\paragraph{Other datasets.}
The main results evolve the full harness on all four benchmarks (Table~\ref{tab:main_results}). On held-out data, \basemethod{} exceeds the strongest fixed baseline on every benchmark: $66.5$ versus $63.4\%$ on Video-MME, $66.6$ versus $62.1\%$ on LongVideoBench, $67.2$ versus $55.4\%$ on OVO-Bench, and $19.0$ versus $18.0$ mIoU on MomentSeeker, $5.12$ points on average. Evolution therefore transfers across general, long-form, and streaming question answering and temporal grounding, although, as Section~\ref{sec:meta_harness_failure_modes} shows, its validation progress overstates these held-out gains.

\section{A Model of Validation Verdicts}
\label{app:verdict_model}

This section derives the verdict model of Eq.~(\ref{eq:verdict}); Appendices~\ref{app:overfitting_analysis}--\ref{app:early_experience_trap} instantiate it for the three failure modes.

\paragraph{Scores and instance-specific fit.}
Let $y_H(q)\in\{0,1\}$ indicate that harness $H$ answers question $q$ correctly, with questions drawn from a population $\mathcal{Q}$. On a validation split $\mathcal{S}_{\mathrm{val}}$ of $n$ questions,
\begin{equation}
    V(H)=\frac{100}{n}\sum_{q\in\mathcal{S}_{\mathrm{val}}}y_H(q),
    \qquad
    \mu(H)=100\,\mathbb{E}_{q\sim\mathcal{Q}}\bigl[y_H(q)\bigr],
\end{equation}
and $T(H)$ is held-out accuracy, an unbiased estimate of $\mu(H)$ because the held-out split never influences search. $V(H)$ is not, since $H$ is built from validation feedback. We decompose $V(H)=\mu(H)+S(H)+e(H)$ with
\begin{equation}
    S(H)=\mathbb{E}\bigl[V(H)\mid H\bigr]-\mu(H),
    \qquad
    \mathbb{E}\bigl[e(H)\mid H\bigr]=0 .
    \label{eq:app_S}
\end{equation}
The \emph{instance-specific fit} $S(H)$ is the systematic, validation-specific part of the score; $S\equiv0$ for any harness chosen without access to $\mathcal{S}_{\mathrm{val}}$, such as a hand-written baseline. In psychometric terms~\citep{lord1968statistical}, validation measures $\mu+S$ reliably but only $\mu$ validly, and the two separate exactly when $S$ varies across candidates.

\paragraph{Paired verdicts and the behavioral footprint.}
A candidate is compared with the incumbent $H^{\mathrm{inc}}$ on the same questions. With $D_q=y_H(q)-y_{H^{\mathrm{inc}}}(q)$ and $n_{10}$, $n_{01}$ the numbers of questions with $D_q=1$ and $D_q=-1$,
\begin{equation}
    \hat{\delta}=\frac{100\,(n_{10}-n_{01})}{n},
    \qquad
    d=n_{10}+n_{01},
    \label{eq:app_verdict_counts}
\end{equation}
and, treating questions as exchangeable given the pair,
\begin{equation}
    \operatorname{Var}(\hat{\delta})
    =\frac{10^{4}}{n}\Bigl(\frac{\mathbb{E}\,d}{n}-\Bigl(\frac{\mathbb{E}\,\hat{\delta}}{100}\Bigr)^{2}\Bigr)
    \approx\frac{10^{4}\,d}{n^{2}},
    \label{eq:app_paired_var}
\end{equation}
the scale of McNemar's statistic~\citep{mcnemar1947note}. We call $d$ the \emph{behavioral footprint} and use the paired standard error $\sigma_H=100\sqrt{d_H}/n$ and $z=\hat{\delta}/\sigma_H$. The key consequence is that an edit changing more answers yields a noisier verdict, whether or not the changes help. Exact clones ($d=0$) get $z=0$ and are excluded from averages of $\sigma$. Because the calculation ignores adaptive reuse of the validation set and within-video dependence, $|z|<1.96$ is only a nominal marker of an unresolved comparison.

\paragraph{Incumbent optimism.}
The incumbent is the validation maximum so far, so the noise in its cached score is biased upward, $\omega=\mathbb{E}\bigl[e(H^{\mathrm{inc}})\bigr]>0$, even though $\mathbb{E}[e(H)\mid H]=0$ for any fixed $H$ (the optimizer's curse~\citep{smith2006optimizer}); for $m$ equally accurate harnesses with independent Gaussian noise of standard deviation $s$, $\omega=s\,c_m$ with $c_m=\mathbb{E}\max_{i\leq m}Z_i\approx\sqrt{2\ln m}$. Every later verdict compares a candidate against this selected score and therefore inherits $-\omega$. Subtracting the two decompositions,
\begin{equation*}
    \hat{\delta}=V(H)-V(H^{\mathrm{inc}})
    =\underbrace{\mu(H)-\mu(H^{\mathrm{inc}})}_{\Delta}
    +\underbrace{S(H)-S(H^{\mathrm{inc}})}_{S}
    -\omega
    +\underbrace{e(H)-e(H^{\mathrm{inc}})+\omega}_{\varepsilon},
\end{equation*}
which is Eq.~(\ref{eq:verdict}) with $S$ the candidate's fit relative to the incumbent's and $\mathbb{E}\,\varepsilon=0$. Each validation-specific part of the incumbent's score thus enters once: its systematic fit $S(H^{\mathrm{inc}})$ through $S$, and the selection bias of its noise through $\omega$.

\paragraph{Three ways a verdict fails.}
Evolution improves $\mu$ only insofar as $\hat{\delta}$ tracks $\Delta$. \textbf{Validation overfitting} (App.~\ref{app:overfitting_analysis}): content-revealing feedback enlarges $S$, which selection rewards. \textbf{Search-budget dilution} (App.~\ref{app:search_space_analysis}): larger spaces lower the evidence per hypothesis through fewer evaluations per mechanism combination and larger footprints $d$. \textbf{Early-experience trap} (App.~\ref{app:early_experience_trap}): single low-SNR verdicts, offset by $-\omega$, are rarely revisited. The model is a lens, not a claim of exact Gaussianity; each appendix states its assumptions.

\section{Validation Overfitting}
\label{app:overfitting_analysis}

This section supports Section~\ref{sec:overfitting}. Its conclusions are:
\begin{itemize}[leftmargin=*,itemsep=1pt,topsep=2pt]
    \item \textbf{Richer feedback, larger gaps.} The mean validation--held-out gap of evolved harnesses grows from $2.0$ (Blindfolded) to $2.9$ (Trace) and $3.2$ points (Multimodal), while hand-written baselines show none (Appendix~\ref{app:overfit_results}).
    \item \textbf{Theory.} Selecting on validation accuracy transfers only a fraction $\beta$ of the validation advantage, and the non-transferable fit a harness can acquire is bounded by the information its feedback carries about the validation sample, which grows along nested channels (Appendix~\ref{app:overfit_theory}).
    \item \textbf{Measurement.} Validation scores are equally reliable in every arm, yet $\beta$ falls from $0.54$ to $0.30$, and the lost share reappears as instance-specific fit; the model predicts the winners' held-out accuracy to within $0.34$ points (Appendix~\ref{app:overfit_beta}).
    \item \textbf{Mechanism.} Proposers that see question text quote validation questions in their code $1.55\times$ (Trace) and $1.27\times$ (Multimodal) as often as chance, whereas Blindfolded proposers do not (Appendix~\ref{app:overfit_channel}).
\end{itemize}

\subsection{Protocol}
\label{app:proposer_information}

All 27 runs use LongVideoBench (video only), the frozen Qwen3-VL-8B-Instruct executor, and the full-harness space with perception tools disabled, evolving for $30$ iterations of three candidates. The 1{,}337 scorable questions are split at random, as in Appendix~\ref{sec:experimental_datasets}, into nested validation sets of 33\% (446 / 891 validation / held-out questions), 50\% (668 / 669), and 75\% (1{,}003 / 334). Each arm runs three seeds per ladder. Only validation-frontier candidates are evaluated on the held-out split, which the proposer never sees. The three arms differ only in what the proposer observes about the validation set (Table~\ref{tab:overfit_arms}); each arm's feedback is a deterministic function of the next one's. After dropping 13 incomplete validation/held-out pairs, 281 proposed candidates remain (89 Blindfolded, 103 Trace, 89 Multimodal), against $6\times3=18$ distinct hand-written baseline values.

\begin{table}[h]
\centering
\small
\caption{\textbf{Information arms} (three seeds $\times$ three ladders per arm).}
\label{tab:overfit_arms}
\begin{tabular}{@{}lp{0.72\textwidth}@{}}
\toprule
Arm & What the proposer observes about the validation set \\
\midrule
Blindfolded & Per-question predicted letter and correctness, and about ten behavioral counters (model and tool calls, frames, rendering, memory). No question text, options, or frames. \\
Trace & Blindfolded feedback plus each question, its options, the full interaction trace, and the first raw answer. \\
Multimodal & Trace feedback plus on-demand rendered validation frames, and a written anti-overfitting protocol. \\
\bottomrule
\end{tabular}
\end{table}

\subsection{Results}
\label{app:overfit_results}

Table~\ref{tab:overfit_results} reports the gaps behind Figure~\ref{fig:diagnosis}(a). The gap grows monotonically with feedback when ladders are pooled; per-ladder orderings are noisy with three runs per cell. Winners' validation accuracy rises with feedback ($70.2\to71.1\%$) while their held-out accuracy declines ($66.5\to65.7\%$), a trend within the confidence intervals.

\begin{table}[h]
\centering
\small
\caption{\textbf{Validation--held-out gaps (pp).} Candidate CIs are 95\% bootstrap intervals over candidates, with run-clustered intervals in brackets. ``Winner'' is each run's validation argmax, averaged over the nine runs of an arm.}
\label{tab:overfit_results}
\footnotesize
\setlength{\tabcolsep}{3.5pt}
\begin{tabular}{@{}lrcccc@{}}
\toprule
 & $n$ & Mean gap, candidates & Val $>$ held-out & Winner val / held-out (\%) & Winner gap \\
\midrule
Hand-written & 18 & $-0.16$ \,$[-1.07, 0.82]$ & 39\% & -- & -- \\
Blindfolded & 89 & $2.02$ \,$[1.62, 2.39]$ \,$[1.11, 2.95]$ & 84\% & $70.2$ / $66.5$ \,$[65.6, 67.4]$ & $3.69$ \\
Trace & 103 & $2.91$ \,$[2.43, 3.39]$ \,$[2.13, 4.00]$ & 89\% & $71.0$ / $66.3$ \,$[64.4, 67.8]$ & $4.76$ \\
Multimodal & 89 & $3.21$ \,$[2.73, 3.71]$ \,$[2.16, 4.22]$ & 90\% & $71.1$ / $65.7$ \,$[64.8, 66.5]$ & $5.46$ \\
\bottomrule
\end{tabular}
\end{table}

\subsection{Theory: Selection on a Partially Transferable Signal}
\label{app:overfit_theory}

\paragraph{Model.}
For the $m$ candidates $h_1,\ldots,h_m$ of one run, Eq.~(\ref{eq:verdict}) written per candidate gives
\begin{equation}
    V_i=\mu_i+S_i+e_i,\qquad T_i=\mu_i+e'_i,
    \label{eq:overfit_model}
\end{equation}
with population accuracy $\mu_i$, instance-specific fit $S_i$, and split noise $e_i,e'_i$. With $\sigma^2_V=\sigma^2_\mu+\sigma^2_S+\sigma^2_e$, the \emph{transfer coefficient} is
\begin{equation}
    \beta=\frac{\operatorname{Cov}(V,T)}{\operatorname{Var}V}=\frac{\sigma^2_\mu}{\sigma^2_\mu+\sigma^2_S+\sigma^2_e}.
    \label{eq:overfit_beta}
\end{equation}
\paragraph{Relation to Eq.~(\ref{eq:verdict}).}
A candidate's verdict is $\hat{\delta}_i=V_i-V(H^{\mathrm{inc}})$. For candidates compared against the same incumbent, the incumbent's score, including its optimism $\omega$, is a common offset, so $\operatorname{Var}_i\hat{\delta}_i=\operatorname{Var}V$ and the three terms of Eq.~(\ref{eq:verdict}) that vary across candidates are $\Delta_i=\mu_i-\mu(H^{\mathrm{inc}})$, $S_i$, and $\varepsilon_i=e_i$ up to constants. Hence
\[
    \beta=\frac{\operatorname{Var}\Delta}{\operatorname{Var}\Delta+\operatorname{Var}S+\operatorname{Var}\varepsilon},
\]
the share of verdict variance carried by the only transferable term; $1-\beta$ splits into an instance-specific share ($S$) and a noise share ($\varepsilon$), which Appendix~\ref{app:overfit_beta} measures separately. $\omega$ shifts every verdict equally and therefore affects which verdicts clear zero (Appendix~\ref{app:eet_verdict_model}), not $\beta$. The reliability $\kappa$ of Proposition~\ref{prop:evidence_value} is the same ratio estimated on paired verdicts; here we write it on scores because held-out accuracy $T_i$ is observed per candidate.

\begin{proposition}[Selection on a partially transferable signal]
\label{prop:goodhart}
Let $(\mu_i,S_i,e_i,e'_i)$ be independent across $i$, with independent Gaussian components of means $(\bar\mu,\bar S,0,0)$ and variances $(\sigma^2_\mu,\sigma^2_S,\sigma^2_e,\sigma^2_{e'})$. Let $i^\star=\arg\max_i V_i$, $H^\star=h_{i^\star}$, $\bar V=\bar\mu+\bar S$, $\bar T=\bar\mu$, and $c_m=\mathbb{E}[\max_{i\le m}Z_i]$ for i.i.d.\ standard normal $Z_i$. Then
\begin{enumerate}[label=(\roman*),leftmargin=*,itemsep=1pt,topsep=2pt]
    \item $\mathbb{E}[T_i\mid V_1,\ldots,V_m]=\bar T+\beta\,(V_i-\bar V)$ for every $i$;
    \item $\mathbb{E}[T(H^\star)]-\bar T=\beta\,\sigma_V c_m$, whereas $\mathbb{E}[V(H^\star)]-\bar V=\sigma_V c_m$;
    \item $\mathbb{E}[V(H^\star)-T(H^\star)]=\bar S+(1-\beta)\,\sigma_V c_m$;
    \item with $\sigma^2_\mu$, $\sigma^2_e$, and $m\ge2$ fixed, a larger $\sigma_S$ raises $\mathbb{E}[V(H^\star)]$ and the gap but lowers $\mathbb{E}[T(H^\star)]$.
\end{enumerate}
\end{proposition}

\noindent\textit{Proof.}
(i) By independence across candidates, $\mathbb{E}[T_i\mid V_1,\ldots,V_m]=\mathbb{E}[T_i\mid V_i]$. Since $(T_i,V_i)$ is bivariate Gaussian with $\operatorname{Cov}(T_i,V_i)=\sigma^2_\mu$, the conditional mean is linear with slope $\sigma^2_\mu/\sigma^2_V=\beta$.

\noindent (ii) The index $i^\star$ is a function of $(V_1,\ldots,V_m)$, so by the tower property and (i),
\begin{align*}
    \mathbb{E}[T_{i^\star}]
    &=\mathbb{E}\Bigl[\textstyle\sum_i\mathbf{1}\{i^\star=i\}\,\mathbb{E}[T_i\mid V_1,\ldots,V_m]\Bigr]
     =\bar T+\beta\,\bigl(\mathbb{E}[V_{i^\star}]-\bar V\bigr),\\
    \mathbb{E}[V_{i^\star}]-\bar V&=\sigma_V c_m \qquad\text{since } V_i\overset{\text{i.i.d.}}{\sim}\mathcal{N}(\bar V,\sigma^2_V).
\end{align*}
(iii) Subtracting (ii) from $\mathbb{E}[V_{i^\star}]$ and using $\bar V-\bar T=\bar S$ gives the gap.

\noindent (iv) Writing each term as a function of $\sigma_V$, which increases with $\sigma_S$,
\begin{align*}
    \mathbb{E}[V_{i^\star}]&=\bar V+\sigma_V c_m &&\text{(increasing in }\sigma_V),\\
    \mathbb{E}[T_{i^\star}]&=\bar T+\frac{\sigma^2_\mu}{\sigma_V}\,c_m &&\text{(decreasing in }\sigma_V),\\
    \frac{\partial}{\partial\sigma_S}\,\frac{c_m(\sigma^2_S+\sigma^2_e)}{\sigma_V}&=\frac{c_m\,\sigma_S(\sigma^2_V+\sigma^2_\mu)}{\sigma^3_V}>0 &&\text{(gap increasing in }\sigma_S). \tag*{$\square$}
\end{align*}

\paragraph{Remarks.}
With $\sigma_S=0$ and $\bar S=0$ the proposition reduces to the optimizer's curse~\citep{smith2006optimizer}: selection alone overfits by $(\sigma^2_e/\sigma_V)\,c_m$, and adaptive fit adds a second, feedback-dependent term. Statement (i) conditions on $V$, so $\beta$ can be estimated from the frontier candidates that receive held-out evaluations; it also survives correlated lineages as long as each candidate's quality is mean-independent of the others' scores given its own.

\begin{proposition}[Information bound and nested feedback channels]
\label{prop:info}
Let $\mathcal{S}_{\mathrm{val}}=(Z_1,\ldots,Z_n)\overset{\text{i.i.d.}}{\sim}\mathcal{Q}$, let $\ell(h,Z)\in\{0,1\}$ indicate a correct answer, and let $V(h)=\frac{100}{n}\sum_j\ell(h,Z_j)$ and $\mu(h)=100\,\mathbb{E}_{Z\sim\mathcal{Q}}\ell(h,Z)$. Let the proposer output $H$ from its feedback and randomness independent of $\mathcal{S}_{\mathrm{val}}$. Then
\begin{enumerate}[label=(\roman*),leftmargin=*,itemsep=1pt,topsep=2pt]
    \item $\bigl|\mathbb{E}[V(H)-\mu(H)]\bigr|\le 100\sqrt{I(\mathcal{S}_{\mathrm{val}};H)/(2n)}$;
    \item a harness chosen without access to $\mathcal{S}_{\mathrm{val}}$ has $I(\mathcal{S}_{\mathrm{val}};H)=0$ and no expected gap;
    \item if the observation maps satisfy $\psi_{\mathrm{blind}}=g\circ\psi_{\mathrm{trace}}$ and $\psi_{\mathrm{trace}}=g'\circ\psi_{\mathrm{mm}}$ for deterministic $g,g'$, then
    \[
        I(\mathcal{S}_{\mathrm{val}};\psi_{\mathrm{blind}})\le I(\mathcal{S}_{\mathrm{val}};\psi_{\mathrm{trace}})\le I(\mathcal{S}_{\mathrm{val}};\psi_{\mathrm{mm}}),
    \]
    and any leakage $I(\mathcal{S}_{\mathrm{val}};H)$ achievable with Blindfolded feedback is achievable with Trace feedback, and likewise from Trace to Multimodal.
\end{enumerate}
\end{proposition}

\noindent\textit{Proof.}
(i) Since $\ell\in[0,1]$, the empirical mean over $n$ i.i.d.\ questions is $\tfrac{1}{2\sqrt n}$-sub-Gaussian for every fixed $h$. The mutual-information generalization bound~\citep{russo2016controlling,xu2017information} then gives
\[
    \Bigl|\mathbb{E}\Bigl[\tfrac{V(H)-\mu(H)}{100}\Bigr]\Bigr|\le\sqrt{2\cdot\tfrac{1}{4n}\,I(\mathcal{S}_{\mathrm{val}};H)}.
\]
(ii) With $I=0$ the bound vanishes.

\noindent (iii) The chain follows from the data-processing inequality along $\mathcal{S}_{\mathrm{val}}\to\psi_{\mathrm{mm}}\to\psi_{\mathrm{trace}}\to\psi_{\mathrm{blind}}$. A Trace proposer can apply $g$ to each observation and run any Blindfolded strategy exactly, so every joint law of $(\mathcal{S}_{\mathrm{val}},H)$ achievable under Blindfolded feedback is achievable under Trace feedback, and likewise from Trace to Multimodal.\hfill$\square$

\paragraph{What the bound implies.}
The bound limits how much adaptive fit a channel permits, not how a proposer uses it. Inverting (i), $I(\mathcal{S}_{\mathrm{val}};H)\ge 2n\,(\text{gap}/100)^2$ nats: plugging in the winners' mean excess gaps, they must carry at least $1.9$ (Blindfolded), $3.3$ (Trace), and $4.1$ nats (Multimodal) about their validation sample, a lower bound that grows with feedback.

\subsection{Measuring Transfer and Instance-Specific Fit}
\label{app:overfit_beta}

\paragraph{The three shares.}
For the evolved harnesses $H_i$, $i=1,\dots,N$, of a run, let $V(H_i)$ and $T(H_i)$ be the validation and held-out accuracy, $V^A(H_i)$ and $V^B(H_i)$ the accuracy on two random halves $A,B$ of the validation questions, and $\operatorname{Var}_i$, $\operatorname{Cov}_i$ the variance and covariance over the $N$ harnesses. The shares of Figure~\ref{fig:diagnosis}(d) are
\begin{equation}
    \rho_\Delta=\frac{\operatorname{Cov}_i\!\big(V(H_i),T(H_i)\big)}{\operatorname{Var}_i V(H_i)},\qquad
    \rho_\varepsilon=\frac{\operatorname{Var}_i\!\big(V^A(H_i)-V^B(H_i)\big)}{4\operatorname{Var}_i V(H_i)},\qquad
    \rho_S=1-\rho_\Delta-\rho_\varepsilon.
    \label{eq:shares}
\end{equation}
Across the harnesses of a run, validation scores differ by $\Delta+S+\varepsilon$ and held-out scores by $\Delta+\varepsilon'$ with independent noise $\varepsilon'$. Since $\Delta$ is the only term the two share, $\rho_\Delta=\operatorname{Var}\Delta/\operatorname{Var}V$, which is also the slope of held-out on validation accuracy: a harness one point higher on validation is on average $\rho_\Delta$ points higher on held-out data. A harness's two half-scores share its $\Delta$ and $S$ and differ only by noise; each half carries twice the noise variance of the full set, so $\operatorname{Var}(V^A-V^B)$ is four times the noise variance of a full validation score, and $\rho_\varepsilon$ needs no held-out data. The validation-fit share $\rho_S$ is the remainder. The paragraphs below give the estimators pooled over runs; $\rho_\Delta$ is the transfer coefficient $\beta$.

\paragraph{Transfer coefficient.}
With run means $\bar V_r$, $\bar T_r$ over tested candidates, we estimate $\beta$ within runs, which removes run-level offsets from the seed, starting harness, and split:
\begin{equation*}
    \hat\beta=\frac{\sum_r\sum_{h\in r}(V_h-\bar V_r)(T_h-\bar T_r)}{\sum_r\sum_{h\in r}(V_h-\bar V_r)^2}.
\end{equation*}
Resampling runs (10{,}000 replicates) gives
\[
    \hat\beta_{\mathrm{blind}}=0.54\ [0.33,0.74],\qquad
    \hat\beta_{\mathrm{trace}}=0.35\ [0.24,0.44],\qquad
    \hat\beta_{\mathrm{mm}}=0.30\ [0.15,0.41],
\]
with $\Pr(\beta_{\mathrm{blind}}>\beta_{\mathrm{mm}})=0.98$ and a one-sided permutation $p=0.010$ for $\beta_{\mathrm{blind}}-\beta_{\mathrm{mm}}$. \textbf{Half or less of a validation difference survives on held-out data, and less so with richer feedback.}

\paragraph{Variance decomposition.}
Split-half resampling isolates noise without held-out data: for 200 random halves $A,B$ of the validation questions, $\tfrac14\operatorname{Var}(V^A-V^B)$ estimates $\sigma^2_e$, unaffected by selection on $V$. With $\operatorname{Var}V$ computed within runs over all complete proposals (excluding, of 2{,}319, the 29 catastrophic ones more than 10\,pp below the best baseline and one with a partial evaluation), the shares of Figure~\ref{fig:diagnosis}(d) are
\begin{equation*}
    \hat\rho_\Delta=\hat\beta,\qquad
    \hat\rho_S=1-\hat\beta-\hat\sigma^2_e/\widehat{\operatorname{Var}}V,\qquad
    \hat\rho_\varepsilon=\hat\sigma^2_e/\widehat{\operatorname{Var}}V.
\end{equation*}
Table~\ref{tab:overfit_decomp} shows that the noise share is the same in every arm ($\hat\sigma^2_e\approx1.0$\,pp$^2$), so \textbf{the validation score is equally reliable, but its instance-specific share grows from $0.31$ to $0.56$}; it exceeds Blindfolded's in 98\% (Multimodal) and 96\% (Trace) of bootstrap replicates.

\begin{table}[h]
\centering
\small
\caption{\textbf{Reliability versus validity of the validation score.} Shares of across-candidate validation variance, with 95\% run-cluster bootstrap CIs; reliability is one minus the noise share.}
\label{tab:overfit_decomp}
\begin{tabular}{@{}lccccc@{}}
\toprule
Arm & Candidates & Transfers ($\beta$) & Instance-specific & Noise & Reliability \\
\midrule
Blindfolded & 752 & $0.54$ $[0.33, 0.74]$ & $0.31$ $[0.13, 0.50]$ & $0.16$ $[0.14, 0.17]$ & $0.84$ \\
Trace & 774 & $0.35$ $[0.24, 0.44]$ & $0.51$ $[0.42, 0.62]$ & $0.14$ $[0.11, 0.18]$ & $0.86$ \\
Multimodal & 763 & $0.30$ $[0.15, 0.41]$ & $0.56$ $[0.44, 0.70]$ & $0.14$ $[0.11, 0.19]$ & $0.86$ \\
\bottomrule
\end{tabular}
\end{table}

\paragraph{Predicting the winners.}
Proposition~\ref{prop:goodhart}(ii) predicts each winner as $\hat T(H^\star)=\bar T_r+\hat\beta_{-r}\,(V(H^\star)-\bar V_r)$, with $\hat\beta_{-r}$ fit on the arm's other runs. The predictions ($67.0$, $66.0$, $66.0\%$) match the observed $66.5$, $66.3$, and $65.7\%$ with mean absolute error $0.34$\,pp, whereas full transfer ($\beta=1$) overestimates by $1.5$--$2.2$\,pp (Figure~\ref{fig:overfit_transfer}a).

\begin{figure*}[t]
    \centering
    \includegraphics[width=0.99\textwidth]{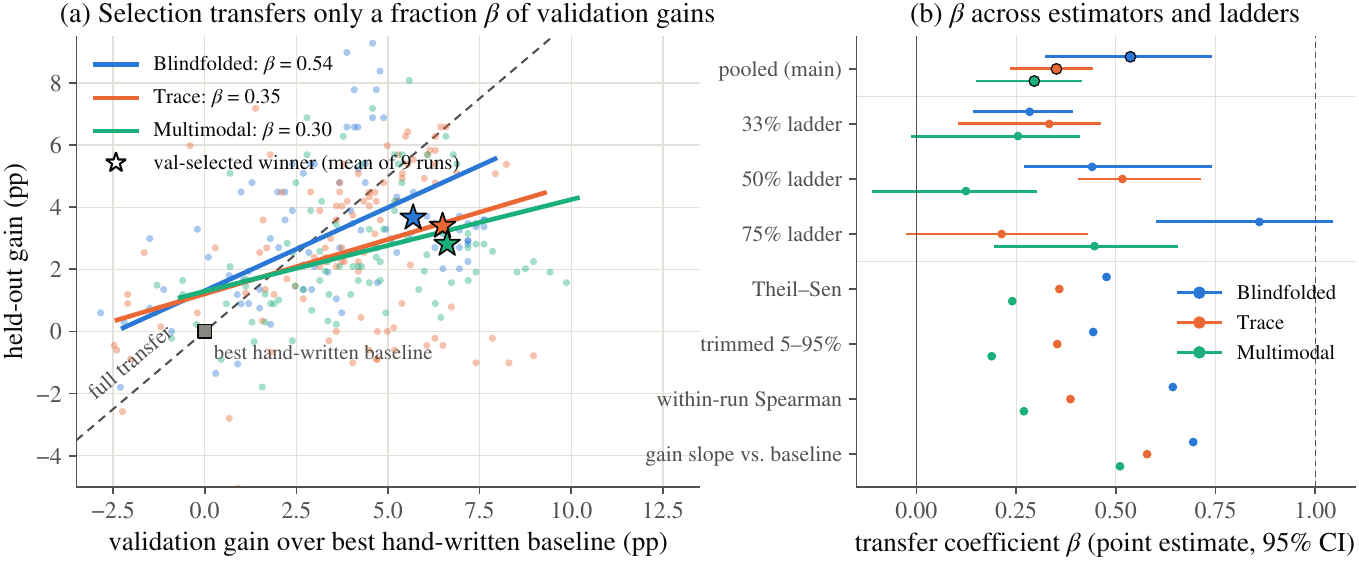}
    \caption{\textbf{Transfer of validation gains.} (a) Held-out versus validation gain over the run's best hand-written baseline for every tested candidate. Lines have slope $\hat\beta$ through the arm's mean; stars mark the mean winners. Selection moves right along a line, so a flatter line converts the same validation advantage into less held-out gain. (b) $\beta$ under alternative estimators and per ladder, with 95\% CIs.}
    \label{fig:overfit_transfer}
\end{figure*}

\subsection{Direct Evidence of the Content Channel}
\label{app:overfit_channel}

Validation and held-out questions are a random split of one pool, so text chosen without regard to the split matches validation questions at the base rate $\pi=|\text{validation}|/|\text{pool}|$. Using the 18 runs of the 33\% and 50\% ladders, for which question text of both splits is available within the 75\% validation set ($\pi=0.445$ and $0.666$), we collect every word 5-gram in evolved code that occurs in at most two question stems and test its matches against $\pi$ by re-drawing the split 20{,}000 times. The leak guard (Appendix~\ref{app:basemethod_loop}) rules out lookup tables by construction.

Table~\ref{tab:overfit_quotes} shows that \textbf{proposers that see question text quote validation questions beyond chance, and Blindfolded proposers do not.} The quotations sit almost entirely in docstrings and comments that record the validation cases motivating a design (e.g., ``a white bowl containing potatoes''): of the 168 matched Trace questions, 144 are quoted in docstrings, 25 in comments, and only 4 in functional literals (a question can be quoted in several places), where vocabulary is at most weakly enriched (ratios $0.94$--$1.08$). Content thus enters through design decisions rather than copied keywords, producing the behavioral fit measured above.

\begin{table}[h]
\centering
\small
\caption{\textbf{Verbatim validation-question 5-grams in evolved code} (33\% and 50\% ladders, six runs per arm). $p$: one-sided permutation test over the split.}
\label{tab:overfit_quotes}
\begin{tabular}{@{}lccccc@{}}
\toprule
Arm & Matched questions & Validation & Expected & Ratio & $p$ \\
\midrule
Blindfolded & 50 & 26 & 27.1 & 0.96 & 0.68 \\
Trace & 233 & 168 & 108.3 & 1.55 & $<10^{-4}$ \\
Multimodal & 69 & 49 & 38.7 & 1.27 & 0.007 \\
\bottomrule
\end{tabular}
\end{table}

\subsection{Robustness}
\label{app:overfit_limits}

The pooled ordering Blindfolded $>$ Trace $>$ Multimodal of $\beta$ holds under Theil--Sen slopes, trimming, median per-run slopes, within-run Spearman correlations, and gain slopes relative to the best baseline (Figure~\ref{fig:overfit_transfer}b). Blindfolded exceeds Multimodal in every leave-one-ladder-out fit, whereas the adjacent orderings can reverse, so the contrast between the two ends of the information range is the robust one. Moving the catastrophic-proposal threshold to 5 or 15\,pp changes the variance shares by at most $0.03$, and even retaining all catastrophic proposals preserves the ordering of instance-specific shares ($0.42$, $0.57$, $0.65$).

\section{Search-Budget Dilution}
\label{app:search_space_analysis}

This section supports Section~\ref{sec:search_space_efficiency}. Its conclusions are:
\begin{itemize}[leftmargin=*,itemsep=1pt,topsep=2pt]
    \item \textbf{Larger spaces do not help.} Deployed champions reach $66.8$, $65.8$, $66.0$, and $65.7\%$ held-out accuracy with one, two, three, and six mechanism families, and 94--100\% of promotions in every space are statistically unresolved (Appendix~\ref{app:dilution_results}).
    \item \textbf{Theory.} The power to detect an advantage $\Delta$ is bounded by the evidence $\lambda=r\Delta^2/\bar\sigma^2$, which a larger space erodes through breadth (fewer evaluations $r$ per mechanism combination) and coupling (larger footprints, hence larger $\sigma^2$) (Appendix~\ref{app:dilution_theory}).
    \item \textbf{Measurement.} From one to six families, $r$ falls $2.05\times$ and the verdict variance rises $2.24\times$, so the power for a 1-point gain drops from $0.45$ to $0.14$. Full-harness edits flip a median of 50 answers against 10, and since only $37\%$ of flips are fixes, larger edits are also more harmful (Appendix~\ref{app:dilution_footprint}).
\end{itemize}

\subsection{Protocol}
\label{app:dilution_protocol}

All runs use Video-MME (450 validation / 2{,}250 held-out questions) with the same executor, proposer, and loop; only the search space and the iteration budget differ (Table~\ref{tab:dilution_ladder}), and evidence statistics and validation gains are compared at a matched budget of 87 candidates; deployed accuracies are those at the end of each run. The nested spaces charter retrieval only (one family, 8 primitives); retrieval with frame selection (two families, 8 primitives), where a candidate may change either or both; perception tools (three families, 12 primitives); and the full harness (six families, 33 primitives), which adds reflection, role-based collaboration, and working memory. The main text contrasts the one- and six-family spaces; the two- and three-family spaces are reported only here. One retrieval-only run and one retrieval with frame selection run were excluded before analysis because their proposer had accessed files outside its workspace, an information leak that gave it information the protocol withholds.

\begin{table}[t]
\centering
\small
\caption{\textbf{The search-space ladder on Video-MME.} $K$: chartered mechanism families; $P$: callable primitives. The deployed champion is the validation-selected harness at the end of each run.}
\label{tab:dilution_ladder}
\begin{tabular}{@{}lccccc@{}}
\toprule
Space & $K$ & $P$ & Runs & Iterations $\times$ candidates & Champion found at iteration \\
\midrule
Retrieval only & 1 & 8 & 3 & $50\times3$ & 21, 46, 50 \\
+ frame selection & 2 & 8 & 3 & $50\times3$ & 28, 45, 39 \\
+ perception tools & 3 & 12 & 2 & $50\times3$;\ $100\times5$ & 49, 99 \\
Full harness & 6 & 33 & 3 & $30\times3$ & 28, 29, 21 \\
\bottomrule
\end{tabular}
\end{table}

A candidate's \emph{mechanism combination} is the set of mechanism tags its implementation exercises (retrieval, re-ranking, fusion, query rewriting, frame selection, reflection, prompting, pixel manipulation, perception tools, memory, routing, and frame budget; these are finer than the families of the space), assigned by fixed, outcome-blind keyword rules shared with Appendix~\ref{app:early_experience_trap}. Each candidate is compared with the champion in force when it was proposed, using $\hat\delta$, $d$, and $\sigma=100\sqrt{d}/n$ from Eqs.~(\ref{eq:app_verdict_counts})--(\ref{eq:app_paired_var}). Evidence statistics use the first 87 candidates of every run, the largest budget all eleven runs reach, because evaluations per mechanism combination are comparable only at a common budget.

\subsection{Results}
\label{app:dilution_results}

\begin{figure*}[t]
    \centering
    \includegraphics[width=0.96\textwidth]{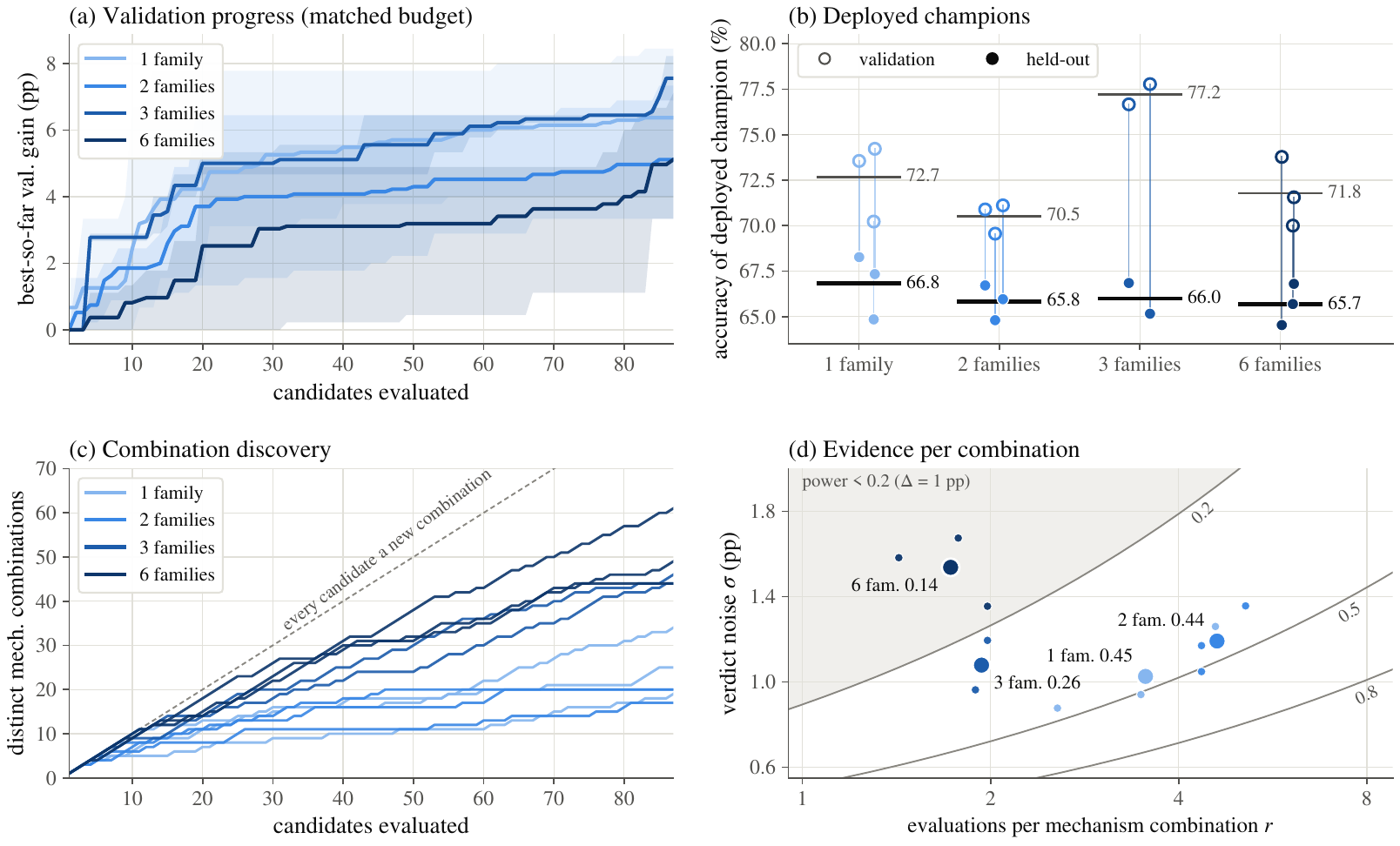}
    \caption{\textbf{All four spaces.} (a)~Best-so-far validation gain against candidates evaluated (mean and min--max over runs), on a budget-matched axis. (b)~Validation (hollow) and held-out (filled) accuracy of each run's deployed champion; bars are means. (c)~Distinct mechanism combinations visited against candidates evaluated; the dashed line is the limit in which every candidate is a new mechanism combination. (d)~Evidence landscape of Figure~\ref{fig:diagnosis}(e) with the three-family space, which has the breadth of the full harness but the footprint of retrieval-only evolution.}
    \label{fig:app_dilution_spaces}
\end{figure*}

\begin{table*}[t]
\centering
\small
\caption{\textbf{Outcome and evidence by space.} Validation gain is the best-so-far gain over the run's best seed harness after 87 candidates; accuracies are those of the deployed champions. Evidence columns use the first 87 candidates of each run: distinct mechanism combinations, evaluations per mechanism combination $r$, mean footprint $d$, mean verdict noise $\sigma$ (pp), mean verdict $\hat\delta$ (pp), and the noncentrality $\lambda$ and power of Proposition~\ref{prop:evidence} for an advantage $\Delta$. All entries are means over runs.}
\label{tab:dilution_summary}
\setlength{\tabcolsep}{4pt}
\begin{tabular}{@{}lccccccccccc@{}}
\toprule
& \multicolumn{3}{c}{Outcome (\%)} & \multicolumn{5}{c}{Search behavior} & \multicolumn{3}{c}{Power for $\Delta$ (pp)} \\
\cmidrule(lr){2-4}\cmidrule(lr){5-9}\cmidrule(l){10-12}
Space & Val.\ gain & Val. & Held-out & Combos & $r$ & $d$ & $\sigma$ & $\hat\delta$ & 0.5 & 1 & 2 \\
\midrule
Retrieval only (1) & 6.4 & 72.7 & 66.8 & 26.0 & 3.54 & 25.3 & 1.03 & $-1.10$ & 0.15 & 0.45 & 0.95 \\
+ frame selection (2) & 5.1 & 70.5 & 65.8 & 19.0 & 4.61 & 31.8 & 1.19 & $-1.42$ & 0.15 & 0.44 & 0.95 \\
+ perception tools (3) & 7.6 & 77.2 & 66.0 & 45.0 & 1.93 & 23.6 & 1.08 & $-0.76$ & 0.10 & 0.26 & 0.73 \\
Full harness (6) & 5.1 & 71.8 & 65.7 & 51.3 & 1.73 & 51.3 & 1.54 & $-3.14$ & 0.07 & 0.14 & 0.41 \\
\bottomrule
\end{tabular}
\end{table*}

\paragraph{No space improves held-out accuracy.}
Table~\ref{tab:dilution_summary} and Figure~\ref{fig:app_dilution_spaces}(a,b) extend the main-text comparison with a three-family space. That space reaches the largest validation gain ($7.6$ points after 87 candidates, $13.9$ at the end), yet its champions beat the seed harness by only $3.5$ points on held-out data. \textbf{Low evidence can thus appear as validation inflation rather than slow progress}: a third of its early proposals are near-clones differing from the champion on at most 1\% of questions, and the largest of many small, noisy verdicts is promoted. Larger spaces are also costlier: held-out gain per evaluation hour falls from $0.79$ (one family) to $0.17$ points (six), so counting candidates rather than time favors them.

\paragraph{Promotions are unresolved everywhere.}
Within the first 87 candidates, 94--100\% of promotions lie within $\pm1.96$ paired standard errors, with mean size 0.5--0.9 points. What differs across spaces is not whether promotions are resolved, but how much evidence the search could have gathered per mechanism combination. Figure~\ref{fig:app_dilution_spaces}(c) shows the breadth channel directly: over candidates 44--87, one- and two-family runs discover $0.26$ and $0.09$ new mechanism combinations per candidate, whereas three- and six-family runs keep discovering $0.44$ and $0.46$, so $r$ stays below two.

\subsection{Theory: Evidence per Mechanism Combination}
\label{app:dilution_theory}

\begin{proposition}[Evidence per mechanism combination]
\label{prop:evidence}
Let mechanism combination $c$ be evaluated $r_c$ times, with paired verdicts $\hat\delta_{c,1},\ldots,\hat\delta_{c,r_c}$ against the incumbent. Assume that (i) all implementations of $c$ share an advantage $\Delta_c$, and (ii) the verdicts are independent Gaussian with mean $\Delta_c$ and known variances $\sigma^2_{c,j}=10^4 d_{c,j}/n^2$. Define
\begin{equation}
    \lambda_c=\Delta^2\sum_{j=1}^{r_c}\sigma_{c,j}^{-2}=\frac{r_c\Delta^2}{\bar\sigma_c^2},
    \qquad
    \bar\sigma_c^{-2}=\frac{1}{r_c}\sum_{j=1}^{r_c}\sigma_{c,j}^{-2}.
    \label{eq:dilution_lambda}
\end{equation}
Then every level-$\eta$ test of the null $\Delta_c=0$ against the alternative $\Delta_c=\Delta>0$ has power at most
\[
    \Phi\bigl(\sqrt{\lambda_c}-z_{1-\eta}\bigr),
\]
which is $\Phi(\sqrt{\lambda_c}-1.96)$ for $\eta=0.025$. If a budget of $B$ evaluations visits $M(B)$ distinct mechanism combinations, the mean evidence per mechanism combination is
\begin{equation}
    \bar\lambda=\underbrace{\frac{B}{M(B)}}_{\text{breadth}}\;\Delta^2\;\underbrace{\overline{\sigma^{-2}}}_{\text{coupling}},
    \label{eq:dilution_budget}
\end{equation}
where $\overline{\sigma^{-2}}$ is the mean inverse verdict variance over the $B$ evaluations.
\end{proposition}

\noindent\textit{Proof.}
Under (i)--(ii), the log-likelihood ratio of the alternative to the null is increasing in the precision-weighted statistic
\[
    Z=\frac{\sum_j\hat\delta_{c,j}/\sigma_{c,j}^2}{\bigl(\sum_j\sigma_{c,j}^{-2}\bigr)^{1/2}},
    \qquad
    Z\sim\mathcal{N}(0,1)\ \text{under the null},
    \qquad
    Z\sim\mathcal{N}\bigl(\sqrt{\lambda_c},1\bigr)\ \text{under the alternative}.
\]
By the Neyman--Pearson lemma, the most powerful level-$\eta$ test rejects when $Z>z_{1-\eta}$, with power $\Pr(\mathcal{N}(\sqrt{\lambda_c},1)>z_{1-\eta})=\Phi(\sqrt{\lambda_c}-z_{1-\eta})$.

\noindent For Eq.~(\ref{eq:dilution_budget}), sum over mechanism combinations and divide by $M(B)$:
\[
    \sum_c\lambda_c=\Delta^2\sum_{j=1}^{B}\sigma_j^{-2}
    \quad\Longrightarrow\quad
    \bar\lambda=\frac{1}{M(B)}\sum_c\lambda_c=\frac{B}{M(B)}\,\Delta^2\,\overline{\sigma^{-2}}.
    \tag*{$\square$}
\]

\paragraph{Remarks.}
The bound is optimistic, so the reported powers are upper bounds: heterogeneous implementations add variance $\tau^2$ to each verdict, verdicts sharing questions and incumbents are positively dependent, and the loop acts on single comparisons ($r_c=1$) rather than pooling. The proposition is the per-combination analogue of fixed-budget best-arm identification, where error grows with the number of arms sharing a budget~\citep{audibert2010best}. Breadth also compounds: if proposals draw mechanism combinations from $p$ over $\mathcal{C}$, the expected number of distinct mechanism combinations after $N$ draws is $\sum_{c}[1-(1-p_c)^N]$, which approaches $N$ (so $r\to1$) when the proposer spreads over many more mechanism combinations than the budget can revisit, as happens once mechanism combinations span many families.

\paragraph{Estimates.}
For each run we set $r=87/M(87)$ and take $\bar\sigma$ as the mean paired standard error over its first 87 candidates, with $\Delta=1$ point (above the mean promotion size and the $0.65$-point held-out value of doubling the frame budget; Appendix~\ref{app:early_experience_trap}). From one to six families,
\[
    r:\ 3.54\to1.73\ (2.05\times\text{ fewer}),\qquad
    \bar\sigma^2:\ 2.24\times\text{ larger},\qquad
    \lambda:\ 3.38\to0.76,\qquad
    \text{power}:\ 0.45\to0.14,
\]
with $0.44$ and $0.26$ for two and three families. \textbf{Power decreases with the number of families} (Spearman $\rho=-0.74$, $p=0.02$ over the nine nested runs; $\rho=-0.86$, $p<0.001$ over all eleven), and the three full-harness runs have the three lowest values (exact $p=1/84$). The ordering is unchanged when $\sigma$ is summarized by its median, by the noise at the mean footprint, or by precision-weighting only verdicts with $d\geq4.5$. The argument does not explain the slower validation progress of the two-family space, which is within seed variation.

\subsection{Behavioral Footprint of Edits}
\label{app:dilution_footprint}

\begin{figure*}[t]
    \centering
    \includegraphics[width=0.96\textwidth]{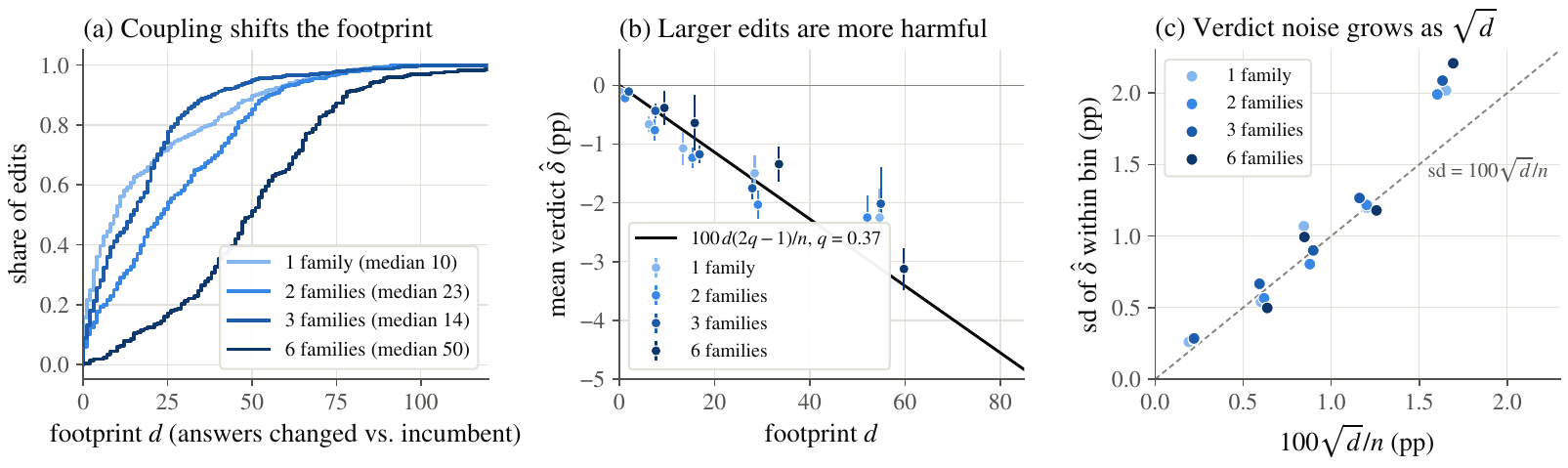}
    \caption{\textbf{The behavioral footprint of edits.} (a)~Distribution of the footprint $d$, the number of validation answers whose correctness an edit flips relative to its incumbent. (b)~Mean verdict by footprint bin (95\% intervals); the line is Proposition~\ref{prop:footprint} with the pooled fix probability $q=0.37$. Bins above 80 are dominated by crashes and omitted. (c)~Standard deviation of verdicts within each bin against the paired noise scale $100\sqrt{d}/n$.}
    \label{fig:app_dilution_footprint}
\end{figure*}

\begin{proposition}[Footprint law]
\label{prop:footprint}
Suppose an edit flips $d$ validation answers, each independently a fix with probability $q$ and a break otherwise. With $F\sim\mathrm{Binomial}(d,q)$ fixes, the verdict $\hat\delta=100\,(2F-d)/n$ satisfies
\begin{equation}
    \mathbb{E}[\hat\delta]=\frac{100\,d\,(2q-1)}{n},
    \qquad
    \operatorname{Var}[\hat\delta]=\frac{4\cdot10^4\,d\,q(1-q)}{n^2}\leq\frac{10^4\,d}{n^2}.
    \label{eq:footprint_law}
\end{equation}
Hence, when $q<1/2$, the expected effect decreases linearly in $d$ while the noise grows as $\sqrt{d}$.
\end{proposition}

\noindent\textit{Proof.}
Substitute $\mathbb{E}[F]=dq$ and $\operatorname{Var}[F]=dq(1-q)$ into $\hat\delta=100(2F-d)/n$; the inequality uses $q(1-q)\leq1/4$.\hfill$\square$

\paragraph{Coupling enlarges the footprint.}
The median footprint is 50 answers in the full harness against 10, 23, and 14 for one, two, and three families (Figure~\ref{fig:app_dilution_footprint}a). This is not because full-harness proposals declare more changes: edits touching at most one family still flip $47.6$ answers on average, against $19.8$ with retrieval alone. \textbf{In a coupled harness a local change alters what every downstream stage receives}---retrieved frames, tool outputs, reflection prompts, and memory---so its effect is not local.

\paragraph{Larger footprints are more harmful and noisier.}
The fix probability is nearly constant across spaces ($q=0.36$--$0.39$; pooled $0.37$ $[0.36,0.39]$). Because $q<1/2$ near a strong incumbent, the mean verdict falls roughly linearly in $d$, as Proposition~\ref{prop:footprint} predicts (Figure~\ref{fig:app_dilution_footprint}b; $R^2=0.39$), and full-harness edits are the most harmful ($-3.13$ points over all candidates, against $-0.90$, $-1.53$, and $-1.02$; Table~\ref{tab:dilution_summary} uses the first 87 candidates of each run). Within-bin dispersion follows $100\sqrt{d}/n$ up to $d=40$ and exceeds it beyond (panel c). This mirrors pleiotropy in Fisher's geometric model, where mutations affecting more traits are less likely to help~\citep{fisher1930genetical,orr2000adaptation}; here, coupling sets the pleiotropic reach.

\section{The Early-Experience Trap}
\label{app:early_experience_trap}

This section supports Section~\ref{sec:early_experience_trap}. Its conclusions are:
\begin{itemize}[leftmargin=*,itemsep=1pt,topsep=2pt]
    \item \textbf{Verdicts are weak and biased.} On Video-MME only $18\%$ of verdict variance reflects true differences ($8\%$ in iterations 1--10), and the champion's cached score is inflated by $\widehat\omega\approx1.0$ point (Appendix~\ref{app:eet_verdict_model}).
    \item \textbf{The policy ignores this.} The held-out win rate changes gradually with the verdict, yet the loop continues winners with probability $0.77$ and non-winners with $0.076$, and treats false losers like true losers (Appendix~\ref{app:eet_verdicts}).
    \item \textbf{Theory.} Because rejected directions are rarely resampled, champion inflation ($\omega>0$) makes a weak beneficial direction more likely to be lost; for the frame budget the loss probability is $0.54$, consistent with the 6 of 12 runs that never expand it (Appendix~\ref{app:eet_model}).
    \item \textbf{Consequences.} Seeds of one setting diverge $2.8\times$ more than independent sampling predicts, and runs that stay at 16 frames lose $2.7$ held-out points (Appendices~\ref{app:eet_allocation} and~\ref{app:eet_traces}).
\end{itemize}

\subsection{Corpus and Definitions}
\label{app:eet_protocol}

We analyze 44 \basemethod{} runs (Table~\ref{tab:eet_settings}): 14 on Video-MME with 2{,}229 candidates and 30 on LongVideoBench with 2{,}577. Held-out evaluations, never shown to the proposer, exist for every candidate of three Video-MME runs and for frontier or best candidates elsewhere.

\begin{table}[t]
\centering
\small
\caption{\textbf{The 44 \basemethod{} runs analyzed in Section~\ref{sec:early_experience_trap}.} Video-MME runs search one or a few mechanism families for 50 iterations of three candidates (one run: 100 iterations of five); the 12 retrieval-family runs (all but self-reflection) form the frame-budget contrast of Figure~\ref{fig:diagnosis}(c). LongVideoBench runs search the full harness for 30 iterations under the feedback arms of Section~\ref{sec:overfitting}, three seeds at each of three validation fractions.}
\label{tab:eet_settings}
\setlength{\tabcolsep}{4pt}
\begin{tabular}{@{}llc@{}}
\toprule
Benchmark & Setting & Runs \\
\midrule
Video-MME & retrieval only (16-frame seeds; one from 32-frame AKS) & 6 \\
          & retrieval + frame selection & 4 \\
          & retrieval + perception tools (one with $100\times5$) & 2 \\
          & self-reflection & 2 \\
\midrule
LongVideoBench & Blindfolded / Trace / Multimodal feedback & $9+9+9$ \\
               & clade-UCB parent suggestions (33\% ladder) & 3 \\
\bottomrule
\end{tabular}
\end{table}

Each candidate $c$ is compared with the incumbent, the validation champion when $c$ is proposed, by its verdict $\hat\delta_c$ and $z_c=\hat\delta_c/\widehat{\mathrm{SE}}_c$ with $\widehat{\mathrm{SE}}_c=100\sqrt{d_c}/n$ (Appendix~\ref{app:verdict_model}); $\delta^{\mathrm{test}}_c$ is the held-out analogue. A candidate is \emph{continued} if a later candidate names it as parent, and is a \emph{false loser} if $\hat\delta_c\le0$ but $\delta^{\mathrm{test}}_c\ge0$. A retrieval-family run \emph{expands the frame budget} if its deployed champion uses more than $1.2\times$ the 16-frame baseline. A within-run regression of held-out accuracy on mechanism families, iteration, and $\log_2$ frames over 491 tested candidates values one doubling of the frame budget at $\widehat{\mu}_f=0.65\pm0.10$\,pp (observational).

\subsection{A Verdict Model with Incumbent Optimism}
\label{app:eet_verdict_model}

For a candidate--incumbent pair evaluated on both splits, Eq.~(\ref{eq:verdict}) specializes to
\begin{equation}
    \hat{\delta}=\Delta-\omega+u+\varepsilon,\qquad
    \delta^{\mathrm{test}}=\Delta+\varepsilon',\qquad
    \Delta\sim\mathcal{N}(m,\tau^2),
    \label{eq:eet_bivariate}
\end{equation}
with independent sampling errors $\varepsilon,\varepsilon'$ of variances $\sigma^2,\sigma_t^2$, non-transferable fit $u$ (the term $S$ of Eq.~(\ref{eq:verdict})) of variance $s^2$, and incumbent optimism $\omega\geq0$ from the optimizer's curse~\citep{smith2006optimizer}. We take $u$ zero-mean, so $\widehat{\omega}$ below estimates the incumbent's net optimism $\omega-\mathbb{E}\,S$, the total amount by which a verdict understates $\Delta$ on average; the two coincide when candidate and incumbent carry equal fit on average. The method-of-moments estimators are
\begin{equation}
\begin{aligned}
    \widehat{m}&=\overline{\delta^{\mathrm{test}}}, &
    \widehat{\omega}&=\overline{\delta^{\mathrm{test}}}-\overline{\hat{\delta}}, &
    \widehat{\tau}^2&=\widehat{\operatorname{Cov}}(\hat{\delta},\delta^{\mathrm{test}}),\\
    \widehat{\sigma}^2&=\overline{\widehat{\mathrm{SE}}^2}, &
    \widehat{s}^2&=\widehat{\operatorname{Var}}(\hat{\delta})-\widehat{\tau}^2-\widehat{\sigma}^2 .
\end{aligned}
\label{eq:eet_moments}
\end{equation}

\begin{proposition}[Evidential value of a verdict]
\label{prop:evidence_value}
Under Eq.~(\ref{eq:eet_bivariate}), let $\kappa=\tau^2/(\tau^2+\sigma^2+s^2)$ and $\mu(\hat{\delta})=m+\kappa\,(\hat{\delta}+\omega-m)$. Then
\begin{equation*}
    \Pr(\Delta>0\mid\hat{\delta})=\Phi\!\left(\frac{\mu(\hat{\delta})}{\tau\sqrt{1-\kappa}}\right),\qquad
    \Pr(\delta^{\mathrm{test}}>0\mid\hat{\delta})=\Phi\!\left(\frac{\mu(\hat{\delta})}{\sqrt{\tau^2(1-\kappa)+\sigma_t^2}}\right).
\end{equation*}
Both equal $1/2$ at $\hat{\delta}^{\ast}=m-\omega-m/\kappa$ ($=-\omega$ when $m=0$), and their slopes in $\hat{\delta}$ increase with the verdict's reliability $\kappa$ and vanish as $\kappa\to0$.
\end{proposition}

\noindent\textit{Proof.}
$(\Delta,\hat{\delta},\delta^{\mathrm{test}})$ is jointly Gaussian with
\[
    \operatorname{Cov}(\Delta,\hat{\delta})=\operatorname{Cov}(\delta^{\mathrm{test}},\hat{\delta})=\tau^2,
    \qquad
    \operatorname{Var}(\hat{\delta})=\tau^2+\sigma^2+s^2.
\]
Gaussian conditioning gives
\[
    \mathbb{E}[\Delta\mid\hat{\delta}]=\mathbb{E}[\delta^{\mathrm{test}}\mid\hat{\delta}]=\mu(\hat{\delta}),
    \qquad
    \operatorname{Var}(\Delta\mid\hat{\delta})=\tau^2(1-\kappa),
    \qquad
    \operatorname{Var}(\delta^{\mathrm{test}}\mid\hat{\delta})=\tau^2(1-\kappa)+\sigma_t^2,
\]
and setting $\mu(\hat{\delta})=0$ yields $\hat{\delta}^\ast$.\hfill$\square$

\paragraph{Implications.}
\textbf{A verdict of zero is not neutral}: with $\omega>0$, a challenger that ties the incumbent is more likely than not to be better, and it must lose by about $\omega$ before it becomes a coin flip. When $\kappa$ is small, the rational probability of continuing a line changes only gradually with $\hat\delta$.

\paragraph{Estimates.}
Table~\ref{tab:eet_verdict_model} reports run-bootstrap estimates after dropping clones and execution failures. On Video-MME the reliability is only $0.18$, and lowest early ($0.08$ versus $0.24$ after iteration~10) because early edits change more answers. The optimism $\widehat\omega\approx1.0$\,pp is stable over time and in the fully tested runs. LongVideoBench verdicts are more reliable ($0.48$) and less biased ($\widehat\omega=0.44$), but the policy is the same.

\begin{table}[t]
\centering
\small
\caption{\textbf{Verdict-model estimates} (Eq.~(\ref{eq:eet_moments}); pp). Reliability is $\tau^2/(\tau^2+\sigma^2)$, i.e., $\kappa$ of Proposition~\ref{prop:evidence_value} with $s^2=0$. Brackets: 95\% run-bootstrap intervals.}
\label{tab:eet_verdict_model}
\setlength{\tabcolsep}{4pt}
\begin{tabular}{@{}lrcccc@{}}
\toprule
Subset & Pairs & $\widehat{\tau}$ & $\widehat{\sigma}$ & $\widehat{\omega}$ & Reliability \\
\midrule
Video-MME, all & 387 & 0.49 [0.31, 0.58] & 1.04 [0.90, 1.23] & 1.02 [0.64, 1.39] & 0.18 [0.09, 0.24] \\
\quad iterations 1--10 & 86 & 0.45 [0.18, 0.66] & 1.49 [1.40, 1.59] & 1.01 [0.69, 1.32] & 0.08 \\
\quad iterations $>10$ & 301 & 0.48 & 0.87 & 1.02 & 0.24 \\
\quad fully tested runs & 296 & 0.46 & 1.00 & 1.11 & 0.18 \\
LongVideoBench, all & 158 & 1.01 [0.72, 1.22] & 1.06 [0.94, 1.16] & 0.44 [0.21, 0.70] & 0.48 [0.32, 0.57] \\
\bottomrule
\end{tabular}
\end{table}

\subsection{Continuation Policy versus Evidential Value}
\label{app:eet_verdicts}

Binning verdicts by $z$, Proposition~\ref{prop:evidence_value} tracks the held-out win rate (mean absolute deviation $0.035$ on Video-MME, $0.085$ on LongVideoBench), which crosses $1/2$ at $z\approx-0.75$ and $-0.25$ (Figure~\ref{fig:eet_theory}a). \textbf{The policy, by contrast, is nearly a step function at $z=0$}: on Video-MME a one-question loss is continued with probability $0.070$ and a one-question win with $0.685$; winners overall with $0.77$ and non-winners with $\rho=0.076$ (odds ratio 41; LongVideoBench 34). False losers are abandoned as often as true losers ($93\%$ versus $96\%$ on Video-MME), although $37\%$ of tested non-winners beat their incumbent on held-out data ($33\%$ on LongVideoBench).

\begin{figure*}[t]
    \centering
    \includegraphics[width=0.99\textwidth]{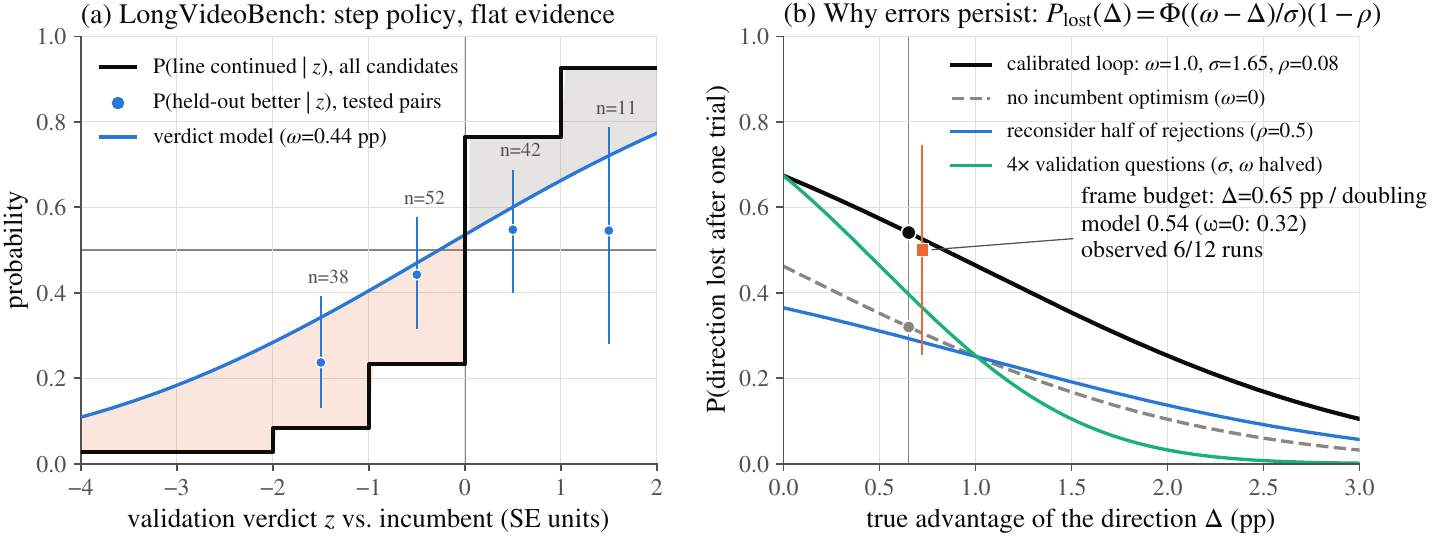}
    \caption{\textbf{Theory of the trap.} (a)~LongVideoBench counterpart of Figure~\ref{fig:diagnosis}(f): the continuation policy is a step at $z=0$, while the held-out win rate of the same verdicts changes gradually; the curve is Proposition~\ref{prop:evidence_value} with the fitted parameters. (b)~Probability that a beneficial direction is lost after one trial (Proposition~\ref{prop:hotstove}) at the Video-MME calibration, without incumbent optimism, with a policy that reconsiders half of all rejections, and with four times as many validation questions (which halves $\sigma$ and, being selection noise, $\omega$). Points mark the frame-budget direction; the square is the observed fraction of retrieval-family runs that never expand the budget (Wilson 95\% interval).}
    \label{fig:eet_theory}
\end{figure*}

\subsection{Why Early Errors Persist}
\label{app:eet_model}

\begin{proposition}[Persistence under endogenous sampling]
\label{prop:hotstove}
Consider a direction with true advantage $\Delta$ over the incumbent. Its first trial yields $\hat{\delta}_1=\Delta-\omega+\varepsilon_1$ with $\varepsilon_1\sim\mathcal{N}(0,\sigma^2)$. The loop continues the direction if $\hat{\delta}_1>0$; otherwise it reconsiders it with probability $\rho$, independently of $\Delta$, and otherwise never evaluates it again. Then:
\begin{enumerate}[label=(\alph*),leftmargin=*,itemsep=1pt,topsep=2pt]
    \item the direction is lost with probability
    \[
        P_{\mathrm{lost}}(\Delta)=\Phi\!\left(\frac{\omega-\Delta}{\sigma}\right)(1-\rho);
    \]
    \item $P_{\mathrm{lost}}$ increases with $\omega$, decreases with $\Delta$ and $\rho$, and increases with $\sigma$ if and only if $\Delta>\omega$;
    \item as $\Delta/\sigma\to0$ with $\omega/\sigma$ fixed, $P_{\mathrm{lost}}\to\Phi(\omega/\sigma)(1-\rho)>(1-\rho)/2$ whenever $\omega>0$, so champion inflation makes a weak but beneficial direction more likely to be lost, and it is more likely lost than kept once $\Phi(\omega/\sigma)(1-\rho)>1/2$;
    \item after a rejection the loop's only estimate is $\hat\delta_1$, which is biased downward,
    \[
        \mathbb{E}[\hat{\delta}_1\mid\hat{\delta}_1\leq0]=\Delta-\omega-\sigma\,\frac{\phi(a)}{\Phi(-a)}<\Delta,
        \qquad a=\frac{\Delta-\omega}{\sigma},
    \]
    and is never corrected, whereas continued directions receive new, independent evidence.
\end{enumerate}
\end{proposition}

\noindent\textit{Proof.}
(a) $\Pr(\hat{\delta}_1\leq0)=\Phi((\omega-\Delta)/\sigma)$, and a rejected direction is lost unless independently reconsidered, with probability $1-\rho$.

\noindent (b) Monotonicity in $\omega$, $\Delta$, and $\rho$ follows from that of $\Phi$; for $\sigma$,
\[
    \frac{\partial P_{\mathrm{lost}}}{\partial\sigma}=-\,\phi\!\left(\frac{\omega-\Delta}{\sigma}\right)\frac{\omega-\Delta}{\sigma^{2}}\,(1-\rho),
\]
which is positive if and only if $\Delta>\omega$.

\noindent (c) follows from (a), since $\Phi(x)>1/2$ for $x>0$. (d) is the mean of a Gaussian truncated above at zero.\hfill$\square$

\paragraph{Interpretation.}
Part (d) is the hot-stove asymmetry of learning from self-selected experience~\citep{denrell2001adaptation} and the incomplete learning of greedy allocation~\citep{rothschild1974two,brezzi2000incomplete}: negative errors persist because the alternative is never resampled. Persistent memory makes this worse, as labels such as ``MINED OUT'' set $\rho\approx0$ for an entire family.

\paragraph{Calibration for the frame budget.}
With $\Delta=\widehat{\mu}_f=0.65$\,pp, $\sigma=1.65$\,pp (median early paired scale), $\omega=1.01$\,pp (Table~\ref{tab:eet_verdict_model}), and $\rho=0.076$,
\[
    \Pr(\text{reject})=0.585,\qquad
    \mathbf{P_{\mathrm{lost}}=0.54}\ \ [0.39,\,0.63]\ \text{over plausible }(\omega,\sigma,\rho),\qquad
    P_{\mathrm{lost}}\big|_{\omega=0}=0.32 .
\]
Six of the 12 retrieval-family runs never expand the budget ($0.50$; Wilson $[0.25,0.75]$), consistent with the model, though too few to reject $\omega=0$. Figure~\ref{fig:eet_theory}(b) compares remedies: reconsidering half of all rejections lowers $P_{\mathrm{lost}}$ to $0.29$, quadrupling validation questions to $0.40$, and judging a probe against its own parent removes $\omega$ directly.

\subsection{Path Dependence across Seeds}
\label{app:eet_allocation}

If early noisy verdicts steer exploration, seeds of one setting should diverge more than independent samples would. Dividing the cross-seed standard deviation of each family's candidate share by its binomial value, \textbf{the median ratio is $2.8$ on Video-MME and $1.7$ on LongVideoBench}, the signature of reinforcement processes that amplify early outcomes~\citep{arthur1989competing}; re-ranking, for example, takes $93\%$ of one seed's candidates but $2$--$11\%$ in the other four. The held-out cost of staying at the baseline budget ($64.9$ versus $67.6\%$, $p=5\times10^{-5}$) holds under alternative definitions of ``stays'' and without the AKS-seeded run ($p\leq0.003$).

\subsection{Case Studies}
\label{app:eet_traces}

\paragraph{A weak loss becomes a prohibition, or not.}
In retrieval-only seed~1, a 32-frame probe loses by 10 questions at iteration~8 ($z=-1.41$); memory then records frame count as ``MINED OUT---do NOT re-propose'', and no later candidate exceeds 16 frames. In seed~4, a 33.9-frame probe loses by 12 questions ($z=-1.52$), but a sibling at 18.3 frames wins by 10 questions ($z=+1.44$); memory keeps the direction open, and the champion reaches 25--30 frames and $67.3\%$ held-out, against $64.8\%$ for seed~1. A Bayes-rational policy (Proposition~\ref{prop:evidence_value}) would still give seed~1's rejected direction probability $\approx0.29$.

\paragraph{Over-commitment to an early winner.}
In the same seed~1, after a re-ranking candidate wins by four questions at iteration~5 ($z=+0.76$), all 105 later candidates are re-ranking variants and $96\%$ fail to beat the champion (Figure~\ref{fig:eet_overcommit}a). Four promotions add $2.4$ validation points but only $0.04$ held-out points, and 23 of 115 tested candidates beat the deployed champion on held-out data (panel b).

\begin{figure*}[t]
    \centering
    \includegraphics[width=0.99\textwidth]{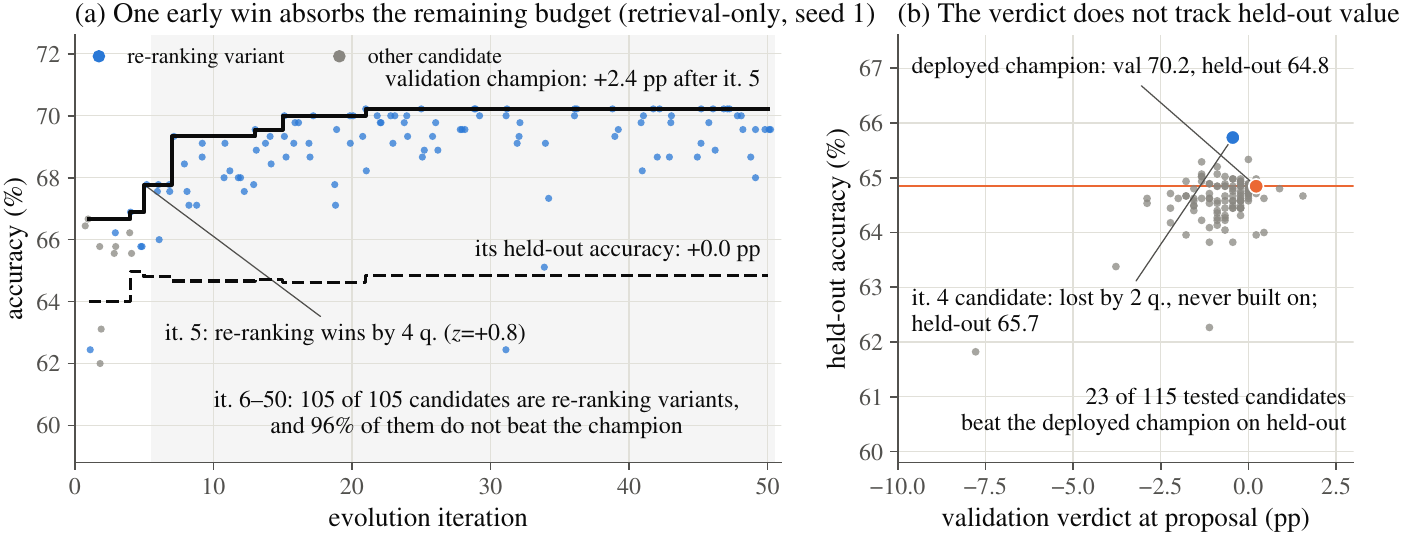}
    \caption{\textbf{Over-commitment to an early winner} (Video-MME, retrieval-only, seed~1; the run of Figure~\ref{fig:diagnosis}(c)). (a)~Validation accuracy of every candidate, the validation champion (solid), and the champion's held-out accuracy (dashed); iterations 6--50 are shaded. (b)~Held-out accuracy of the 115 tested candidates against their validation verdict at proposal. The horizontal line marks the deployed champion's held-out accuracy.}
    \label{fig:eet_overcommit}
\end{figure*}

\section{Why Simple Remedies Are Insufficient}
\label{sec:simple_remedies}

This section supports the discussion of simple remedies in Section~\ref{sec:meta_harness_failure_modes}. Its conclusions are:
\begin{itemize}[leftmargin=*,itemsep=1pt,topsep=2pt]
    \item \textbf{Parent selection.} Changing which harness the proposer builds on reshapes the search as intended, but does not improve held-out accuracy (Appendix~\ref{sec:search_algorithm}).
    \item \textbf{Stronger initialization.} Starting from AKS speeds up early validation progress, but the final harness is no better than those evolved from simple seeds (Appendix~\ref{sec:better_initialization}).
\end{itemize}
Neither remedy changes the verdicts on which the loop acts, so neither removes the failure modes.

\subsection{Parent Selection}
\label{sec:search_algorithm}

\paragraph{Setting.}
We run 18 full-harness \basemethod{} runs (1{,}556 candidates, 30 iterations of three) on Video-MME and LongVideoBench. Within a benchmark, arms differ only in the list of suggested parents shown to the proposer each iteration: none (\textbf{plain loop}); the current champion (\textbf{greedy}); six names ranked by an upper confidence bound over the lineage tree, the best score in a subtree plus an exploration bonus $\gamma/\sqrt{\text{subtree size}}$ with $\gamma=1.5$ (\textbf{clade-UCB}); or six names drawn uniformly from the archive (\textbf{random}). Video-MME runs all four arms and LongVideoBench compares the plain loop with clade-UCB, three seeds each. The Video-MME plain-loop runs are the full-harness runs of the search-space ladder.

\paragraph{Results.}
\textbf{No parent rule yields a consistent held-out advantage} (Figure~\ref{fig:search_algorithm}a). On Video-MME the deployed harnesses reach $65.7$, $66.9$, $67.2$, and $67.4\%$ for the plain loop, greedy, clade-UCB, and random parents; among the three rules every pairwise difference is at most $0.5$ points ($p\geq0.74$, Welch), and all seed ranges overlap. On LongVideoBench, clade-UCB trails the plain loop by $1.01\pm1.17$ points over seed-paired per-question comparisons ($p=0.48$). The rules do not separate on validation either (Figure~\ref{fig:search_algorithm}b): clade-UCB trails the plain loop on LongVideoBench at every matched budget.

\paragraph{The manipulation worked.}
The proposer builds on a listed parent in $94$--$100\%$ of clade-UCB, $74$--$86\%$ of greedy, and $80$--$86\%$ of random-parent proposals. On Video-MME the share of candidates bred from the champion falls from $47\%$ (greedy) to $30\%$ (clade-UCB) and $8\%$ (random), while distinct parents rise from 19 to 32 and 47; on LongVideoBench clade-UCB halves the spread of validation scores (3.6 versus 6.6 points). Yet \textbf{no search-shape statistic predicts held-out gain} across the 18 runs (top-parent share $\rho=-0.10$, $p=0.69$; distinct parents $\rho=-0.01$, $p=0.98$; Figure~\ref{fig:search_algorithm}c). What does is how quickly validation rose (candidates to $90\%$ of the final gain, $\rho=-0.53$, $p=0.025$): the quality of the harnesses the proposer happened to write, not the rule that chose their parents.

\paragraph{Interpretation.}
A parent rule decides which lineage the next verdict concerns, not how informative that verdict is. Every arm still compares single candidates against a selected champion on the same validation questions, so $S$, $\omega$, and $\varepsilon$ in Eq.~(\ref{eq:verdict}) are unchanged and the evidence per mechanism combination cannot rise. This motivates organizing evolution around mechanisms, where evidence can be concentrated, rather than lineages. Since no rule differs significantly from the others in this matched comparison, the loops inside \method{}, such as the warm-up, keep the default of leaving the parent choice to the proposer, and we do not repeat the parent-selection comparison for \method{}.

\begin{figure*}[t]
    \centering
    \includegraphics[width=\textwidth]{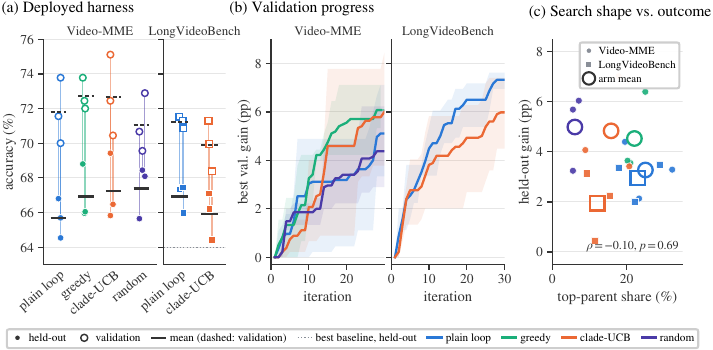}
    \caption{\textbf{Parent selection.} (a)~Validation (hollow) and held-out (filled) accuracy of each run's deployed harness; bars mark arm means. (b)~Best-so-far validation gain over the best seeded harness (mean and seed range). (c)~Share of candidates bred from the most-used parent against held-out gain over the best baseline; large outlined markers are arm means.}
    \label{fig:search_algorithm}
\end{figure*}

\subsection{Stronger Initialization}
\label{sec:better_initialization}

\paragraph{Setting.}
We run seven Video-MME runs (924 candidates, 50 iterations of three). Six start from simple 16-frame seeds (three evolve retrieval alone, three retrieval with frame selection); the seventh evolves the retrieval-only space from an Adaptive Keyframe Sampling (AKS)~\citep{tang2025adaptive} seed harness, which allocates 32 frames by CLIP relevance (distinct from the fixed AKS baselines of Figure~\ref{fig:initial_metaharness}). AKS starts at $68.7\%$ validation and $67.1\%$ held-out accuracy, $3.4$ and $4.0$ points above the simple seeds.

\paragraph{Results.}
\textbf{A stronger start accelerates early progress but not the endpoint} (Figure~\ref{fig:better_initialization}). The AKS run reaches $70\%$ validation accuracy at iteration~2 and $72\%$ at iteration~7, against 4--15 and 21--45 iterations for retrieval-only runs from simple seeds, which pass the AKS starting level within 4--18 iterations. Its deployed harness ($74.0\%$ validation, $68.2\%$ held-out) lies inside the range of the six other runs ($69.6$--$74.2\%$ and $64.8$--$68.3\%$) and is statistically tied with the three best (paired differences $-0.1$, $+0.8$, $+1.5$ points; $p\geq0.06$). Evolution adds only $1.1$ points to AKS ($p=0.13$), against $1.4$--$5.8$ points for simple seeds, four of six significantly, while using 36 frames against 14--25.

\paragraph{Interpretation.}
A stronger start mainly shortens the search to competitive performance; with one AKS run and a larger initial frame budget, it does not establish a reproducible final gain. It does suggest a way around the early-experience trap: supplying a useful mechanism before chance validation wins shape exploration, which motivates the instrumented warm-up of \method{}.

\begin{figure*}[t]
    \centering
    \includegraphics[width=\textwidth]{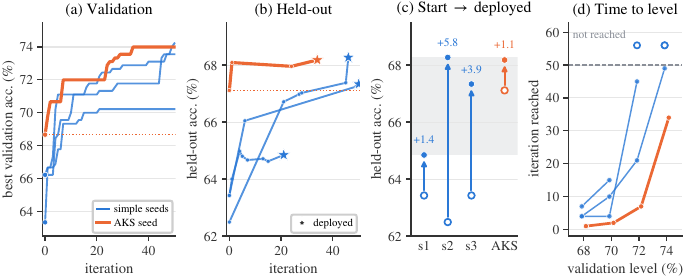}
    \caption{\textbf{Stronger initialization} (Video-MME, retrieval only; blue: three runs from simple seeds, orange: the AKS-seeded run). (a)~Best-so-far validation accuracy. (b)~Held-out accuracy of tested running champions; stars mark deployed harnesses. (c)~Held-out accuracy of the starting point (hollow) and deployed harness (filled); shading spans the simple-seed endpoints. (d)~Iteration at which each validation level is first reached; hollow markers are levels not reached in 50 iterations.}
    \label{fig:better_initialization}
\end{figure*}

\section{Implementation and Statistical Details of \method{}}
\label{app:method_details}

This section specifies what Section~\ref{sec:method} leaves out. Each seed runs \emph{WARM} (Stage~I), \emph{NOMINATE}, \emph{PROBE}, and \emph{RACE} (Stage~II), and \emph{COMPOSE} and \emph{MERGE} (Stage~III), after the six seeded harnesses are scored into a shared executor cache (Algorithm~\ref{alg:method}; settings in Table~\ref{tab:method_settings}). All comparisons are paired verdicts $\hat{\delta}(H,H')$ over questions scored for both harnesses, with variance given by Eq.~(\ref{eq:app_paired_var}); on MomentSeeker they use the indicator $\mathrm{IoU}\geq0.5$, while the loop and final selection use mIoU.

\begin{algorithm}[t]
\caption{\method{} as implemented}
\label{alg:method}
\small
\begin{algorithmic}[1]
\Require Frozen VLM $M$, proposer $P$, validation set $\mathcal{S}_{\mathrm{val}}$, seed archive $\mathcal{D}_0$, vocabulary $\mathcal{V}$, budgets $r_1<r_2$
\Statex \textbf{Stage I: Warm-up and seed selection} \hfill (setup, prewarm, WARM)
\State Grow $\mathcal{D}^{\mathrm{warm}}\supseteq\mathcal{D}_0$ by rounds of three free, one removal, and one intensity candidate until the plateau rule holds
\State $\mathcal{B}\gets$ six best evolved harnesses in $\mathcal{D}^{\mathrm{warm}}$; $H_0\gets\arg\max_{H\in\mathcal{B}}V(H)$; compute the warm-up report $\{\mathcal{R}^{\mathrm{warm}}_{\alpha}\}$
\Statex \textbf{Stage II: Isolated branch evolution} \hfill (NOMINATE, PROBE, RACE)
\State $\mathcal{A}\gets$ ladder-screened templates of $\mathcal{V}$ and a free-form axis; $P$ writes charters $\{c_{\alpha}\}$
\ForAll{$\alpha\in\mathcal{A}$}
    \State Evolve $r_1$ candidates in an isolated branch that holds $\mathcal{B}$, with parents $H_0$ or earlier branch candidates; $m_{\alpha}\gets\max_{H\in\mathcal{D}^{\mathrm{probe}}_{\alpha}}\hat{\delta}(H,H_0)$
\EndFor
\State $\mathcal{A}_{\mathrm{race}}\gets$ the $k_{\mathrm{r}}$ axes with the largest $m_{\alpha}$, plus at most one tied axis
\ForAll{$\alpha\in\mathcal{A}_{\mathrm{race}}$}
    \State Continue the branch to $r_2$ candidates in total; estimate $(\bar{\delta}_{\alpha},\hat{s}_{\alpha})$ by Eq.~(\ref{eq:method_bootstrap})
\EndFor
\State $\mathcal{A}_{\mathrm{adm}}\gets$ axes with $\bar{\delta}_{\alpha}>1.64\,\hat{s}_{\alpha}$ and $\bar{\delta}_{\alpha}>0.5$, in decreasing order of $\bar{\delta}_{\alpha}$
\State $\mathcal{A}_{\mathrm{comp}}\gets$ at most three axes of $\mathcal{A}_{\mathrm{race}}$, admitted axes first, then by $\bar{\delta}_{\alpha}$
\Statex \textbf{Stage III: Factorial composition and joint refinement} \hfill (COMPOSE, MERGE)
\State $P$ fixes one modification $d_{\alpha}$ per $\alpha\in\mathcal{A}_{\mathrm{comp}}$ and builds every cell $H_{\mathcal{U}}$ (Eq.~(\ref{eq:method_composition}))
\State Score the cells; record effects, interactions (Eq.~(\ref{eq:method_interaction})), and the gate-eligible outcome
\State Refine jointly for 10 rounds of three candidates, starting from a ledger that holds every $H_{\mathcal{U}}$
\State \Return $H^{\star}\gets$ validation arg max of the refinement ledger
\end{algorithmic}
\end{algorithm}

\subsection{Information Isolation}
\label{app:method_isolation}

WARM, PROBE, RACE, and MERGE run the loop of \basemethod{} (Appendix~\ref{app:basemethod_loop}) in separate workspaces; NOMINATE and COMPOSE are single proposer calls. Each per-question trace holds opaque identifiers for pairing, the prediction and correctness, and counters of executor and tool calls, frames, rendering, personas, memory operations, and exercised mechanism families; it never contains question or option text, frames, or gold answers. $\mathbf{u}(H)$ in Eq.~(\ref{eq:method_warm_evidence}) collects per-question means of these counters.

A deny list enforced on all file, search, and shell tools confines each proposer call to its own workspace and run, and keeps datasets, caches, held-out outputs, and evaluation code out of reach. Warm-up evidence reaches a branch only through its charter and the pool members copied into it; the merge workspace regains read access to earlier workspaces only after branch effects are measured, and no phase before MERGE evaluates the held-out split. (In the LongVideoBench and OVO-Bench runs, a monitor also evaluated validation-front members on the held-out split during search; these results were unreadable by the proposer and entered no decision.) Isolation withholds content, not outcomes: because records pair predictions with correctness, a proposer can reconstruct the answer key for questions some harness answers correctly (one Video-MME run did so for 402 of 450), but harness code cannot use it, since the leak guard bars access to question identity.

\subsection{Stage I: Instrumented Warm-up}
\label{app:method_warmup}

\paragraph{Mechanism families.}
\method{} groups the harness primitives into seven families, the units to which removals, intensity changes, traces, and axis templates refer (Prompts~1--4): \texttt{retrieve} (frame selection and query-conditioned retrieval), \texttt{multiturn} (multi-pass reading and self-reflection), \texttt{roles} (multi-role collaboration), \texttt{perceive} (detection and localization tools), \texttt{render} (pixel re-presentation such as zoom, crop, and montage), \texttt{analyze} (shot, motion, and similarity signals), and \texttt{memory} (working memory). The search-space ladder of Appendix~\ref{app:search_space_analysis} instead counts the families chartered in the task description of \basemethod{}, which uses a different granularity (e.g., it separates retrieval from frame selection); its family counts are not in these units.

\paragraph{Slots and audit.}
Each round evaluates five candidates in three slots: Slot~$\mathsf{A}$ makes three unrestricted \emph{free} proposals; Slot~$\mathsf{B}$ makes one \emph{removal} that deletes the champion's most-used family not yet removed; and Slot~$\mathsf{C}$ makes one \emph{intensity} proposal that changes one family's usage by about a factor of two. Each candidate's evidence $\mathbf{w}(H)$ follows Eq.~(\ref{eq:method_warm_evidence}). Requests are audited from traces rather than trusted: a candidate is a removal if its family set is a strict subset of its parent's, an intensity change if the sets agree and exactly one nonzero coordinate of $\mathbf{u}$ moves by at least 25\%, and free otherwise. Candidates with empty predictions on more than 5\% of questions count as crashed.

\paragraph{Stopping and backbone.}
WARM stops once $t\geq10$ rounds, or $t\geq8$ with no improvement in the last four. The seed pool $\mathcal{B}$ holds the six best evolved warm-up candidates, preferably at most two per parent, and the backbone $H_0$ is the best of them.

\paragraph{Warm-up report.}
NOMINATE receives, per family: (i)~\emph{importance}, the mean verdict of audited removals against their parents; (ii)~\emph{dose response}, the Spearman correlation between usage and accuracy (observational, not used for ranking); (iii)~usage share and oracle-union gain; (iv)~a \emph{baseline ladder} of paired differences between seeded harnesses with nested family sets, which measures value from a zero floor; and the frame range and median paired standard error. Blocks (i)--(iii) form $\mathcal{R}^{\mathrm{warm}}_{\alpha}$. Evidence that cannot be computed is shown as missing rather than zero, which marks under-explored mechanisms without removing any axis. Averaged per removed family, removal contrasts measure each family's contribution; the champion's inflation $\omega$ shifts them all alike and so does not distort how families compare. Intensity evidence is observational, so the report uses it to flag directions, not to rank axes.

\subsection{Stage II: Nomination, Probe, and Race}
\label{app:method_branch}

\paragraph{Axes and branches.}
Six templates---\emph{retrieval}, \emph{frame budget}, \emph{read schedule}, \emph{multi-role}, \emph{pixels and detectors}, and \emph{derived state}---fix each axis's permitted families and an on-axis test every branch edit must pass (e.g., a frame-budget edit must change mean frames by at least 20\%, and other templates may not change them); one free-form axis has a proposer-written test. A template is pruned only if it touches no family of $H_0$ and all its ladder effects are negative, keeping at least four. The proposer charters each axis with a mechanism, its rationale, a first edit, a falsification criterion, and a predicted effect. Every branch holds the pool $\mathcal{B}$ and the seeded harnesses, but candidates must descend from $H_0$ or earlier branch candidates, so all branch verdicts share the backbone that the lattice requires; other pool members may only be mined for mechanisms.

\paragraph{Probe.}
Each axis produces $r_1=3$ candidates, and axes are ranked by their best probe verdict,
\begin{equation}
    m_{\alpha}=\max_{H\in\mathcal{D}^{\mathrm{probe}}_{\alpha}}
    \hat{\delta}\bigl(H,H_0\bigr),
    \label{eq:method_probe}
\end{equation}
where $\mathcal{D}^{\mathrm{probe}}_{\alpha}$ is the archive of axis $\alpha$ after the probe. The $k_{\mathrm{r}}=\min\bigl(k,\max(2,\lceil k/2\rceil)\bigr)$ of the $k$ axes with the largest $m_{\alpha}$ continue (plus at most one tie). Since every axis is probed equally often against the same backbone, the backbone's inflation $\omega$ and the maximum's upward bias shift all axes alike and leave the ranking intact. We use the maximum rather than the mean because mechanisms that need compounding edits score near zero at first, and no standard error because probe candidates form one lineage cluster.

\paragraph{Race and admission.}
Survivors continue to $r_2=15$ candidates. With $\delta_h=\hat{\delta}(h,H_0)$ grouped by declared parent into lineage clusters $C_{\alpha,1},\ldots,C_{\alpha,L_{\alpha}}$, the effect $\bar{\delta}_{\alpha}$ is the mean of $\delta_h$, and its standard error comes from 2{,}000 cluster-bootstrap replicates,
\begin{equation}
    \bar{\delta}^{(b)}_{\alpha}
    =\frac{\sum_{\ell=1}^{L_{\alpha}}\sum_{h\in C_{\alpha,J^{(b)}_{\ell}}}\delta_h}
          {\sum_{\ell=1}^{L_{\alpha}}\bigl|C_{\alpha,J^{(b)}_{\ell}}\bigr|},
    \qquad
    \hat{s}_{\alpha}=\operatorname{sd}_{b}\bigl(\bar{\delta}^{(b)}_{\alpha}\bigr),
    \label{eq:method_bootstrap}
\end{equation}
with each $J^{(b)}_{\ell}$ uniform on $\{1,\ldots,L_{\alpha}\}$, so that edits of one program do not count as independent evidence. An axis is \textbf{admitted} if $\bar{\delta}_{\alpha}>1.64\,\hat{s}_{\alpha}$ and $\bar{\delta}_{\alpha}>0.5$~pp; with fewer than three clusters it is recorded as unjudged. Because $\omega$ lowers every $\bar{\delta}_{\alpha}$, this test is conservative. The 1.64 multiplier is a one-sided normal reference, not a simultaneous guarantee. At most three axes enter composition, admitted axes first.

\subsection{Stage III: Composition and Joint Refinement}
\label{app:method_compose}

\paragraph{Factorial lattice.}
COMPOSE distills each composed axis $\alpha\in\mathcal{A}_{\mathrm{comp}}$ into one modification $d_{\alpha}$, the smallest edit of $H_0$ that carries the branch's mechanism, and writes all $2^{k_{\mathrm{c}}}-1$ nonempty cells of Eq.~(\ref{eq:method_composition}), with each $d_{\alpha}$ byte-identical across cells. With three axes, each modification's main effect contrasts the four cells that contain it with the four that do not. With $x_{\mathcal{U}}(q)$ the correctness of $H_{\mathcal{U}}$ on question $q$ ($x_{\emptyset}$ for $H_0$), each cell's effect is $\hat{\delta}(H_{\mathcal{U}},H_0)$, and for $|\mathcal{U}|\geq2$ the interaction is
\begin{equation}
    I_{\mathcal{U}}=\frac{100}{|Q|}\sum_{q\in Q}\psi_{\mathcal{U}}(q),
    \qquad
    \psi_{\mathcal{U}}(q)=\sum_{\mathcal{W}\subseteq\mathcal{U}}(-1)^{|\mathcal{U}|-|\mathcal{W}|}\,x_{\mathcal{W}}(q),
    \label{eq:method_interaction}
\end{equation}
over the questions $Q$ shared by all cells, with standard error $100\operatorname{sd}_{q}\psi_{\mathcal{U}}(q)/\sqrt{|Q|}$; computing it per question retains the dependence through $H_0$. The pre-registered gate selects among $H_0$ and the cells of successive prefixes of the admission order ($H_0$ if no axis is admitted).

\paragraph{Joint refinement and final selection.}
MERGE runs 10 rounds of three candidates from a ledger holding all cells, so the proposer may build on any of them. $H^{\star}$ is the validation arg max of this ledger. The seeded harnesses and the MERGE validation front are then evaluated on the held-out split.

\subsection{Budget and Configuration}
\label{app:method_config}

With $T$ warm-up rounds, $k$, $k'$, and $k_{\mathrm{c}}$ probed, raced, and composed axes, and $R=10$ MERGE rounds, a seed evaluates
\begin{equation}
    B_{\mathrm{total}}
    =5T+k\,r_1+k'(r_2-r_1)+\bigl(2^{k_{\mathrm{c}}}-1\bigr)+3R
    \label{eq:method_budget}
\end{equation}
candidates, excluding re-scored seeded harnesses. With $T\leq10$, $k\leq7$ (hence $k_{\mathrm{r}}\leq4$), $k'=k_{\mathrm{r}}$, and $k_{\mathrm{c}}\leq3$, this is at most $50+21+48+7+30=156$ candidates per seed. We report this as the budget bound; the tie slack of the probe cut-off can add one raced axis ($r_2-r_1=12$ candidates) only when two axes tie exactly at the cut-off. The three Video-MME runs evaluated 138--141.

\begin{table}[t]
\centering
\small
\caption{\textbf{Settings of \method{}.} The free-only warm-up control (Section~\ref{sec:exp:warmablation}) differs only in the warm-up slots (5\,/\,0\,/\,0). Here $k_{\mathrm{r}}=\min\bigl(k,\max(2,\lceil k/2\rceil)\bigr)$ for $k$ probed axes. The warm-up runs in chunks of two rounds and stops once the champion has not changed for four rounds, after at least 8 and at most 10 rounds.}
\label{tab:method_settings}
\setlength{\tabcolsep}{4pt}
\begin{tabular}{@{}lc@{\hspace{14pt}}lc@{}}
\toprule
Setting & Value & Setting & Value \\
\midrule
Warm-up slots A\,/\,B\,/\,C per round & 3\,/\,1\,/\,1 & Probe survivors (tie slack) & $k_{\mathrm{r}}$ (+1) \\
Warm-up chunk; min.\,/\,max.\ rounds; patience & 2; 8\,/\,10; 4 & Gate $z_0$\,/\,floor (pp)\,/\,min.\ clusters & 1.64\,/\,0.5\,/\,3 \\
Class-1\,/\,free-form templates; min.\ kept & 6\,/\,1; 4 & Composed axes (cells) & $\leq3$ ($\leq7$) \\
Branch candidates per round; $r_1$; $r_2$ & 3; 3; 15 & MERGE rounds $\times$ candidates & $10\times3$ \\
\bottomrule
\end{tabular}
\end{table}

\section{Ablation: Instrumented Warm-up}
\label{app:warmablation}

This section details the warm-up ablation of Section~\ref{sec:exp:warmablation}. Its conclusions are:
\begin{itemize}[leftmargin=*,itemsep=1pt,topsep=2pt]
    \item \textbf{Same validation, better held-out.} Both arms reach the same validation accuracy, yet every full-arm run beats every free-only run on held-out data, for the warm-up backbone ($+1.85$~pp) and the deployed system ($+2.8$~pp).
    \item \textbf{More controlled evidence.} The full arm yields $2.7\times$ as many removal contrasts and far fewer undefined family estimates.
    \item \textbf{Attribution.} A within-run replay attributes $28\%$ of the backbone gap to selecting structured candidates; the rest is consistent with slot results informing free proposals.
\end{itemize}

\subsection{Protocol}
\label{app:warmablation:protocol}

Both arms run the complete \method{} pipeline on Video-MME, three runs each. Per warm-up round, the \emph{full} arm proposes three free candidates (Slot~$\mathsf{A}$), one family removal (Slot~$\mathsf{B}$), and one intensity change (Slot~$\mathsf{C}$); the \emph{free-only} arm spends the same five evaluations on free proposals. Everything else, including seeds, splits, caches, isolation, stopping rule, and later stages, is shared. Because slots are instructed rather than enforced, we classify every warm-up candidate from its trace with the audit of Appendix~\ref{app:method_warmup}; counts below are behavioral. Arms are compared by exact one-sided permutation tests over the $\binom{6}{3}=20$ assignments of runs, whose smallest attainable $p$-value, $0.05$, requires complete separation; with three runs per arm, all results indicate direction rather than precise effect sizes.

\begin{table}[t]
\centering
\caption{\textbf{Per-run results of the warm-up ablation} (Video-MME). Slots: behavioral counts of warm-up candidates. Families with an importance estimate (Imp.) or an undefined coverage increment (Und.). Backbone: warm-up champion; deployed: validation argmax of joint refinement. Val, Test in \%; Fr.: frames per validation question. $^\dagger$Stopped after round~8 (plateau rule).}
\label{tab:app_warm_runs}
\footnotesize
\setlength{\tabcolsep}{3pt}
\begin{tabular*}{\linewidth}{@{\extracolsep{\fill}}l*{11}{c}@{}}
\toprule
& \multicolumn{3}{c}{Slots} & \multicolumn{2}{c}{Families} & \multicolumn{3}{c}{Backbone} & \multicolumn{3}{c}{Deployed} \\
\cmidrule(lr){2-4}\cmidrule(lr){5-6}\cmidrule(lr){7-9}\cmidrule(lr){10-12}
Run & A & B & C & Imp. & Und. & Val & Test & Fr. & Val & Test & Fr. \\
\midrule
Full-1 & 27 & 11 & 12 & 6 & 2 & 73.56 & 69.87 & 50 & 77.11 & 71.07 & 110 \\
Full-2 & 26 &  9 & 15 & 3 & 1 & 71.11 & 68.40 & 55 & 73.33 & 70.00 & 141 \\
Full-3 & 23 & 15 & 12 & 5 & 2 & 70.44 & 68.58 & 38 & 74.67 & 69.16 & 66 \\
Mean & 25.3 & 11.7 & 13.0 & 4.7 & 1.7 & 71.70 & 68.95 & 48 & 75.04 & 70.07 & 106 \\
\midrule
Free-1 & 38 &  2 & 10 & 2 & 4 & 73.11 & 66.58 & 44 & 74.00 & 66.31 & 63 \\
Free-2$^\dagger$ & 27 & 8 & 5 & 3 & 4 & 71.33 & 66.62 & 65 & 72.44 & 67.16 & 81 \\
Free-3 & 34 &  3 & 13 & 3 & 3 & 71.11 & 68.09 & 59 & 74.89 & 68.40 & 86 \\
Mean & 33.0 & 4.3 & 9.3 & 2.7 & 3.7 & 71.85 & 67.10 & 56 & 73.78 & 67.29 & 77 \\
\bottomrule
\end{tabular*}
\end{table}

\subsection{Results}

\paragraph{Warm-up.}
Table~\ref{tab:app_warm_runs} lists per-run results. The free-only arm climbs faster early (71.18 versus 69.26 after six rounds; Figure~\ref{fig:app_warm}a), but the full arm catches up in rounds 7--9, and the arms end at 71.70 and 71.85. In the full arm, 8 of 11 champion replacements are structured (3 removals, 5 intensity changes), against 5 of 13 in the free-only arm, and all three full-arm backbones are intensity changes with structured parents. Deliberate contrasts are also cleaner: frames move by more than 10\% in $23\%$ of the full arm's removals, versus $77\%$ of the removals that arise unrequested in the free-only arm.

\paragraph{Evidence.}
The full arm yields 11.7 versus 4.3 removal contrasts per run ($p=0.05$) and 13.0 versus 9.3 intensity contrasts ($p=0.15$), giving 14 importance estimates covering six of the seven families versus 8 covering five, and leaving 5 versus 11 of the 21 family--run coverage increments (seven families $\times$ three runs) undefined ($p=0.05$). Retrieval shows the difference most clearly: each full run removes it four times (mean $-7.15$~pp), whereas the free-only arm has a single retrieval-free candidate. Removals are not always frame-matched: the largest drop ($-23.11$~pp) also cut frames from 80 to 16.

\begin{figure}[t]
\centering
\includegraphics[width=0.92\linewidth]{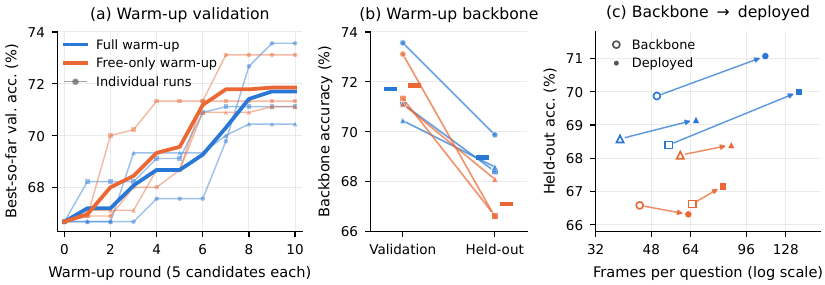}
\caption{\textbf{Warm-up ablation on Video-MME} (thin lines, markers: runs; thick lines, bars: arm means). (a)~Best-so-far validation accuracy (round~0: best seeded harness). (b)~Warm-up backbones. (c)~Held-out accuracy vs.\ frames per question, backbone (hollow) to deployed system (filled).}
\label{fig:app_warm}
\end{figure}

\paragraph{Held-out contrasts.}
\textbf{Every full-arm backbone and deployed system beats every free-only one on held-out data}: $68.95$ versus $67.10$ for backbones and $70.07$ versus $67.29$ for deployed systems (both $p=0.05$; Welch $t=2.73$ and $3.40$), while validation does not separate the arms ($p=0.65$ and $0.25$). Frames, which each harness chooses, are the main confound (Figure~\ref{fig:app_warm}c), but regressing held-out accuracy on arm and $\log_2$ frames leaves arm effects of $+2.0$ and $+2.2$~pp ($t=2.38$ and $2.55$, three residual degrees of freedom).

\paragraph{Attribution.}
Because all full-arm backbones are structured, part of the gap could come from selecting structured candidates. Replaying each full run's selection on its free candidates only yields substitutes within 0.67~pp of the actual backbones on validation but lower on held-out data in all three runs, so \textbf{selection accounts for $+0.52$~pp, or $28\%$ of the backbone gap}. The substitutes still average 1.33~pp above the free-only backbones. This remainder is consistent with slot results informing free proposals through the archive---8 of the full arm's 64 requested free proposals cite them, and the Full-3 backbone cites a removal result and a dose curve over retrieval calls---but three runs per arm cannot separate this channel from run-to-run variation (standard error about 0.7~pp).

\section{Held-out Questions on Unseen Videos}
\label{app:unseen_videos}

On MomentSeeker, the split assigns whole source videos to one side, so every held-out result concerns videos never seen during evolution. On Video-MME, LongVideoBench, and OVO-Bench, the split is over questions, so a held-out question may come from a video that also has validation questions: 702 of 2{,}250 held-out questions on Video-MME, 303 of 891 on LongVideoBench, and 622 of 979 on OVO-Bench. A harness could thus gain on such questions from properties of videos it was selected on. To check that the information-isolation and evolved-versus-AKS comparisons do not rest on such gains, we re-score them on the \emph{unseen-video} questions only, i.e., held-out questions whose video has no validation question (1{,}548 on Video-MME and 588 on LongVideoBench). Harnesses and predictions are those of the original evaluations; only the scored subset changes. Intervals are 95\% bootstrap intervals over videos (2{,}000 resamples).

\begin{table}[t]
\centering
\caption{\textbf{Held-out accuracy (\%) on all held-out questions and on unseen-video questions.} $\Delta$: difference on unseen-video questions, with its 95\% bootstrap interval over videos.}
\label{tab:unseen_videos}
\footnotesize
\setlength{\tabcolsep}{3pt}
\begin{tabular}{@{}llccc@{}}
\toprule
Comparison & Harness & All held-out & Unseen videos & $\Delta$ (unseen) \\
\midrule
\multirow{2}{*}{\shortstack[l]{Information isolation\\(LongVideoBench, Table~\ref{tab:isolation_ablation})}}
& \basemethod{} feedback (Multimodal) & 65.43 & 67.12 & \multirow{2}{*}{$+1.64$ \,$[+0.11, +3.28]$} \\
& \basemethod{} + isolation (Blindfolded) & 66.93 & 68.76 & \\
\midrule
\multirow{3}{*}{\shortstack[l]{Evolved vs.\ AKS\\(Video-MME, Figure~\ref{fig:initial_metaharness})}}
& Evolved \texttt{medium\_augspread} (38 frames) & 68.36 & 68.02 & -- \\
& AKS-BLIP (62 frames) & 68.40 & 68.73 & $-0.71$ \,$[-2.45, +1.16]$ \\
& AKS-CLIP (62 frames) & 68.09 & 68.22 & $-0.19$ \,$[-2.20, +1.74]$ \\
\bottomrule
\end{tabular}
\\[2pt]
{\footnotesize For AKS rows, $\Delta$ is the evolved harness minus that AKS variant.}
\end{table}

\paragraph{Information isolation.}
As shown in Table~\ref{tab:unseen_videos}, restricting the proposer to content-free feedback raises accuracy on unseen-video questions by $1.64$ points, with an interval that excludes zero. The effect of information isolation in Section~\ref{sec:exp:isolation} therefore holds on videos never seen during evolution.

\paragraph{Evolved versus hand-designed harnesses.}
On unseen-video questions, the evolved retrieval harness of Section~\ref{sec:metaharness_effectiveness} stays within one point of both AKS variants, and neither difference is significant, while it uses 38 rather than 62 frames per question. The conclusion that \basemethod{} evolves harnesses competitive with this hand-designed method therefore also holds on videos never seen during evolution.

\section{Failure-Mode Diagnostics of \method{}}
\label{app:method_diagnostics}

\begin{table}[t]
    \centering
    \caption{\textbf{Evidence per decision on Video-MME:} evaluations per configuration $r$, verdict noise $\sigma$, and power to detect a 1-point gain. The \basemethod{} row uses the first 57 candidates of each run.}
    \label{tab:method_diagnostics}
    \small
    \setlength{\tabcolsep}{6pt}
    \begin{tabular}{@{}llccc@{}}
        \toprule
        & Stage & $r\uparrow$ & $\sigma\downarrow$ & Power$\uparrow$ \\
        \midrule
        BRSI & Full harness & 1.55 & 1.55 & 0.13 \\
        \midrule
        \multirow{4}{*}{SRSI}
        & I: warm-up & 1.59 & 1.62 & 0.12 \\
        & II: branch & \textbf{9.80} & 1.61 & \textbf{0.50} \\
        & III: merge & 1.68 & \textbf{1.12} & 0.22 \\
        & Whole run & 2.31 & 1.48 & 0.18 \\
        \bottomrule
    \end{tabular}
\end{table}

\begin{table}[t]
    \centering
    \caption{\textbf{Early-experience trap.} Runs whose champion never exceeds $1.2\times$ the seeds' frame budget (16 frames; 32 on MS), three seeds each; \basemethod{} searches the full harness with Multimodal feedback. VM: Video-MME; OVO: OVO-Bench; MS: MomentSeeker.}
    \label{tab:app_stuck_runs}
    \small
    \begin{tabular}{@{}lccccc@{}}
        \toprule
        & VM & LVB & OVO & MS & Total \\
        \midrule
        \basemethod{} & 1 / 3 & 0 / 3 & 1 / 3 & 2 / 3 & 4 / 12 \\
        \method{}     & \textbf{0 / 3} & 0 / 3 & \textbf{0 / 3} & \textbf{0 / 3} & \textbf{0 / 12} \\
        \bottomrule
    \end{tabular}
\end{table}

We re-measure the three diagnostics of Section~\ref{sec:meta_harness_failure_modes} on the runs of Table~\ref{tab:main_results}.

\paragraph{Validation overfitting.}
\method{} narrows the validation--test gap on three of the four benchmarks (Table~\ref{tab:main_results}). On LongVideoBench, where the gap does not shrink, information isolation alone still raises the held-out accuracy of \basemethod{} from $65.4$ to $66.9\%$ at equal validation accuracy, close to the $67.3\%$ of \method{} (Section~\ref{sec:exp:isolation}), and this gain holds on held-out questions whose videos never appear in validation (Appendix~\ref{app:unseen_videos}).

\paragraph{Search-budget dilution.}
The probe--race concentrates the branch budget on a few axes, so each axis decision on Video-MME rests on $9.80$ evaluations instead of the $1.55$ of \basemethod{}, raising the power to detect a 1-point gain from $0.13$ to $0.50$ (Table~\ref{tab:method_diagnostics}). Stage~III, in turn, has the lowest verdict noise of any stage ($\sigma=1.12$).

\paragraph{Early-experience trap.}
Because the warm-up of \method{} collects removal and intensity contrasts across the mechanism families regardless of earlier verdicts, no \method{} run stays at the seeds' frame budget on any benchmark (Table~\ref{tab:app_stuck_runs}), whereas four of the twelve \basemethod{} runs do, on Video-MME, OVO-Bench, and MomentSeeker. An early, noisy loss can no longer close off a mechanism for the rest of the run.

\section{Proposer Prompts of \method{}}
\label{app:prompts}

This section reproduces the most important proposer prompts of \method{}, taken from a Video-MME run with the Qwen3-VL-8B executor and Claude Opus~5 as the proposer. Text in braces, such as \texttt{\{BACKBONE\}}, is filled in by the orchestrator at run time with that run's measurements; run constants that the prompts leave as placeholders are filled in with the values of Table~\ref{tab:method_settings}. We elide long passages with \texttt{[...]} and drop Markdown emphasis and non-ASCII symbols; everything else is verbatim. The prompts use internal names that map onto Section~\ref{sec:method} as follows: a \emph{family} is one of the seven mechanism families of Appendix~\ref{app:method_warmup}; the \emph{backbone} is $H_0$; \emph{WARM} is Stage~I, whose Slots~A/B/C are the free, removal, and intensity proposals; a \emph{class-1 slot} is one of the six fixed axis templates and a \emph{class-2} charter is the free-form axis $\alpha_{\mathrm{free}}$; \emph{PROBE} and \emph{RACE} are the $r_1=3$ and $r_2=15$ schedules; a branch's $d_i$ is its effect $\bar{\delta}_{\alpha}$; a \emph{cell} is a composition $H_{\mathcal{U}}$; and the \emph{gate} is the admission test of Appendix~\ref{app:method_branch}.

\paragraph{Iteration skill.}
Every evolution loop of \method{} (WARM, PROBE, RACE, and MERGE) runs the loop of Appendix~\ref{app:basemethod_loop}, whose proposer session loads the skill below. In PROBE and RACE, the branch's charter skill replaces the opening sections, and the interface, constraints, and workflow are appended verbatim, identical on every branch. The proposer sees only content-free traces (Section~\ref{sec:method_setup}); the leak guard at the end is enforced on every candidate before it is evaluated.

\begin{promptbox}[Prompt 1: Iteration skill shared by all evolution loops (excerpt)]
# Meta-Harness: Full Agentic Framework

Run ONE iteration of evolution on Video-MME. Do all work in the main session - do
NOT delegate to subagents (constraints get lost, and you end up shipping parameter
tweaks with no prototyping).

You do NOT run benchmarks. You analyze results + traces, prototype, and implement
new frameworks. The outer loop benchmarks them on the resident Qwen pool and runs
the selection layer.

## What is being evolved

A complete training-free agentic framework that takes a video + a multiple-choice
question and returns the FINAL answer letter. Everything is on the table at once:

| mechanism                                   | primitive                                        |
|---------------------------------------------|--------------------------------------------------|
| frame selection                             | uniform_indices, ctx.pool_indices(), ctx.read_frames |
| query-conditioned retrieval                 | ctx.retrieve(model, query, k)                    |
| perception tools                            | ctx.tools.detect / locate                        |
| cheap video signals                         | ctx.tools.shots / motion / similarity            |
| pixel re-presentation                       | ctx.tools.zoom / crop / montage + ctx.look(...)  |
| self-reflection & multi-turn correction     | repeated ctx.vlm_answer conditioned on earlier output |
| debate / multi-agent collaboration          | ctx.agent(role, prompt, indices=...)             |
| working-memory management                   | ctx.memory.note / digest / set_policy / forget   |

There is exactly ONE model doing inference: the frozen Qwen3-VL. "Multi-agent"
means several ROLES of that same model. You (the outer agent) write the framework
code; you never answer a question yourself, and neither does any other external
model. No finetuning, no weight updates, no new model downloads, no network.

Goal: maximize Video-MME accuracy. The cost axis is average frames summed over ALL
passes (tool calls and long generations cost too). A deeper or broader pipeline has
to EARN its cost: the frontier is accuracy vs frames, so machinery that buys nothing
is strictly worse than not having it.

## Interface - implement answer(ctx)
[... interface, primitives on ctx, and per-question budgets ...]

## CRITICAL CONSTRAINTS

- Tag each candidate's components with the mechanisms it uses. This is a ledger
  field for later analysis; nothing in the loop reads it back to you.
- You MUST implement 3 new frameworks every iteration, each with a base_strategy -
  the exact parent you built on, taken from evolution_summary.jsonl /
  frontier_val.json, or a baseline. Mix exploitation and exploration.
- Do NOT abort early, do NOT write "the frontier is optimal", do NOT skip
  prototyping.
- Anti-parameter-tuning: a new candidate must change a mechanism - what triggers an
  action, which evidence is gathered, how rounds/roles/memory/fusion work - not just
  num_frames, k, max_rounds or a threshold. If your candidate differs from its
  parent only by constants, REWRITE it.
- One frozen model only. All reasoning happens inside Qwen3-VL via ctx.vlm_answer /
  ctx.ask_vlm / ctx.agent / ctx.look. No external API, no second model, no
  answering by hand.
- Tool-cost discipline. [...]
- No dataset-specific hints. General patterns are fine ("keep frames where the
  queried object is detected"). Hardcoded knowledge about particular videos or
  questions is not.

### Leakage rules (enforced automatically - a violation REJECTS the candidate)

Every candidate is AST-scanned by leak_guard.py before it is benchmarked. Your
framework gets everything it needs through ctx, so:
- No filesystem, network or subprocess access - no open, pathlib, pandas, requests,
  subprocess, os.environ, no reading logs or the dataset.
- Imports must be on the allow-list [...]. No random (be deterministic).
- No dataset-identity handles: ctx.video_path, question_id, videoID, task_type,
  duration-label routing. (ctx.duration_sec - the number - is fine and useful.)
- No eval/exec/compile/__import__, no hashing question text, no large literal
  lookup tables.

Separately: you must never read the dataset parquet, the test split, other runs'
logs, or private .hstate/ state. The gold answer is not available to you or to
your framework and must not be sought.

## WORKFLOW

1. Analyze: read evolution_summary.jsonl, frontier_val.json, config.yaml, and recent
   val_log.jsonl traces. Look at where a framework flips right<->wrong, which
   families it actually used, and what it spent.
2. Prototype (MANDATORY): a throwaway script in .scratch/. With the stubbed VLM,
   verify the loop terminates, returns a valid letter, exercises the branch you
   intended, and survives the all-empty fallback path.
3. Implement: copy a baseline from agents/ to agents/<name>.py and build the new
   mechanism. Self-critique: is this a genuinely new mechanism, or a parameter
   variant? Then validate the import and run leak_guard.py on it.
4. Write pending_eval.json to the path in the task prompt.
\end{promptbox}

\paragraph{Stage I: warm-up slots.}
In the WARM workspace, the orchestrator appends the addendum below to the iteration skill, which turns each round into the three free, one removal, and one intensity candidates of Section~\ref{sec:method_warm}. The prompt only requests the slots; as stated in its last rule and in Appendix~\ref{app:method_warmup}, each candidate's slot is audited from its trace. The free-only control of Section~\ref{sec:exp:warmablation} receives an addendum that assigns all five candidates to Slot~A.

\begin{promptbox}[Prompt 2: Stage I warm-up slot addendum]
# WARM PHASE ADDENDUM - five candidates per iteration, in three slots

This warm phase is not only searching for a good framework. It is also
manufacturing the evidence that the next phase needs in order to charter branches.
Those are different jobs and they need different candidates.

| slot                            | count | job                          |
|---------------------------------|-------|------------------------------|
| A free exploration              | 3     | find a good backbone         |
| B one-family ablation           | 1     | manufacture a control group  |
| C one-family intensity sweep    | 1     | manufacture an intensity range |

Slot A keeps all three of the candidates a normal round would have had. B and C
are added on top, not carved out of A, so nothing about your search for a good
framework is reduced.

## slot A - free exploration (3 per iteration)

Whatever you think is most promising. No constraint beyond the SKILL above.

## slot B - one-family ablation of the current champion (1 per iteration)

Take the current best framework and produce a variant that removes exactly one
primitive family it uses, keeping everything else byte-for-byte. Pick the family
with the highest usage share that you have not already ablated.
Name it slotB_<family>_off (e.g. slotB_retrieve_off).

Why this slot exists: a family that every candidate uses has no control group, so
any with/without estimator on it returns "too few candidates to test" - and the
next phase is then structurally unable to say anything about it. That is a
property of the estimator, not of the family, and this slot is what fixes it.

You are not trying to make this candidate win. A large drop is the useful outcome:
it is the measurement.

## slot C - one-family intensity sweep (1 per iteration)

Take any framework in the archive and produce a variant that changes how hard one
family is worked - and nothing else. Halve or double the retrieval budget, run the
detector on 3 frames instead of 12, take 2 reflection passes instead of 1. Same
family set, same structure, one dial moved by at least ~2x.
Name it slotC_<family>_<up|down> (e.g. slotC_retrieve_up).

Why: a family that every candidate already uses cannot be measured by
presence/absence, but it can still be measured by dose. This slot is what creates
the intensity range that measurement needs.

## Rules for all five candidates

- All five are implemented and evaluated normally. B and C compete for the
  champion like anything else - they are simply unlikely to win, and that is fine.
- Do not spend the frame budget to make B or C look good. Frame count is a strong
  confound on this task: a B or C candidate that also moves the frame budget
  measures the budget, not the family.
- If the champion has only one family left to remove and you have already ablated
  it, use slot B for a second intensity sweep on a different family and say so in
  the hypothesis field.
- Put the slot letter in the strategy name. The orchestrator audits the slots
  afterwards from behaviour, not from the name (a candidate counts as an ablation
  if its family set is a strict subset of its parent's), so a mislabelled name is
  harmless - but a correct name makes the log readable.
\end{promptbox}

\paragraph{Stage II: nomination.}
NOMINATE is a single proposer call that receives the warm-up report (Appendix~\ref{app:method_warmup}) in place of the braced tables and writes one charter per axis: its mechanism, rationale, first edit, falsification criterion, and predicted effect (Section~\ref{sec:method_branches}). The axis templates it must fill are given in Prompt~4.

\begin{promptbox}[Prompt 3: Stage II NOMINATE (excerpt)]
# NOMINATE - fill the fixed axis slots, then add your own

The warm phase is over. Its job was not to find the best framework - it was to
produce a shared backbone and enough measured history to charter the branches that
come next.

Your task is not to write a framework, and it is not to decide which axes are
worth exploring. The axis set is largely decided for you. Your task is:
1. for every class-1 slot below, write the mechanism that slot will try on this
   run, and write that branch its own harness (SKILL.md);
2. add up to 1 class-2 charter of your own - a composite or mechanism this
   vocabulary has no slot for;
3. commit to a numeric prediction for each.

After you finish, every charter gets a short blind branch of 3 candidates. The
ones that survive elimination continue to 15, and only then is an effect size
with a standard error computed and gated.

## Why the slots are fixed, and why you cannot skip one

An earlier version of this phase let the proposer choose the axes freely. Its
recorded failure mode was reading a null union-ceiling increment as evidence of no
headroom and excluding that family - when a null there means only that every warm
candidate used the family, so there was no control group to measure it against.

increment: null means UNTESTABLE, never REFUTED. [...] a slot cannot be skipped,
so a family cannot be dropped for being unmeasurable.

You are not being told which axis is good. You are being told that the question of
which axis is good is settled by the PROBE and RACE phases that follow, on this
run's own measurements - not by an argument made here.

## What the warm run measured

Backbone: {BACKBONE} at {BACKBONE_VAL}% on {N_VAL} val questions.
Families it already uses: {BACKBONE_FAMILIES}.
Paired noise floor for a single candidate on this run: {PAIRED_SE}.

### Family share, and the union-ceiling increment
{FAMILY_TABLE}
### IMPORTANCE - what the ablation slots measured (removing the family from the champion)
{IMPORTANCE_TABLE}
### HEADROOM - dose-response among candidates that use the family
{HEADROOM_TABLE}
> It is observational - it is read off candidates that were free to change several
> families at once, so it is confounded by construction [...]. Do not rank axes on it.
### The free baseline ladder
{LADDER_TABLE}
> They measure importance from a zero floor, not headroom on top of your backbone.
> [...] treat the ladder rows for those families as uninformative.
### Frames
{FRAMES_NOTE}
> Frame budget is the strongest single confound on this task: an axis that raises
> accuracy by also raising frames has measured the budget, not the mechanism.
> Every class-1 template except frame_budget forbids moving it, and frame_budget
> exists so that the compute axis is priced openly instead of being smuggled into
> another one. Honour that in the mechanism you write.

## The class-1 slots you must fill

For each slot, families_allowed, candidate_test and forbidden are fixed and printed
in full below - copy them into your entry. What you actually write is the
mechanism: the specific thing this branch should try on this backbone, this run.
Read candidate_test before you write the mechanism. It is the definition of "on
axis" that every candidate on that branch is measured against, so a mechanism that
cannot satisfy it is a branch that measures nothing.

{VOCAB_TABLE}

## Output contract

Write logs/warm/nomination.json:
{"axes": [{
  "name": "retrieval_span_walk", "axis_class": 1, "vocabulary_id": "retrieval",
  "families_allowed": ["retrieve"],
  "candidate_test": "<copied VERBATIM from the vocabulary template>",
  "charter": "one paragraph: what this branch may change and what it may not",
  "mechanism": "the specific thing to try on THIS backbone, and why this one",
  "why_this_axis": "cite a number from the warm run above",
  "what_would_falsify_it": "what pattern over ~15 candidates would kill it",
  "first_move": "the exact first diff, small enough to implement in one iteration",
  "expected_effect_pp": 0.8,
  "frame_neutral": true }]}

Rules the contract check enforces - a violation costs you a retry:
- every class-1 slot listed above appears exactly once. You may not skip one.
- at most 1 entry with "axis_class": 2. [...] you write all fields - including your
  own candidate_test, since there is no template to take one from. Write it the way
  the class-1 tests above are written: what a diff must provably do for the branch
  to count it as on-axis.
- expected_effect_pp is a real prediction in [-5, +5], and they must not all be the
  same number. It is scored against the measured d_i afterwards.
- frame_neutral states whether the mechanism can raise accuracy without raising the
  frame budget.

And write one harness per axis at logs/warm/axis_skills/<axis_name>/SKILL.md [...].
Every SKILL is written by you in one session, but the branches are blind to each
other. A SKILL that names another branch voids that blindness and is rejected by
the contract check. Write each one as if it were the only branch.
\end{promptbox}

\paragraph{Stage II: axis templates.}
The six class-1 templates of Section~\ref{sec:method_branches} fill \texttt{\{VOCAB\_TABLE\}} in Prompt~3. Each fixes the families a branch may touch, the on-axis test its every edit must pass, and what it must not change; only the mechanism is left to the proposer. Only \texttt{frame\_budget} may change the number of frames, so that the frame budget, the strongest confound on these benchmarks, is measured on its own axis.

\begin{promptbox}[Prompt 4: The six fixed axis templates (class-1 slots)]
retrieval            families: retrieve
  WHICH frames the answering pass is shown.
  candidate_test: The diff MUST change how the indices handed to ctx.read_frames /
    ctx.vlm_answer(indices=...) are produced, on the path to the returned letter, so
    that reverting the diff provably changes which frames the model sees.
  forbidden: changing the total frame budget (num_frames) relative to the backbone.

frame_budget         families: retrieve
  HOW MANY frames, and how they are spread across a question.
  candidate_test: The diff MUST change the number of frames read per question (mean
    over val) by at least 20% relative to the parent, either globally or by a
    runtime rule, while keeping the parent's selection rule byte-for-byte.
  forbidden: changing which retriever, ranking or selection rule is used.

read_schedule        families: multiturn
  The pass graph AND the rule that turns pass outputs into a letter.
  candidate_test: The diff MUST change at least one of: (a) which ctx.vlm_answer /
    ctx.ask_vlm passes run or what they are conditioned on at runtime, or (b) the
    function from collected pass outputs to the returned letter (abstention,
    tie-breaks, thresholds, option elimination).
  forbidden: changing the total frame budget; changing which frames are selected.

multi_role           families: roles
  More than one answering persona, and how they are adjudicated.
  candidate_test: The diff MUST issue at least two passes with materially different
    role prompts and combine them by a rule that depends on their disagreement.
  forbidden: changing the total frame budget.

pixels_and_detectors families: perceive, render
  What image is built and shown, and the object geometry that decides it.
  candidate_test: The diff MUST either (a) hand ctx.look a pixel array built by a
    composition rule the parent does not have (different tiling/geometry, a crop or
    rescale, several moments joined, an extra plane), or (b) call ctx.tools.detect /
    locate and route a value from that call into the composition or the prompt. In
    both cases the frame list feeding it MUST be copied from the parent.
  forbidden: changing which frames are selected; changing the total frame budget.

derived_state        families: analyze, memory
  Quantities the framework computes about the video or about its own passes, and
  carries forward.
  candidate_test: The diff MUST either (a) compute a quantity from the VALUES
    returned by ctx.tools.shots / motion / similarity that the parent does not
    compute, or (b) contain a ctx.memory.note(...) whose text derives from an earlier
    pass's raw output together with a later prompt that embeds ctx.memory.digest(...)
    or ctx.memory.items(...). In both cases there MUST be a visible data path from
    that quantity into the frame indices or a prompt on the way to the returned
    letter.
  forbidden: changing the total frame budget.
\end{promptbox}

\paragraph{Stage III: composition.}
COMPOSE is a single proposer call after the race. It receives each raced axis's effect $\bar{\delta}_{\alpha}$ with its lineage-clustered standard error (Eq.~(\ref{eq:method_bootstrap})) and the gate's verdict, and it builds every cell of Eq.~(\ref{eq:method_composition}) from one fixed modification $d_{\alpha}$ per axis. The cells are then scored and seed the MERGE ledger (Appendix~\ref{app:method_compose}). The prompt's ``15 rounds'' refers to the $r_2=15$ candidates of each raced branch.

\begin{promptbox}[Prompt 5: Stage III COMPOSE (excerpt)]
# COMPOSE - build the whole lattice, ship only what the gate pre-registered

The {K} branches are finished. Each spent 15 rounds moving exactly one axis of the
same frozen backbone, blind to the others. This step turns them into one framework.

## What the measurement says

Backbone: {BACKBONE} (val {BACKBONE_VAL}), families {BACKBONE_FAMILIES}.
Per-axis mechanism estimate - the mean paired val delta of that branch's candidates
against the backbone, with a lineage-clustered bootstrap SE (same-parent
descendants count as one draw, because they are edits of each other and not
independent evidence):
{AXIS_TABLE}
Admission gate (pre-registered):
{GATE}

## The rule you are implementing

The gate does not decide what you BUILD. You build the whole lattice. The gate
decides what may later be SHIPPED, which is a separate step you are not performing.
Your job is to make every combination measurable, including the ones the gate
rejected - a rejected axis that turns out to combine well is exactly the thing this
design exists to catch, and a filter can never show it.

### Step 1 - fix one diff per axis, before you write anything

For each of the {K} axes, extract one diff: the smallest edit to the backbone that
carries that axis's mechanism, generalised across that branch's candidates if you
like, but attributable to that axis alone. Write these down first.

### Step 2 - build all {N_CELLS} cells
{LATTICE}
Every cell is backbone + the diffs of its axes, and the diff for an axis must be
byte-identical in every cell that contains it. This is not a style rule. The
lattice is read as a factorial design: I_ab = d_ab - d_a - d_b is only an
interaction if the a in cell_ab is the same a as in cell_a. If a diff has to change
to fit alongside another, that is not a merge - it is a rewrite, and the right move
is to shrink both diffs until they compose without editing each other, then say so.

Singleton cells are not optional. They are the main-effect cells and no interaction
is computable without them. Note that the effect of cell_a is not the same quantity
as axis a's d_i in the table above: d_i is the mean over ~15 branch candidates,
cell_a is the effect of the one diff you chose. The gap between them is real
information - if your chosen diff underperforms the branch mean, you picked the
wrong diff.

### Step 3 - if a cell cannot be built, say which and why

Two diffs may be genuinely incompatible (both replace the same seam). Do not paper
over it: write the cells you can, and in the hypothesis of the smallest affected
cell name the conflict precisely. [...]

## Constraints that are checked

- Complexity budget. Count if/elif branches and named numeric constants in the
  backbone and in each diff you graft. The merged candidate may not exceed
  backbone + sum(diffs). [...] declare the counts in your hypothesis so they are in
  the ledger.
- No rewriting the backbone. If a graft seems to require restructuring the
  backbone, that is evidence the two mechanisms are not composable at this seam:
  say so in the hypothesis, and graft the smaller version instead.
- Every candidate declares base_strategy = the backbone (for the first prefix) or
  the previous prefix.
- Read the on-axis / bespoke on-off axis / fired columns before the d column. [...]
  A branch with more bespoke code off its axis than on it spent its rounds
  somewhere other than its charter, so its d_i is not an estimate of the mechanism
  its name claims - treat that branch's diffs as unlabelled, and say so.

## What to write
[... manifest path and format; every cell lists the axes it grafts ...]
In each hypothesis, state: which axes this cell grafts, which branch candidate each
diff came from, the complexity counts, and what you expect the effect to be - so
that a smaller-than-predicted effect is visible as a failed prediction rather than
absorbed silently. For the multi-axis cells, also predict the sign of the
interaction: do you expect these mechanisms to add, to overlap (sub-additive, they
fix the same questions), or to fight (super-additive against you)?
\end{promptbox}

\end{document}